\documentclass[twoside,11pt]{article}

\usepackage[abbrvbib]{jmlr2e}

\usepackage{bbm}
\usepackage{amsbsy}
\usepackage{graphicx}
\usepackage{amsfonts}
\usepackage{bm}
\usepackage{url}
\usepackage{enumitem}
\usepackage{booktabs}
\usepackage{multirow}
\usepackage{makecell}
\usepackage{mathtools}
\usepackage{dsfont}
\usepackage{threeparttable}
\usepackage{tablefootnote}
\usepackage{adjustbox}
\usepackage[normalem]{ulem}
\usepackage[ruled,linesnumbered]{algorithm2e}
\newcolumntype{V}{!{\vrule width 1pt}}
\mathtoolsset{showonlyrefs}  
\usepackage{placeins}
\allowdisplaybreaks

\makeatletter
\renewcommand*{\@opargbegintheorem}[3]{\trivlist
      \item[\hskip \labelsep{\bfseries #1\ #2}] \textbf{(#3)}\ \itshape}
\makeatother

\makeatletter
\newcommand*{\Rom}[1]{\expandafter\@slowromancap\romannumeral #1@}
\makeatother
\newcommand*{\rom}[1]{\romannumeral #1}
\newcommand{\fullset}{S_{\textup{Full}}}
\newcommand{\ds}{d_S}
\newcommand{\hds}{\hat{d}_S}
\newcommand{\p}{\textup{P}}
\newcommand{\bp}{\textbf{\textup{P}}}
\newcommand{\tp}{\mathbb{P}}
\newcommand{\e}{\textup{E}}
\newcommand{\be}{\textbf{\textup{E}}}
\newcommand{\te}{\mathbb{E}}

\newcommand{\bv}{\textbf{\textup{Var}}}

\newcommand{\btheta}{\bm{\theta}}
\newcommand{\hthetas}{\hat{\bm{\theta}}_S}
\newcommand{\bthetas}{\bm{\theta}_S}
\newcommand{\bdelta}{\bm{\delta}}

\newcommand{\bdeltas}{\bm{\delta}_S}
\newcommand{\hdeltas}{\hat{\bm{\delta}}_S}
\newcommand{\tr}{\textup{Tr}}
\newcommand{\ty}{\mathcal{T}}
\newcommand{\infnorma}[1]{\left\|#1\right\|_{\infty}}

\newcommand{\twonorma}[1]{\left\|#1\right\|_{2}}
\newcommand{\infnorm}[1]{\|#1\|_{\infty}}
\newcommand{\twonorm}[1]{\|#1\|_{2}}

\newcommand{\maxnorma}[1]{\left\|#1\right\|_{\max}}
\newcommand{\maxnorm}[1]{\|#1\|_{\max}}

\newcommand{\fnorm}[1]{\|#1\|_{F}}
\newcommand{\norma}[1]{\left|#1\right|}
\newcommand{\norm}[1]{|#1|}
\newcommand{\ric}{\textup{RIC}}
\newcommand{\kl}{\textup{KL}}
\newcommand{\bx}{\bm{x}}
\newcommand{\hpi}{\hat{\pi}}
\newcommand{\taus}{\tau_S}
\newcommand{\htaus}{\hat{\tau}_S}

\newcommand{\bxs}{\bm{x}_S}

\newcommand{\bmu}{\bm{\mu}}
\newcommand{\hmusr}[1]{\hat{\bm{\mu}}_S^{(#1)}}
\newcommand{\bmusr}[1]{\bm{\mu}_S^{(#1)}}
\newcommand{\hmur}[1]{\hat{\bm{\mu}}^{(#1)}}
\newcommand{\bmur}[1]{\bm{\mu}^{(#1)}}
\newcommand{\hsig}{\hat{\Sigma}}
\newcommand{\hsigr}[1]{\hat{\Sigma}^{(#1)}}
\newcommand{\sigr}[1]{\Sigma^{(#1)}}
\newcommand{\sigsr}[2]{\Sigma^{(#1)}_{#2}}
\newcommand{\hsigsr}[2]{\hat{\Sigma}^{(#1)}_{#2}}
\newcommand{\hinvsig}{\hat{\Sigma}^{-1}}
\newcommand{\invsig}{\Sigma^{-1}}

\newcommand{\omegar}[1]{\Omega^{(#1)}}
\newcommand{\homegasr}[2]{\hat{\Omega}^{(#1)}_{#2}}
\newcommand{\omegasr}[2]{\Omega^{(#1)}_{#2}}
\newcommand{\supp}{\sup_{S: |S| \leq D}}
\newcommand{\supps}{\sup_{\substack{S: S \supseteq S^* \\ |S| \leq D}}}
\newcommand{\gammal}{\gamma_l}
\newcommand{\gammaq}{\gamma_q}
\newcommand{\cri}{\textup{Cr}}
\newcommand{\barp}{\bar{p}^*}

\def\boxit#1{\vbox{\hrule\hbox{\vrule\kern6pt\vbox{\kern6pt#1\kern6pt}\kern6pt\vrule}\hrule}}

\firstpageno{1}

\usepackage{lastpage}
\jmlrheading{22}{2021}{1-\pageref{LastPage}}{6/20; Revised 1/21}{1/21}{20-600}{Ye Tian and Yang Feng}
\ShortHeadings{RaSE: Random Subspace Ensemble Classification}{Tian and Feng}

\begin{document}

\title{RaSE: Random Subspace Ensemble Classification}

\author{\name Ye Tian \email ye.t@columbia.edu \\
       \addr Department of Statistics\\
       Columbia University\\
       New York, NY 10027, USA
       \AND
       \name Yang Feng \email yang.feng@nyu.edu \\
       \addr Department of Biostatistics, School of Global Public Health\\
       New York University\\
       New York, NY 10003, USA}

\editor{Jie Peng}

\maketitle

\begin{abstract}
We propose a flexible ensemble classification framework, \emph{Random Subspace Ensemble} (RaSE), for sparse classification. In the RaSE algorithm, we aggregate many weak learners, where each weak learner is a base classifier trained in a subspace optimally selected from a collection of random subspaces. To conduct subspace selection, we propose a new criterion, \emph{ratio information criterion} (RIC), based on weighted Kullback-Leibler divergence. The theoretical analysis includes the risk and Monte-Carlo variance of the RaSE classifier, establishing the screening consistency and weak consistency of RIC, and providing an upper bound for the misclassification rate of the RaSE classifier. In addition, we show that in a high-dimensional framework, the number of random subspaces needs to be very large to guarantee that a subspace covering signals is selected. Therefore, we propose an iterative version of the RaSE algorithm and prove that under some specific conditions, a smaller number of generated random subspaces are needed to find a desirable subspace through iteration. An array of simulations under various models and real-data applications demonstrate the effectiveness and robustness of the RaSE classifier and its iterative version in terms of low misclassification rate and accurate feature ranking. The RaSE algorithm is implemented in the R package \texttt{RaSEn} on CRAN.
\end{abstract}



\section{Introduction}\label{sec: introduction}
Ensemble classification is a very popular framework for carrying out classification tasks, which typically combines the results of many weak learners to form the final classification. It aims at improving both the accuracy and stability of weak classifiers and usually leads to a better performance than the individual weak classifier \citep{rokach2010ensemble}. Two prominent examples of ensemble classification are bagging \citep{breiman1996bagging, breiman1999pasting} and random forest \citep{breiman2001random}, which focused on decision trees and aimed to improve the performance by bootstrapping training data and/or randomly selecting the splitting dimension in different trees, respectively. Boosting is another example that converts a weak learner that performs only slightly better than random guessing into a strong learner achieving arbitrary accuracy \citep{freund1995desicion}. Recently, several new ensemble ideas appeared.  \cite{blaser2016random} aggregate decision trees with random rotations to get an ensemble classifier. In particular, it randomly rotates the feature space each time prior to fitting the decision tree. Random rotations make it possible for the tree-based classifier to arrive at oblique solutions, increasing the flexibility of the decision trees. As a variant, to make the ensembles favor simple base learners, \cite{blaser2019regularizing} proposed a regularized random rotation algorithm. Another popular framework of ensemble classifiers is via random projection. \cite{durrant2015random} studied a random projection ensemble classifier with linear discriminant analysis (LDA) and developed its theoretical properties. Furthermore, \cite{cannings2017random} proposed a very general framework for random projection ensembles. Each weak learner first randomly projects the original feature space into a low-dimensional subspace, then the base classifier is trained in that new space. The choice of the base classifier is flexible, and it was shown that the random projection ensemble classifier performs competitively and has desirable theoretical properties. There are two key aspects of their framework. One is that since na\"ively aggregating all projections might lead to a poor performance, they first select some ``good" projections and only aggregate these ones. The other key idea is that they tune the decision threshold instead of applying the na\"ive majority vote. These two ideas will also appear in our framework (to be proposed). To make the random projection include more important features in the linear combinations, \cite{mukhopadhyay2019targeted} proposed a targeted random projection ensemble approach, which includes each variable with probability proportional to their marginal utilities.

Another example of ensemble classification is the random subspace method, which was first studied in the context of decision trees  \citep{ho1998random}. As the name suggests, it randomly selects a feature subset and grows each tree within the chosen subspace. A similar idea is used in random forest when we restrict the splitting of each tree to a subset of features. The random subspace method is closely related to other aggregation-based approaches, including the bootstrap procedure for features \citep{boot2020subspace}. Also, as \cite{cannings2017random} pointed out, the random subspace method can be regarded as the random projection ensemble classification method when the projection space is restricted to be axis-aligned. Compared to other ensemble approaches, the random subspace method keeps the data structure via sticking to the original features, which can be helpful for interpretation and provide a direct way for feature ranking. It has been coupled with various base classifiers, including linear discriminant analysis \citep{skurichina2002bagging}, $k$-nearest neighbor classifier \citep{bay1998combining}, and combined with other techniques such as  boosting \citep{garcia2008boosting}. A related approach is random partition \citep{ahn2007classification}, where the whole space is partitioned into multiple parts. \cite{bryll2003attribute} introduced optimization ideas into the framework of the random subspace method, and selected optimal subspaces by evaluating how the corresponding fitted models performed on the training data. Despite these developments, most existing works do not have theoretical support, and the research on the link between random subspace and feature ranking is scarce to the best of our knowledge. Furthermore, the existing literature usually considers the ensemble of all generated random subspaces, which may not be a wise idea in the sparse classification scheme as many random subspaces will contain no signals. Our new ensemble framework on random subspaces is designed to tackle the sparse classification problems with theoretical guarantees. Instead of naively aggregating all generated random subspaces, we divide them into groups and only keep the ``optimally" performing subspace inside each group to construct the ensemble classifier. 

Feature ranking and selection are of critical importance in many real-world applications. For example, in disease diagnosis, beyond getting an accurate prediction for patients, we are also interested in understanding how each feature contributes to our prediction, which can facilitate the advancement of medical science. It has been widely acknowledged that in many high-dimensional classification problems, we only have a handful of useful features, with the rest being irrelevant ones. This is sometimes referred to as the sparse classification problem, which we briefly review next. \cite{bickel2004some} showed that linear discriminant analysis (LDA) is equivalent to random guessing in the worst scenario when the sample size is smaller than the dimensionality. Exploiting the underlying sparsity plays a significant role in improving the performance of the classic methods, including the LDA and the quadratic discriminant analysis (QDA) \citep{mai2012direct, jiang2018direct, hao2018model,fan2012road,  shao2011sparse, fan2015innovated, li2015sparse}. While those methods work well under their corresponding models, it is not clear how to conduct feature ranking with other types of base classifiers.  In this work, we propose a flexible ensemble classification framework,  named \emph{Random Subspace Ensemble (RaSE)}, which can be combined with any base classifiers and provide feature ranking as a by-product. 

RaSE is a flexible ensemble classification framework, the main mechanism of which is briefly described as below. Suppose the observation pair $(\bm{x}, y)$ takes values from $\mathcal{X} \times \{0, 1\}$, where $\mathcal{X}$ is an open subset of $\mathbb{R}^p$, $p$ is a positive integer and $y$ is the class label. Assume the training set consists of $n$ observation pairs $\{(\bm{x_i}, y_i), i = 1,\ldots,n\}$.  We use $C_n^{S-\ty}(\bm{x}) \in \{0, 1\}$ to represent the prediction result of the classifier trained with only features in subset $S$ when the base classifier is $\ty$. For the $j$-th ($j\in\{1,\ldots, B_1\}$) weak learner, $B_2$ random subspaces $\{S_{jk}\}_{k=1}^{B_2}$ are generated and the optimal one $S_{j*}$ is selected according to some criterion to be specified. Then this weak learner will be trained by using only the slice of training samples in this subspace $S_{j*}$. Finally the $B_1$  weak classifiers $C_n^{S_{1*}-\ty}, \ldots, C^{S_{B_1*}-\ty}_n$ are aggregated to form the decision function
\begin{equation}\label{eq: decision}
	C^{RaSE}_n(\bm{x}) = \mathds{1}\left(\frac{1}{B_1}\sum_{j=1}^{B_1}C_n^{S_{j*}-\ty}(\bm{x}) > \alpha\right),
\end{equation}
where $\alpha$ is a threshold to be determined. This framework contributes to the research of ensemble method and feature ranking in the following aspects. First, it admits a flexible framework, in which any classification algorithm can serve as the base classifier. Some examples include the standard LDA, QDA, $k$-nearest neighbor classifier ($k$NN), support vector machines (SVM), and decision trees. Second, the ensemble process naturally implies a ranking of the importance of variables via the frequencies they appear in the $B_1$ subspaces $\{S_{j*}\}_{j=1}^{B_1}$. For several specific sparse classification problems, equipped with a new information criterion named \emph{ratio information criterion (RIC)}, RaSE is shown to cover the minimal discriminative set (to be defined later) for each weak learner with high probability when $B_2$ and sample size $n$ are sufficiently large. 

Although the RaSE framework shares some similarities with the random projection (RP) framework of \cite{cannings2017random}, there are several essential differences between them. First, the key workhorse behind RaSE is to search for a desirable subspace that covers the signals, which makes it amenable for feature ranking and selection. The RP framework, on the other hand, is not naturally designed for feature ranking. Second, a key condition (Assumption 2 in \cite{cannings2017random}) to guarantee the success of RP implies that using the criterion for choosing the optimal random projection, each random projection does not deviate much from the optimal one with a non-zero probability that is independent of $n$ and $p$, which may not be satisfied under the high-dimensional setting. RaSE, however, assumes a set of conditions that explicitly take into account the high-dimensionality, which leads to the screening consistency and weak consistency. Third, we propose a new information criterion RIC with its theoretical properties analyzed under the high-dimensional setting. Fourth, motivated by the stringent requirement of a large $B_2$ for the vanilla RaSE (see Sections \ref{subsec: weak cosis} and \ref{subsec: iterative rase}), we propose the iterative RaSE, which relaxes the requirement on $B_2$ by taking advantage of the feature ranking in preceding steps.

The rest of this paper is organized as follows. In Section \ref{sec: methodology}, we first introduce the RaSE algorithm, and discuss some important concepts, including minimal discriminative set and RIC. At the end of Section \ref{sec: methodology}, an iterative version of the RaSE algorithm is presented. In Section \ref{sec: theory}, theoretical properties of RaSE and RIC are investigated, including the impact of $B_1$ on the risk and Monte-Carlo variance of RaSE classifier, the screening consistency and weak consistency of RIC, the upper bound of expected misclassification rate, and the theoretical analysis for iterative RaSE algorithm. In Section \ref{sec::computation}, we discuss several important computational issues in the RaSE algorithm, tuning parameter selection, and how to apply the RaSE framework for feature ranking. Section \ref{sec: numerical} focuses on numerical studies in terms of extensive simulations and several real data applications, through which we compare RaSE with various competing methods. The results frequently feature RaSE classifiers among the top-ranked methods and also shows its effectiveness in feature ranking. Finally, we summarize our contributions and point out a few potential directions for future work in Section \ref{sec::discussion}. We present some additional results for empirical studies in Appendix \ref{appendix b: additional fig}, and all proofs are relegated to Appendix \ref{appendix a: proof}.

\section{Methodology}\label{sec: methodology}
Recall that we have $n$ pairs of observations $\{(\bm{x}_i,y_i),i=1,\dots,n\} \stackrel{i.i.d.}{\sim}(\bm{x}, y)\in \mathcal{X} \times \{0, 1\}$, where $\mathcal{X}$ is an open subset of $\mathbb{R}^p$, $p$ is a positive integer and $y\in\{0,1\}$ is the class label. We use $S_{\textrm{Full}} = \{1,\ldots, p\}$ to represent the whole feature set. We assume the marginal densities of $\bm{x}$ for class 0 ($y=0$) and 1 ($y=1$) exist and are denoted as $f^{(0)}$ and $f^{(1)}$, respectively. The corresponding probability measures they induce are denoted as $\p^{(0)}$ and $\p^{(1)}$. Thus, the joint distribution of $(\bm{x},y)$ can be described in the following mixture model
\begin{equation}\label{eq: model}
	\bx|y = y_0 \sim (1-y_0)f^{(0)} + y_0f^{(1)}, y_0 = 0,1,
\end{equation}
where $y$ is a Bernoulli variable with success probability $\pi_1=1-\pi_0\in(0,1)$. For any subspace $S$,  we use $|S|$ to denote its cardinality. Denote $\p^{\bx}$ as the probability measure induced by the marginal distribution of $\bm{x}$, which is in fact $\pi_0\p^{(0)} + \pi_1\p^{(1)}$. When restricting to the feature subspace $S$, the corresponding marginal densities of class 0 and 1 are denoted as $f^{(0)}_S$ and $f^{(1)}_S$, respectively. 


\subsection{Random Subspace Ensemble Classification (RaSE)}\label{sec: main}

 Motivated by \cite{cannings2017random}, to train each weak learner (e.g., the $j$-th one), $B_2$ independent random subspaces are generated as $S_{j1}, \ldots, S_{jB_2}$. Then, according to some specific criterion (to be introduced in Section \ref{sec: ric}), the optimal subspace $S_{j*}$ is selected and the $j$-th weak learner is trained only in $S_{j*}$. Subsequently, $B_1$ such weak classifiers $\{C_n^{S_{j*}-\ty}\}_{j=1}^{B_1}$ are obtained. Finally, we aggregate outputs of $\{C_n^{S_{j*}-\ty}\}_{j=1}^{B_1}$ to form the final decision function by taking a simple average. 
The whole procedure can be summarized in the following algorithm. 
 
 \begin{algorithm}
\caption{Random subspace ensemble classification (RaSE)}
\label{algo}
\KwIn{training data $\{(\bm{x}_i, y_i)\}_{i = 1}^n$, new data $\bm{x}$, subspace distribution $\mathcal{D}$, criterion $\mathcal{C}$, integers $B_1$ and $B_2$, type of base classifier $\mathcal{T}$}
\KwOut{predicted label $C^{RaSE}_n(\bx)$, the selected proportion of each feature $\bm{\eta}$}
Independently generate random subspaces $S_{jk} \sim \mathcal{D}, 1 \leq j \leq B_1, 1 \leq k \leq B_2$\\
\For{$j\leftarrow 1$ \KwTo $B_1$}{
  Select the optimal subspace $S_{j*}$ from $\{S_{jk}\}_{k = 1}^{B_2}$ according to $\mathcal{C}$ and $\mathcal{T}$
}
Construct the ensemble decision function $\nu_n(\bm{x}) = \frac{1}{B_1}\sum_{j = 1}^{B_1}C_n^{S_{j*}-\ty}(\bm{x})$ \\
{Set the threshold $\hat{\alpha}$ according to \eqref{eq: thresholding}}\\
Output the predicted label $C^{RaSE}_n(\bx) = \mathds{1}(\nu_n(\bm{x}) > \hat{\alpha})$, the selected proportion of each feature $\bm{\eta}=(\eta_1,\ldots,\eta_p)^T$ where $\eta_l=B_1^{-1}\sum_{j=1}^{B_1}\mathds{1}(l\in S_{j*}), l=1,\ldots,p$
\end{algorithm}
In Algorithm \ref{algo}, the subspace distribution $\mathcal{D}$ is chosen as a \emph{hierarchical uniform distribution} over the subspaces by default. Specifically, with $D$ as the upper bound of the subspace size\footnote{How to set $D$ in practice will be discussed in  Section \ref{subsec: tune parameter}.}, we first generate the subspace size $d$ from the uniform distribution over $\{1, \ldots, D\}$. Then, the subspace $S_{11}$ follows the uniform distribution over $\{S \subseteq S_{\textrm{Full}}: |S| = d\}$.  In practice, the subspace distribution can be adjusted if we have prior information about the data structure. 

In Step 6 of Algorithm \ref{algo}, we set the decision threshold to minimize the empirical classification error on the training set,  \begin{equation}\label{eq: thresholding}
	\hat{\alpha} = \arg\min_{\alpha \in [0,1]}[\hat{\pi}_0 (1-\hat{G}_n^{(0)}(\alpha)) + \hat{\pi}_1 \hat{G}_n^{(1)}(\alpha)],
\end{equation}
where
\begin{align}
	n_r &= \sum_{i=1}^n \mathds{1}(y_i = r), r = 0, 1, \\
	\hat{\pi}_r &= \frac{n_r }{n}, r = 0, 1,\\
	\hat{G}_n^{(r)}(\alpha) &=  \frac{1}{n_r}\sum_{i=1}^{n}\mathds{1}(y_i = r)\mathds{1}(\nu_n(\bm{x}_i) \leq \alpha), r = 0, 1,\\
	\nu_n(\bm{x}_i) &= \frac{1}{B_1}\sum_{j=1}^{B_1}\mathds{1}(C_n^{S_{j*}}(\bm{x}_i) = 1).
\end{align}

\subsection{Minimal Discriminative Set}\label{sec: min set}
For sparse classification problems, it is of significance to accurately separate signals from noises. Motivated by \cite{kohavi1997wrappers} and \cite{zhang2011bic}, we define the discriminative set and study some of its properties as follows.
\begin{definition}\label{def: discriminative set}
	A feature subset $S$ is called a \textbf{discriminative set} if $y$ is conditionally independent with $\bx_{S^c}$ given $\bx_S$, where $S^c = \fullset\setminus S$. We call $S$ a \textbf{minimal discriminative set} if it has minimal cardinality among all discriminative sets.
\end{definition}

\begin{assumption}\label{asmp: density}
     The densities $f^{(0)}$ and $f^{(1)}$ have the same support a.s. with respect to $\p^{\bx}$.
\end{assumption}

\begin{remark}
	Note that the existence of densities and the common support requirement are not necessary for the definition of the minimal discriminative set and the RaSE framework. We focus on the continuous distribution purely for notation convenience. We will discuss this assumption again after introducing our information criterion in Definition \ref{def: ric}.
\end{remark}

\begin{proposition}\label{prop: discriminative set}
Under Assumption \ref{asmp: density}, we can characterize the discriminative set using the marginal density ratio due to the following two facts.
\begin{enumerate}[label=(\roman*)]
	\item If $S$ is a discriminative set, then \[
		\frac{f^{(1)}(\bm{x})}{f^{(0)}(\bm{x})} = \frac{f^{(1)}_S(\bm{x}_S)}{f^{(0)}_S(\bm{x}_S)}
	\] almost surely with respect to $\p^{\bx}$.
	\item If for a feature subset $S$, there exists a function $h: \mathbb{R}^{|S|} \rightarrow [0, +\infty]$ such that \[
		\frac{f^{(1)}(\bm{x})}{f^{(0)}(\bm{x})} = h(\bm{x}_S)
	\] almost surely with respect to $\p^{\bx}$, then $S$ is a discriminative set and \[
	h(\bm{x}_S) = \frac{f^{(1)}_S(\bm{x}_S)}{f^{(0)}_S(\bm{x}_S)}
	\] almost surely with respect to $\p^{\bx}$.
\end{enumerate}
\end{proposition}

In general, there may exist more than one minimal discriminative sets. For instance, if two features are exactly the same, then more than one minimal discriminative sets may exist since we cannot distinguish between them. To rule out this type of degenerate scenario, we impose the uniqueness assumption for the minimum discriminative set. 

\begin{assumption}\label{asmp: uniqueness}
   There is only one minimal discriminative set, which is denoted as $S^*$. In addition, all discriminative sets cover $S^*$.
\end{assumption}

In the classification problem, we are often interested in the risk of a classifier $C$. With the 0-1 loss, the risk or the misclassification rate is defined as \[
	R(C) = \e[\mathds{1}(C(\bm{x}) \neq y)] = \p(C(\bm{x}) \neq y).
\]The Bayes classifier
\begin{equation}\label{eq: bayes classifier}
	C_{Bayes}(\bm{x}) = \begin{cases}
		1, & \p(y = 1|\bm{x}) > \frac{1}{2}, \\
		0, & \textrm{otherwise}.
	\end{cases}
\end{equation}
is known to achieve the minimal risk among all classifiers \citep{devroye2013probabilistic}. If only features in $S$ are used, it will provide us a ``local" Bayes classifier $C^S_{Bayes}(\bm{x}_S)$ which achieves the minimal risk among all classifiers using only features in $S$. In general, there is no guarantee that $R(C_{Bayes}^S) = R(C_{Bayes})$. Fortunately, the equation holds when $S$ is a discriminative set.
\begin{proposition}\label{prop: mds Bayes error}
	For any discriminative set $S$,  it holds that
	\begin{equation}
		R(C_{Bayes}^S) = R(C_{Bayes}^{S^*}) = R(C_{Bayes}).
	\end{equation}
\end{proposition}
This direct result illustrates that covering $S^*$ is sufficient to obtain performance as good as the  Bayes classifier.

To clarify the notions above, we next take the two-class Gaussian settings as  examples and investigate what $S^*$ is in each case.
\begin{example}[LDA]\label{exp: lda}
	Suppose $f^{(0)} \sim N(\bm{\mu}^{(0)}, \Sigma), f^{(1)} \sim N(\bm{\mu}^{(1)}, \Sigma)$, where $\Sigma$ is positive definite. The log-density ratio is
	\begin{equation}
		\log\left(\frac{f^{(0)}(\bm{x})}{f^{(1)}(\bm{x})}\right) = C - (\bm{\mu}^{(1)} - \bm{\mu}^{(0)})^T\Sigma^{-1}\bm{x},
	\end{equation}
	where $C$ is a constant independent of $\bm{x}$. By Proposition \ref{prop: discriminative set}, the minimal discriminative set $S^* = \{j: [\Sigma^{-1}(\bm{\mu}^{(1)} - \bm{\mu}^{(0)})]_j \neq 0\}$.
\end{example}
This definition is equivalent to that in \cite{mai2012direct} for the LDA case. 
\begin{example}[QDA]\label{exp: qda} 
Suppose $f^{(0)} \sim N(\bm{\mu}^{(0)}, \Sigma^{(0)})$, $f^{(1)} \sim N(\bm{\mu}^{(1)}, \Sigma^{(1)})$, where $\Sigma^{(0)}$ and $\Sigma^{(1)}$ are positive definite matrices with $\Sigma^{(0)} \neq \Sigma^{(1)}$, then the log-density ratio is 
\begin{equation}\label{eq: qda}
	\log\left(\frac{f^{(0)}(\bm{x})}{f^{(1)}(\bm{x})}\right) = C + \frac{1}{2}\bm{x}^T\left[(\Sigma^{(1)})^{-1} - (\Sigma^{(0)})^{-1}\right]\bm{x} + \left[(\Sigma^{(0)})^{-1}\bm{\mu}^{(0)} - (\Sigma^{(1)})^{-1}\bm{\mu}^{(1)}\right]^T\bm{x},
\end{equation}
where $C$ is a constant independent of $\bm{x}$. Let $S^*_l = \{j: [(\Sigma^{(0)})^{-1}\bm{\mu}^{(0)} - (\Sigma^{(1)})^{-1}\bm{\mu}^{(1)}]_j \neq 0\}$, $S^*_q = \{j: [(\Sigma^{(1)})^{-1} - (\Sigma^{(0)})^{-1}]_{ij} \neq 0, \exists i\}$. The  elements in $S^*_l$ are often called variables with main effects while elements in $S^*_q$ are called variables with quadratic effects \citep{hao2018model, fan2015innovated, jiang2018direct}. By Proposition \ref{prop: discriminative set}, the minimal discriminative set $S^* = S^*_l \cup S^*_q$.
\end{example}

\begin{proposition}\label{prop: uni discriminative set for gaussian}
	If $f^{(0)} \sim N(\bm{\mu}^{(0)}, \Sigma^{(0)}), f^{(1)} \sim N(\bm{\mu}^{(1)},\Sigma^{(1)})$, where $\Sigma^{(0)}$ and $\Sigma^{(1)}$ are positive definite matrices, then the following conclusions hold:
	\begin{enumerate}[label=(\roman*)]
		\item The minimal discriminative set $S^*$ is unique;
		\item For any discriminative set $S$, we have  $S\supseteq S^*$;
		\item Any set $S \supseteq S^*$ is a discriminative set. This conclusion also holds without the Gaussian assumption.
	\end{enumerate}
\end{proposition}

\subsection{Ratio Information Criterion (RIC)}\label{sec: ric}
As discussed in Section \ref{sec: min set}, it is desirable to identify the minimal discriminative set $S^*$ for the classifier to achieve a low misclassification rate. Hence in Algorithm \ref{algo}, it is important to apply a proper criterion to select the ``optimal" subspace. In the variable selection literature, a criterion enjoying the property of correctly selecting the minimal discriminative set with high probability is often referred to as a ``consistent" one. For model \eqref{eq: model}, \cite{zhang2011bic} proved that BIC, in conjunction with a backward elimination procedure, is selection consistent in the Gaussian mixture case. However, the BIC they investigated involved the joint log-likelihood function for $(\bm{x}, y)$, which involves estimating high-dimensional covariance matrices that could be problematic when $p$ is close to or larger than  $n$ without additional structural assumptions. 

Here we propose a new criterion, which enjoys the weak consistency under the general setting of \eqref{eq: model}, based on Kullback-Leibler divergence \citep{kullback1951information}. Two asymmetric Kullback-Leibler divergences for densities $f$ and $g$ are defined as 
\[
	\textup{KL}(f||g) = \e_{\bm{x}\sim f}\left[\log\left(\frac{f(\bm{x})}{g(\bm{x})}\right)\right],
	\textup{KL}(g||f) = \e_{\bm{x}\sim g}\left[\log\left(\frac{g(\bm{x})}{f(\bm{x})}\right)\right],
\]
where $E_{\bm{x} \sim f}$ represents taking expectation with respect to $\bm{x} \sim f$.
In binary classification model \eqref{eq: model}, marginal probabilities can be crucial because imbalanced marginal distributions can significantly compromise the performance of most standard learning algorithms \citep{he2009learning}. Therefore,  we consider a weighted version of two KL divergences for the marginal distributions $f^{(0)}_S$ and $f^{(1)}_S$ with subspace $S$, i.e. $\pi_0 \mbox{KL}(f^{(0)}_S||f^{(1)}_S) + \pi_1 \mbox{KL}(f^{(1)}_S||f^{(0)}_S)$. Denote by $\hat{f}^{(0)}_S, \hat{f}^{(1)}_S, \hat{\pi}_0, \hat{\pi}_1$ the estimated version via MLEs of parameters, then it holds that
\begin{align}
	\hat{\pi}_0 \widehat{\textup{KL}}(f^{(0)}_S||f^{(1)}_S) &= n^{-1}\sum_{i = 1}^n \mathds{1}(y_i = 0)\cdot \log\left[\frac{\hat{f}^{(0)}_S(\bm{x}_{i, S})}{\hat{f}^{(1)}_S(\bm{x}_{i, S})}\right], \\
	\hat{\pi}_1  \widehat{\textup{KL}}(f^{(1)}_S||f^{(0)}_S) &= n^{-1}\sum_{i = 1}^n \mathds{1}(y_i = 1)\cdot \log\left[\frac{\hat{f}^{(1)}_S(\bm{x}_{i, S})}{\hat{f}^{(0)}_S(\bm{x}_{i, S})}\right].
\end{align}
Now, we are ready to introduce the following new criterion named \emph{ratio information criterion} (RIC), with a proper penalty term. 
\begin{definition}\label{def: ric}
	For model \eqref{eq: model}, the ratio information criterion (RIC) for feature subspace $S$ is defined as
\begin{equation}\label{eq: def ric}
	\textup{RIC}_n(S) = - 2[\hat{\pi}_0 \widehat{\textup{KL}}(f^{(0)}_S||f^{(1)}_S) + \hat{\pi}_1  \widehat{\textup{KL}}(f^{(1)}_S||f^{(0)}_S)] + c_n \cdot \textup{deg}(S),
\end{equation}
where $c_n$ is a function of sample size $n$ and $\textup{deg}(S)$ is the degree of freedom corresponding to the model with subspace $S$. 
\end{definition}
\begin{remark} Assumption \ref{asmp: density} is necessary to make sure  RIC is well-defined.  To see this, note that 
the existence of both Kullback-Leibler divergences requires (1) $f^{(0)}(\bx) = 0\Rightarrow f^{(1)}(\bx) = 0$ a.s. with respect to $\p^{(1)}$; (2)  $f^{(1)}(\bx) = 0\Rightarrow f^{(0)}(\bx) = 0$ a.s. with respect to $\p^{(0)}$. This is equivalent to Assumption \ref{asmp: density} when $\pi_0,\pi_1 > 0$. 
\end{remark}

Note that although AIC is also motivated by the Kullback-Leibler divergence, it aims to minimize the KL divergence between the estimated density and the true density  \citep{burnham1998practical}. In our case, however, the goal is to maximize the KL divergence between the conditional  densities under two classes to achieve a greater separation.  

Next, let's work out a few familiar examples where explicit expressions exist for RIC. 

\begin{proposition}[RIC for the LDA model]\label{prop: RIC for LDA case}
	Suppose $f^{(0)} \sim N(\bm{\mu}^{(0)}, \Sigma), f^{(1)} \sim N(\bm{\mu}^{(1)}, \Sigma)$, where $\Sigma$ is positive definite. The MLEs of the parameters are
\begin{align*}
	\hat{\bm{\mu}}_S^{(r)} &= \frac{1}{n_r}\sum_{i = 1}^n \mathds{1}(y_i =r)\bm{x}_{i, S}, r = 0, 1, \\
	\hat{\Sigma}_{S, S} &= \frac{1}{n} \sum_{i = 1}^n\sum_{r = 0}^1  \mathds{1}(y_i =r)\cdot (\bm{x}_{i,S} - \hat{\bm{\mu}}_S^{(r)})(\bm{x}_{i,S} - \hat{\bm{\mu}}_S^{(r)})^T,
\end{align*}
Then we have
	\begin{equation}
		\textup{RIC}_n (S) = - (\hat{\bm{\mu}}_{S}^{(1)} - \hat{\bm{\mu}}_{S}^{(0)})^T\hat{\Sigma}_{S, S}^{-1}(\hat{\bm{\mu}}_{S}^{(1)} - \hat{\bm{\mu}}_{S}^{(0)}) + c_n(|S|+1).
	\end{equation}
\end{proposition}

\begin{proposition}[RIC for the QDA model]\label{prop: RIC for QDA case}
Suppose $f^{(0)} \sim N(\bm{\mu}^{(0)}, \Sigma^{(0)}), f^{(1)} \sim N(\bm{\mu}^{(1)}, \Sigma^{(1)})$, where $\Sigma^{(0)}$, $\Sigma^{(1)}$ are positive definite but not necessarily equal. 
The MLEs of the estimators are as follows. $\{\hat{\bmu}_S^{(r)}, r=0, 1\}$  are the same as in Proposition \ref{prop: RIC for LDA case}, and 
\begin{align*}
	\hat{\Sigma}_{S, S}^{(r)} = \frac{1}{n_r} \sum_{i = 1}^n \mathds{1}(y_i =r)\cdot (\bm{x}_{i,S} - \hat{\bm{\mu}}_S^{(r)})(\bm{x}_{i,S} - \hat{\bm{\mu}}_S^{(r)})^T, r = 0, 1. 
\end{align*}
Then we have
	\begin{align}\label{RIC for the QDA model}
		\textup{RIC}_n (S) &= - (\hat{\bm{\mu}}_{S}^{(1)} - \hat{\bm{\mu}}_{S}^{(0)})^T[\hat{\pi}_1(\hat{\Sigma}_{S, S}^{(0)})^{-1} + \hat{\pi}_0(\hat{\Sigma}_{S, S}^{(1)})^{-1}](\hat{\bm{\mu}}_{S}^{(1)} - \hat{\bm{\mu}}_{S}^{(0)}) \nonumber\\
		&\quad+ \textup{Tr}[((\hat{\Sigma}_{S, S}^{(1)})^{-1} - (\hat{\Sigma}_{S, S}^{(0)})^{-1})(\hat{\pi}_1\hat{\Sigma}_{S, S}^{(1)} - \hat{\pi}_0\hat{\Sigma}_{S, S}^{(0)})]+ (\hat{\pi}_1 - \hat{\pi}_0)(\log |\hat{\Sigma}_{S, S}^{(1)}| - \log|\hat{\Sigma}_{S, S}^{(0)}|) \nonumber\\ 
		&\quad+ c_n\cdot \left[\frac{|S|(|S| + 3)}{2} + 1\right].
	\end{align}
\end{proposition}

Note that the primary term of RIC for the LDA case, is the Mahalanobis distance \citep{mclachlan1999mahalanobis}, which is closely related to the Bayes error of LDA classifier \citep{efron1975efficiency}. And for the QDA case, the KL divergence components contain three terms. The first term is similar to the Mahalanobis distance, representing the contributions of linear signals to the classification model. And the second and third terms represent the contributions of quadratic signals.

Note that the KL divergence can also be estimated by non-parametric methods including the $k$-nearest neighbor distance \citep{wang2009divergence, ganguly2018nearest}, which may sometimes lead to more robust estimates than the parametric ones in our numerical experiments. Specifically, consider two samples $\{\bx^{(0)}_{i, S}\}_{i=1}^{n_0} \overset{i.i.d.}{\sim} f^{(0)}_S$ and $\{\bx^{(1)}_{i, S}\}_{i=1}^{n_1} \overset{i.i.d.}{\sim} f^{(1)}_S$, and write $\rho_{k_0, 0}(\bx_{j, S}^{(0)})$ for the Euclidean distance between $\bx_{j, S}^{(0)}$ and its $k_0$-th nearest neighbor in the sample $\{\bx^{(0)}_{i, S}\}_{i=1}^{n_0} \backslash \{\bx_{j, S}^{(0)}\}$, and write $\rho_{k_1, 1}(\bx_{j, S}^{(0)})$ for the Euclidean distance between $\bx_{j, S}^{(0)}$ and its $k_1$-th nearest neighbor in the sample $\{\bx^{(1)}_{i, S}\}_{i=1}^{n_1}$.  \cite{wang2009divergence} and \cite{ganguly2018nearest} defined the following asymptotic unbiased estimator given $k_0$ and $k_1$:
\begin{equation}\label{eq: nonpara kl}
	\widehat{\kl}(f^{(0)}_S || f^{(1)}_S) = \frac{|S|}{n_0}\sum_{i=1}^{n_0}\log \left(\frac{\rho_{k_0, 0}(\bx_{i, S}^{(0)})}{\rho_{k_1, 1}(\bx_{i, S}^{(0)})}\right) + \log\left(\frac{n_1}{n_0-1}\right) + \Psi(k_0) - \Psi(k_1),
\end{equation}
where $\Psi$ denotes the diGamma function \citep{abramowitz1948handbook}. Similarly we can obtain estimate $\widehat{\kl}(f_S^{(1)} || f_S^{(0)})$. Besides, \cite{berrett2019efficient} proposed a weighted estimator based on \eqref{eq: nonpara kl} and investigated its efficiency. We will compare the performance of RaSE when using the estimate in \eqref{eq: nonpara kl} to that using parametric methods in simulation.

Another line of research for classification is to study the conditional distribution of $y|\bm{x}$. For this setup, there has been a rich literature on various information type criteria that involves the conditional log-likelihood function. Akaike information criterion (AIC) \citep{akaike1973information} was shown to be inconsistent. It was demonstrated that Bayesian information criterion (BIC) is consistent under certain regularity conditions \citep{rao1989strongly}. \cite{chen2008extended, chen2012extended} modified the definition of conventional BIC to form the extended BIC (eBIC) for the high-dimensional setting where $p$ grows at a polynomial rate of $n$. \cite{fan2013tuning} proved the consistency of a generalized information criterion (GIC) for generalized linear models in ultra-high dimensional space.
\subsection{Iterative RaSE}
The success of the RIC proposed in Section \ref{sec: ric} relies on the assumption that the minimal discriminative set $S^*$ appears in some of the $B_2$ subspaces for each weak learner. For sparse classification problems, the size of $S^*$ can be very small compared to $p$. When $p$ is large, the probability of generating a subset that covers $S^*$ is quite low according to the hierarchical uniform distribution for subspaces. It turns out by using the selected frequency of each feature in $B_1$ subspaces $\{S_{j*}\}_{j=1}^{B_1}$ from Algorithm \ref{algo}, we can improve the RaSE algorithm by running the RaSE algorithm again with a new hierarchical  distribution for subspaces. In particular, we first calculate the percentage vector $\bm{\eta}=(\eta_1,\ldots,\eta_p)^T$ representing the proportion of each feature appearing among  $B_1$ subspaces $\{S_{j*}\}_{j=1}^{B_1}$, where $\eta_l=B_1^{-1}\sum_{j=1}^{B_1}\mathds{1}(l\in S_{j*}), l=1, \ldots, p$. The new hierarchical distribution is specified as follows. In the first step, we generate the subspace size $d$ from the uniform distribution over $\{1, \ldots, D\}$ as before. Before moving on to the second step, note that each subspace $S$ can be equivalently represented as $\bm{J}=(J_1,\ldots, J_p)^T$,  where $J_l = \mathds{1}(l\in S), l=1,\ldots,p$. Then, we generate $\bm{J}$ from a restrictive multinomial distribution with parameter $(p, d, \tilde{\bm{\eta}})$, where $\tilde{\bm{\eta}} = (\tilde{\eta}_1, \ldots, \tilde{\eta}_p)^T $, $\tilde{\eta}_l = \eta_l\mathds{1}(\eta_l > C_0/\log p) + \frac{C_0}{p}\mathds{1}(\eta_l \leq C_0/\log p)$, and the restriction is that $J_l\in\{0,1\}, l=1,\ldots,p$. Here $C_0$ is a constant.

%
Intuitively, this strategy can be repeated to increase the probability that signals in $S^*$ are covered in the subspaces we generate. It could also lead to an improved feature ranking according to the updated proportion of each feature $\bm{\eta}$. This will be verified via multiple simulation experiments in Section \ref{sec: numerical}. This iterative RaSE algorithm is summarized in Algorithm \ref{algo_iteration}. 

A related idea was introduced by \cite{mukhopadhyay2019targeted} to generate random projections with probabilities proportional to the marginal utilities. One major difference in RaSE is that the feature importance is determined via their joint contributions in the selected subspaces. 

\begin{algorithm}[!h]
\SetAlgoLined
\caption{Iterative RaSE ($\mbox{RaSE}_T$)}
\label{algo_iteration}
\KwIn{training data $\{(\bm{x}_i, y_i)\}_{i = 1}^n$, new data $\bm{x}$, initial subspace distribution $\mathcal{D}^{(0)}$, criterion $\mathcal{C}$, integers $B_1$ and $B_2$, the type of base classifier $\mathcal{T}$, the number of iterations $T$}
\KwOut{predicted label $C^{RaSE}_n(\bx)$, the proportion of each feature $\bm{\eta}^{(T)}$}
\For{$t\leftarrow 0$ \KwTo $T$}{
	Independently generate random subspaces $S_{jk}^{(t)} \sim \mathcal{D}^{(t)}, 1 \leq j \leq B_1, 1 \leq k \leq B_2$\\
	\For{$j\leftarrow 1$ \KwTo $B_1$}{
  		Select the optimal subspace $S_{j*}^{(t)}$ from $\{S_{jk}^{(t)}\}_{k = 1}^{B_2}$ according to $\mathcal{C}$ and $\mathcal{T}$

	}
	Update $\bm{\eta}^{(t)}$ where $\eta_l^{(t)}=B_1^{-1}\sum_{j=1}^{B_1}\mathds{1}(l\in S_{j*}^{(t)}), l=1,\ldots,p$\\
Update $\mathcal{D}^{(t)} \leftarrow$ restrictive multinomial distribution with parameter $(p, d, \tilde{\bm{\eta}}^{(t)})$, where $\tilde{\eta}_l^{(t)} = \eta_l^{(t)}\mathds{1}(\eta_l^{(t)} > C_0/\log p) + \frac{C_0}{p}\mathds{1}(\eta_l^{(t)} \leq C_0/\log p)$ and $d$ is sampled from the uniform distribution over $\{1,\ldots, D\}$
}
{Set the threshold $\hat{\alpha}$ according to \eqref{eq: thresholding}}\\
Construct the ensemble decision function $\nu_n(\bm{x}) = \frac{1}{B_1}\sum_{j = 1}^{B_1}C_n^{S_{j*}^{(T)}-\ty}(\bm{x})$\\
Output the predicted label $C^{RaSE}_n(\bx) = \mathds{1}(\nu_n(\bm{x}) > \hat{\alpha})$ and $\bm{\eta}^{(T)}$
\end{algorithm}


\section{Theoretical Studies}\label{sec: theory}
In this section, we investigate various theoretical properties of RaSE, including {the impact of  $B_1$ on the risk of RaSE classifier} and expectation of misclassification rate. Furthermore, we will demonstrate that RIC achieves weak consistency. Throughout this section, we allow the dimension $p$ grows with sample size $n$. 

To streamline the presentation, we first introduce some additional notations. In this paper, we have three different sources of randomness: (1) the randomness of the training data, (2) the randomness of the subspaces, and (3) the randomness of the test data. We will use the following notations to differentiate them. 
\begin{itemize}
	\item Analogous to \cite{cannings2017random}, we write $\bp$ and $\be$ to represent the probability and expectation with respect to the collection of $B_1B_2$ random subspaces $\{S_{jk}: 1\leq j \leq B_1, 1 \leq k \leq B_2\}$;
	\item $\tp$ and $\te$ are used when the randomness comes from the training data $\{(\bm{x}_i, y_i)\}_{i = 1}^n$;
	\item We use $\p$ and $\e$ when considering all three sources of randomness. 
\end{itemize}
Recall that in Algorithm \ref{algo}, the decision function is
\begin{equation}
	\nu_n(\bm{x}) = \frac{1}{B_1}\sum_{j = 1}^{B_1}C_n^{S_{j*}}(\bm{x}).
\end{equation}
For a given threshold $\alpha \in (0,1)$, the RaSE classifier is
\begin{equation}
	C_n^{RaSE}(\bm{x}) = \begin{cases}	
	1, & \nu_n(\bm{x}) > \alpha, \\
	0, & \nu_n(\bm{x}) \leq \alpha.\\
 \end{cases}
\end{equation}
By the weak law of large numbers, as $B_1 \rightarrow \infty$, $\nu_n$ will converge in probability to its expectation
\begin{equation}
	\mu_n(\bm{x}) = \bp\left(C_n^{S_{1*}}(\bm{x}) = 1\right).
\end{equation}
It should be noted that here as both the training data and the criterion $\mathcal{C}$ in Algorithm \ref{algo} are fixed, $S_{1*}$ is a deterministic function of $\{S_{1k}\}_{k=1}^{B_2}$. Nevertheless, $\mu_n(\bm{x})$ is still random due to randomness of the new $\bm{x}$. Then, it is helpful to define the conditional cumulative distribution function of $\mu_n$ for class $0, 1$ respectively as $G_n^{(0)}(\alpha') = \p(\mu_n(\bm{x}) \leq \alpha' | y = 0) = \p^{(0)}(\mu_n(\bm{x}) \leq \alpha')$ and $G_n^{(1)}(\alpha') = \p(\mu_n(\bm{x}) \leq \alpha'| y = 1) = \p^{(1)}(\mu_n(\bm{x}) \leq \alpha')$. Since the distribution  $\mathcal{D}$ of subspaces is discrete, $\mu_n(\bm{x})$ takes finite unique values almost surely, implying $G_n^{(0)}, G_n^{(1)}$ to be step functions. We denote the corresponding probability mass functions of $G_n^{(0)}$ and $G_n^{(1)}$ as $g_n^{(0)}$ and $g_n^{(1)}$, respectively. Since for any $\bm{x}$, $\nu_n(\bm{x}) \rightarrow \mu_n(\bm{x})$ as $B_1 \rightarrow \infty$, we consider the following randomized RaSE classifier in population
\begin{equation}\label{eq: rm_RSE_inf_c}
		C^{RaSE*}_n(\bm{x}) = \begin{cases}
          1, & \mu_n(\bm{x}) > \alpha, \\
          0, & \mu_n(\bm{x}) < \alpha, \\
          \textrm{Bernoulli$\left(\frac{1}{2}\right)$}, & \mu_n(\bm{x}) = \alpha.
    \end{cases}
\end{equation}
as the infinite simulation RaSE classifier with $B_1 \rightarrow \infty$. 

In the following sections, we would like to study different properties of $C_n^{RaSE}$. In Section \ref{sec: mc error}, we condition on the training data and study the impact of $B_1$ via the relationship between test error of $C^{RaSE}_n$ and $C^{RaSE*}_n$ and Monte-Carlo variance $\bv(R(C^{RaSE}_n))$, which can reflect the stability of RaSE classifier as $B_1$ increases. It will be demonstrated that conditioned on training data, both the difference between the test errors of $C^{RaSE}_n$ and $C^{RaSE*}_n$, and the Monte-Carlo variance of $C^{RaSE}_n$, converge to zero as $B_1 \rightarrow \infty$ for almost every threshold $\alpha \in (0,1)$ at an exponential rate. In Section \ref{subsec: error rase}, we will prove an upper bound for the expected misclassification rate $R(C^{RaSE}_n)$ with respect to all the randomness for fixed threshold $\alpha$. 
Next, we introduce several scaling notations. For two sequences $a_n$ and $b_n$, we use $a_n = o(b_n)$ or $a_n \ll b_n$ to denote $|a_n/b_n| \rightarrow 0$; $a_n=O(b_n)$ or $a_n \lesssim b_n$ to denote $|a_n/b_n| < \infty$. The corresponding stochastic scaling notations are $o_p$ and $O_p$, where $a_n = o_p(b_n), a_n = O_p(b_n)$ imply that $|a_n/b_n| \xrightarrow{p} 0$ and for any $\epsilon > 0$, there exists $M > 0$ such that $\p(|a_n/b_n| > M) \leq \epsilon$, $\forall n$. Also, we use $\lambda_{\min}(A)$ and $\lambda_{\max}(A)$ to denote the smallest and largest eigenvalues of a square matrix $A$. For any vector $\bx = (x_1, \ldots, x_p)^T$, the Euclidean norm $\|\bx\|=\sqrt{\sum_i x_i^2}$. And for any matrix $A = (a_{ij})_{p\times p}$, we define the operator norm $\twonorm{A} = \sup\limits_{\twonorm{\bx} = 1}\twonorm{A\bx}$, the infinity norm $\infnorm{A} = \sup\limits_{i}\sum_{j=1}^p|a_{ij}|$, the maximum norm $\maxnorm{A} = \max\limits_{i,j}|a_{ij}|$, and the Frobenius norm $\fnorm{A} = \sqrt{\sum_{i,j}a_{ij}^2}$.


\subsection{Impact of $B_1$}\label{sec: mc error}
In this section, we study the impact of $B_1$ by presenting  upper bounds of the absolute difference between the test error of $C^{RaSE}_n$ and $C^{RaSE*}_n$, and Monte-Carlo variance for RaSE when conditioned on the training data as $B_1$ grows. Note that the discrete distribution of random subspaces leads to both bounds vanishing at exponential rates, except for a finite set of thresholds $\alpha$, which is very appealing.

\begin{theorem}[Risk for the RaSE classifier conditioned on training data]\label{thm: MC_error_bdd}
 Denote $G_n(\alpha') = \pi_1 G_n^{(1)}(\alpha') - \pi_0 G_n^{(0)}(\alpha')$ and $\{\alpha_i\}_{i = 1}^N$ represents the discontinuity points of $G_n$. Given training samples with size $n$, we have the following bound for expected misclassification rate of RaSE classifier with threshold $\alpha$ when $B_1 \rightarrow \infty$ as
\begin{equation}
	|\be[R(C_n^{RaSE})] - R(C_n^{RaSE*})| \leq \begin{cases} 
          O\left(\frac{1}{\sqrt{B_1}}\right), & \alpha  \in \{\alpha_i\}_{i = 1}^N,\\
          \exp\left\{-C_{\alpha}B_1\right\}, & otherwise, \\
    \end{cases}
\end{equation}
where $C_{\alpha} = 2\min\limits_{1 \leq i \leq N}(|\alpha - \alpha_i|^2)$.
\end{theorem}

It shows that as $B_1 \rightarrow \infty$, the RaSE classifier $C_n^{RaSE}$ and its infinite simulation version $C_n^{RaSE*}$ achieve the same expected misclassification rate conditioned on the training data. Many similar results about the excess risk of randomized ensembles have been studied in literature \citep{cannings2017random, lopes2019estimating, lopes2020estimating}.  Next, we provide a similar upper bound for the MC variance of the RaSE classifier. Suppose the discontinuity points of $G_n^{(0)}$ and $G_n^{(1)}$ are $\{\alpha^0_i\}_{i = 1}^{N_0}$ and $\{\alpha^1_j\}_{j = 1}^{N_1}$, respectively.

\begin{theorem}[MC variance for the RaSE classifier]\label{thm: var}
	 It holds that
	 \begin{equation}
	\bv[R(C_n^{RaSE})] \leq  \begin{cases} 
          \frac{1}{2}\left[\pi_0(g^{(0)}_n(\alpha))^2 + \pi_1(g^{(1)}_n(\alpha))^2\right] + O\left(\frac{1}{\sqrt{B_1}}\right), & \alpha \in \{\alpha^{(0)}_i\}_{i = 1}^{N_0} \cup \{\alpha^{(1)}_j\}_{j = 1}^{N_1}\\
          \exp\left\{-C_{\alpha}B_1\right\}, & \text{otherwise}, \\
    \end{cases}
	\end{equation}
	where $C_{\alpha} = 2\min\limits_{\substack{1 \leq i \leq N_0 \\ 1 \leq j \leq N_1}}(|\alpha - \alpha^{(0)}_i|^2, |\alpha - \alpha^{(1)}_j|^2)$.
\end{theorem}
This theorem asserts that except for a finite set of threshold $\alpha$, the MC variance of RaSE classifier shrinks to zero at an exponential rate. 

\subsection{Theoretical Properties of RIC}\label{subsec: weak cosis}

An important step of the RaSE classifier is the choice of an ``optimal" subspace among $B_2$ subspaces for each of the $B_1$ weak classifiers. Before showing the screening consistency and weak consistency of RIC, we first present a proposition that explains the intuition of why RIC can succeed.

\begin{proposition}\label{prop: KL}
When Assumptions \ref{asmp: density} and \ref{asmp: uniqueness}  hold,  we have the following conclusions:
	\begin{enumerate}[label=(\roman*)]
	\item $\textup{KL}(f_S^{(0)}||f_S^{(1)}) = \textup{KL}(f_{S^*}^{(0)}||f_{S^*}^{(1)}), \textup{KL}(f_S^{(1)}||f_S^{(0)}) = \textup{KL}(f_{S^*}^{(1)}||f_{S^*}^{(0)})$ hold for any $S \supseteq S^*$;
	\item $\pi_0\textup{KL}(f_S^{(0)}||f_S^{(1)}) +  \pi_1\textup{KL}(f_S^{(1)}||f_S^{(0)}) < \pi_0\textup{KL}(f_{S^*}^{(0)}||f_{S^*}^{(1)}) + \pi_1\textup{KL}(f_{S^*}^{(1)}||f_{S^*}^{(0)})$ if $S \not\supseteq S^*$;
	\item $\textup{KL}(f_{S^*}^{(0)}||f_{S^*}^{(1)}) = \sup\limits_S \textup{KL}(f_S^{(0)}||f_S^{(1)}), \textup{KL}(f_{S^*}^{(1)}||f_{S^*}^{(0)}) = \sup\limits_S \textup{KL}(f_S^{(1)}||f_S^{(0)})$.
\end{enumerate}
\end{proposition}

From Proposition \ref{prop: KL}, if we define the population RIC without penalty as 
\begin{equation}
	\textup{RIC}(S) = -2\left[\pi_0 \textup{KL}(f_S^{(0)}||f_S^{(1)}) + \pi_1 \textup{KL}(f_S^{(1)}||f_S^{(0)})\right],
\end{equation}
it can be easily seen that $\sup\limits_{S: S\supseteq S^*} \ric(S) = \ric(S^*) < \inf\limits_{S: S\not\supseteq S^*} \ric(S)$. To successfully differentiate $S^*$ from $\{S: S\not\supseteq S^*\}$ using $\ric_n$, we need to impose a condition on the minimum gap of $\ric$ on $S^*$ and that on $S$ where $S\not\supseteq S^*$. Similar assumptions such as the ``beta-min" condition appears in the high-dimensional variable selection literature \citep{buhlmann2011statistics}. Denote \[\psi(n, p, D) =  \sqrt{\frac{D\log p + \kappa_1 \log D}{n}} \max\left\{D^{\kappa_1}(D^{\kappa_3} + D^{\kappa_4}), D^{\kappa_5}, D^{2\kappa_1 + \kappa_2}\sqrt{\frac{D\log p + \kappa_1 \log D}{n}}\right\}.\] The complete set of conditions is presented as follows.
\begin{assumption}\label{asmp: consistency}
	Suppose densities $f^{(0)}$ and $f^{(1)}$ are in parametric forms with $f^{(0)}(\bm{x}) = f^{(0)}(\bm{x}; \btheta)$, $f^{(1)}(\bm{x}) = f^{(1)}(\bm{x}; \btheta)$, where $\btheta$ contains all parameters for both $f^{(0)}$ and $f^{(1)}$. Note that not all elements of $\btheta$ appear in $f^{(0)}$ or $f^{(1)}$.  Denote the dimension of $\btheta$ as $p'$. As in Section \ref{sec: main}, $D$ represents the upper bound of the subspace size. Define $L_S(\bm{x}, \btheta) = \log \left(\frac{f^{(0)}_S(\bm{x}_S; \bthetas)}{f^{(1)}_S(\bm{x}_S; \bthetas)}\right)$.
	Assume the following conditions hold, where $\kappa_1$, $\kappa_2$, $\kappa_3$, $\kappa_4$, $\kappa_5 \geq 0$ and $C_1$, $C_2$, $C_3$, $C_4$, $C_5 > 0$ are  some universal constants which does not depend on $n$, $p$ or $D$:
	\begin{enumerate}[label=(\roman*)]
		\item $p'$ is a function of $p$ and $p'(p) \lesssim p^{\kappa_1}$;
		\item $\max\left[\textup{KL}(f^{(0)}||f^{(1)}), \textup{KL}(f^{(1)}||f^{(0)}) \right]\lesssim (p^*)^{\kappa_1} \lesssim D^{\kappa_1}$;
		\item There exists a family of functions 
		$\{V_S(\{\bm{x}_{i, S}\}_{i=1}^n)\}_{S}$ where $V_S(\{\bm{x}_{i, S}\}_{i=1}^n)\in \mathbb{R}$, such that for $\forall \{\bm{x}_{i, S}\}_{i=1}^n \in \mathcal{X}$ and subset $S$ with $|S| \leq D$, there exists a constant $\zeta$ such that if $\twonorm{\bthetas' - \bthetas} \leq \zeta$, then \[
		\maxnorma{\frac{1}{n}\sum_{i=1}^n\nabla^2_{\bthetas} L_S(\bm{x}_i, \btheta')} \leq V_S(\{\bm{x}_{i, S}\}_{i=1}^n),\]
		and the following tail probability bound holds for $V_S$:
		\begin{equation}
			\tp\left(V_S(\{\bm{x}_{i, S}\}_{i=1}^n) > C_1D^{\kappa_2}\right) \lesssim \exp\{-C_2n\},
		\end{equation}
		where $\bx_{i, S} \sim f_S^{(0)}$ or $f_S^{(1)}$;
		\item Each component of $\nabla_{\bthetas} L_S(\bm{x}, \btheta) $ is $\sqrt{2C_3}D^{\kappa_3}$-subGaussian for any subset $S$ with $|S| \leq D$ where $\bm{x}_S \sim f_S^{(0)}$ or $\bm{x}_S \sim f_S^{(1)}$, respectively; 
		\item $\sup\limits_{S: |S| \leq D}\infnorma{\te_{\bm{x}_S \sim f^{(0)}_S} \nabla_{\bthetas} L_S(\bm{x}, \btheta) }, \sup\limits_{S: |S| \leq D}\infnorma{\te_{\bm{x}_S \sim f^{(1)}_S} \nabla_{\bthetas} L_S(\bm{x}, \btheta)} \lesssim D^{\kappa_4}$;
		\item $ L_S(\bm{x}, \btheta) $ is a $\sqrt{2C_4}D^{\kappa_5}$-subGaussian variable, where $\bm{x}_S \sim f_S^{(0)}$ or $\bm{x}_S \sim f_S^{(1)}$;
		\item Denoting the MLE of $\btheta$ based on subset $S$ with $|S| \leq D$ as\[
			\hthetas = \arg\max_{\bthetas}\sum_{i=1}^n \sum_{r=0}^1 \mathds{1}(y_i = r)\log f^{(r)}_S(\bm{x}_{i, S}; \bthetas),
		\]then when $\epsilon$ is smaller than a positive constant, it holds that
		\begin{equation}
			\tp(\infnorm{\hthetas - \bthetas} > \epsilon) \lesssim |S|^{\kappa_1}\exp\{-C_5n\epsilon^2\}  \Rightarrow \infnorm{\hthetas - \bthetas} = O_p\left(\sqrt{\frac{\kappa_1\log |S|}{n}}\right).
		\end{equation}
		\item The signal strength satisfies
		\begin{equation}
			\Delta \coloneqq \inf_{\substack{S: S \not\supseteq S^* \\ |S| \leq D}}\textup{RIC}(S) - \textup{RIC}(S^*) \gg \psi(n, p, D),
		\end{equation}
		where $\psi(n, p, D) = o(1)$.
	\end{enumerate}
\end{assumption}

Here condition (\rom{1}) assumes the number of parameters in the model grows slower than a polynomial rate of the dimension. Condition (\rom{2}) assumes an upper bound for the two KL divergences. Conditions (\rom{3})-(\rom{6}) are imposed to guarantee the accuracy of the second-order approximation for RIC. And condition (\rom{7}) is usually satisfied for common distribution families \citep{van2000asymptotic}. The last condition is a requirement for the signal strength. A condition of this type is necessary to prove the consistency result for any information criterion.

For condition (\rom{8}), it imposes a constraint among $n$, $p$, and $D$. When $D \ll p$, it can be simplified as \[
	D^\kappa \cdot \frac{\log p}{n} = o(1),
	\]where $\kappa$ is some positive constant. To cover all the signals in $S^*$, we require $D \geq p^*$. Therefore, an implied requirement for $n$, $p$, and $p^*$ is \[
	(p^*)^\kappa \cdot \frac{\log p}{n} = o(1),
	\]which is similar to the conditions in the literature of variable selection. Here, we allow $p$ to grow at an exponential rate of $n$.

To help readers understand these conditions better, we show that some commonly used conditions (presented in Assumption \ref{asmp: lda}) for high-dimensional LDA model are sufficient for Assumption \ref{asmp: consistency}.(\rom{1})-(\rom{7}) to hold with the results presented in Proposition \ref{prop: lda check}. Assumption \ref{asmp: consistency}.(\rom{8}) can be relaxed under the LDA model.
\begin{assumption}[LDA model] \label{asmp: lda}
	Suppose the following conditions are satisfied, where $m$, $M$, $M'$ are constants:
	\begin{enumerate}[label=(\roman*)]
		\item $\lambda_{\min}(\Sigma) \geq m > 0, \maxnorm{\Sigma} \leq M < \infty$;
		\item $\infnorm{\bmur{1} - \bmur{0}} \leq M' < \infty$;
		\item Denote $\bdeltas = \invsig_{S, S}(\bmur{1}_{S} - \bmur{0}_{S}), \gamma = \inf\limits_{j}\norm{(\bdelta_{S^*})_j} > 0$, then
		\begin{equation}
			\gamma^2 \gg D^2\sqrt{\frac{\log p}{n}} = o(1).
		\end{equation}
	\end{enumerate}
\end{assumption}

\begin{remark}
	Here, condition (\rom{1}) constrains the eigenvalues of the common covariance matrix, which is similar to condition (C2) in \cite{hao2018model} and \cite{wang2009forward}, condition (2) in \cite{shao2011sparse}, condition (C4) in \cite{li2019robust}. Condition (\rom{2}) imposes an upper bound for the maximal componentwise mean difference of the two classes, which is similar to condition (3) in \cite{shao2011sparse}. Condition (\rom{3}) assumes a lower bound on the minimum signal strength $\gamma$, which is in a similar spirit to condition 2 in \cite{mai2012direct}.
\end{remark}
\begin{proposition}\label{prop: lda check}
	Suppose $f^{(0)} \sim N(\bm{\mu}^{(0)}, \Sigma), f^{(1)} \sim N(\bm{\mu}^{(1)}, \Sigma)$, where $\Sigma$ is positive definite. If Assumption \ref{asmp: lda} holds, then Assumption \ref{asmp: consistency}.(\rom{1})-(\rom{7}) hold with $\bthetas = ( (\bmusr{0})^T, (\bmusr{1})^T, \textup{vec}(\Sigma_{S, S})^T)^T$ and $\kappa_1 = 2, \kappa_2 = \kappa_4 =1, \kappa_3 = \kappa_5 = \frac{1}{2}$.  
\end{proposition}
The detailed proof for this proposition can be found in Appendix \ref{appendix a: proof}. We are now ready to present the consistency result for RIC as defined in \eqref{eq: def ric}.
\begin{theorem}[Consistency of RIC]\label{thm: consistency}
	Under Assumptions 1-3, we have
	\begin{enumerate}[label=(\roman*)]
		\item If $\supp\limits\deg(S)\cdot c_n/\Delta = o(1)$, then the following screening consistency holds for RIC:\footnote{Note that here we assume that MLE of $\btheta$ is well-defined for all models under consideration.}
			\begin{align}
				&\tp\left(\sup_{\substack{S: S \supseteq S^* \\ |S| \leq D}}\textup{RIC}_n(S) < \inf_{\substack{S: S \not\supseteq S^* \\ |S| \leq D}}\textup{RIC}_n(S)\right)  \geq 1- O\left(p^{D}D^{\kappa_1}\exp\left\{-Cn\left(\frac{\Delta}{D^{2\kappa_1 + \kappa_2}}\right)\right\}\right)\\
				&- O\left(p^{D}D^{\kappa_1}\exp\left\{-Cn\left(\frac{\Delta}{D^{\kappa_1}(D^{\kappa_3} + D^{\kappa_4})}\right)^2\right\}\right) - O\left(p^{D}\exp\left\{-Cn\left(\frac{\Delta}{D^{\kappa_5}}\right)^2\right\}\right) \\
				&\rightarrow 1.
			\end{align}
		\item If in addition $c_n \gg \psi(n, p, D)$, then the weak consistency in the following sense holds for RIC:		
			\begin{align}
				&\tp\left(\textup{RIC}_n(S^*) = \inf_{S: |S| \leq D}\textup{RIC}_n(S)\right)  \geq 1 - O\left(p^{D}D^{\kappa_1}\exp\left\{-Cn\left(\frac{c_n}{D^{2\kappa_1 + \kappa_2}}\right)\right\}\right)\\
				&-  O\left(p^{D}D^{\kappa_1}\exp\left\{-Cn\left(\frac{c_n}{D^{\kappa_1}(D^{\kappa_3} + D^{\kappa_4})}\right)^2\right\}\right)- O\left(p^{D}\exp\left\{-Cn\left(\frac{c_n}{D^{\kappa_5}}\right)^2\right\}\right) \\
				&\rightarrow 1.
			\end{align}
	\end{enumerate}
\end{theorem}

\begin{corollary}\label{cor: coverage}
	Denote $p_{S^*} = \bp(S_{11} \supseteq S^*) = \frac{1}{D}\sum\limits_{p^* \leq d \leq D}\frac{\binom{p-p^*}{d-p^*}}{\binom{p}{d}}$, where the subspaces are generated from the hierarchical uniform distribution. Assume the conditions stated in Assumption \ref{asmp: density}-\ref{asmp: uniqueness} hold, and in addition there holds\[
	B_2p_{S^*} \gg 1.
	\]
	If we select the optimal subspace by minimizing RIC as the criterion in RaSE where $D \geq p^*$, then we have
	\begin{equation}\label{eq: cover ineq 2}
	\p(S_{1*} \supseteq S^*) \geq \tp\left(\sup_{\substack{S: S \supseteq S^* \\ |S| \leq D}}\textup{RIC}_n(S) < \inf_{\substack{S: S \not\supseteq S^* \\ |S| \leq D}}\textup{RIC}_n(S)\right)\cdot \bp\left( \bigcup_{j=1}^{B_2} \{S_{1j} \supseteq S^*\}\right),
	\end{equation}
	where \[
		\bp\left( \bigcup_{j=1}^{B_2} \{S_{1j} \supseteq S^*\}\right) = 1-(1 - p_{S^*})^{B_2} \geq 1-O\left(\exp\left\{-B_2p_{S^*}\right\}\right).
	\]
\end{corollary}

\begin{remark}
	Note that this corollary actually holds  when we  replace  $\ric_n$ with a general criterion $\cri_n$ (smaller value leads to better subspace) with more discussions in Section \ref{subsec: iterative rase}. And for RIC, the bounds for the first probability on the right-hand side of \eqref{eq: cover ineq 2} in Theorems \ref{thm: consistency}, \ref{thm: lda consistency} and \ref{thm: qda consistency} can be plugged in to get the explicit bounds. 
\end{remark}

Also, we want to point out that direct analyses of RIC for discriminant analysis models are also insightful and interesting. We can show similar consistency results as those in Theorem \ref{thm: consistency} from properties of discriminant analysis approach itself based on some common conditions used in literature about sparse discriminant analysis, instead of applying the general analysis of KL divergence.

\begin{theorem}[LDA consistency]\label{thm: lda consistency}
	For the LDA model, under Assumption \ref{asmp: lda}, we have
	\begin{enumerate}[label=(\roman*)]
		\item If $Dc_n/\gamma^2 = o(1)$, then the following screening consistency holds for RIC:
			\begin{equation}
				\tp\left(\sup_{\substack{S: S \supseteq S^* \\ |S| \leq D}}\textup{RIC}_n(S) < \inf_{\substack{S: S \not\supseteq S^* \\ |S| \leq D}}\textup{RIC}_n(S)\right) \geq 1 - O\left(p^2\exp\left\{-Cn\left(\frac{\gamma^2}{D^2}\right)^2\right\}\right) \rightarrow 1.
			\end{equation}		\item If in addition $c_n \gg D^2\sqrt{\frac{\log p}{n}}$, then RIC is weakly consistent: 	
		\begin{equation}
				\tp\left(\textup{RIC}_n(S^*) = \inf_{S: |S| \leq D}\textup{RIC}_n(S)\right)  \geq 1 - O\left(p^2\exp\left\{-Cn\left(\frac{c_n}{D^2}\right)^2\right\}\right) \rightarrow 1.
			\end{equation}	
	\end{enumerate}
\end{theorem}

\begin{assumption}[QDA model]\label{asmp: qda}
	Denote $\omegasr{r}{S, S} = (\Sigma_{S, S}^{(r)})^{-1}, r=0, 1$. Suppose the following conditions are satisfied, where $m, M, M'$ are constants:
	\begin{enumerate}[label=(\roman*)]
		\item $\lambda_{\min}(\sigr{r}) \geq m > 0, \lambda_{\max}(\sigr{r}) \leq M < \infty, r = 0, 1$;
		\item $\infnorm{\bmur{1} - \bmur{0}} \leq M' < \infty$;
		\item Denote $\gammal = \inf\limits_{j}\norma{(\omegasr{1}{S, S}\bmur{1}_{S} - \omegasr{0}{S, S}\bmur{0}_{S})_j} > 0, \gammaq = \inf\limits_{i}\sup\limits_j\norma{(\omegasr{1}{S^*_q, S^*_q} - \omegasr{0}{S^*_q, S^*_q})_{ij}} > 0$, then
		\begin{equation}
			\min\{\gammal^2, \gammaq^2, \gammaq\} \gg D^2\sqrt{\frac{\log p}{n}} = o(1).
		\end{equation}
	\end{enumerate}
\end{assumption}

\begin{remark}
	The conditions here are similar to Assumption \ref{asmp: lda} for the LDA model. A set of analogous conditions were used in \cite{jiang2018direct}.
\end{remark}

\begin{theorem}[QDA consistency]\label{thm: qda consistency}
	For the QDA model, under Assumption \ref{asmp: qda},
	\begin{enumerate}[label=(\roman*)]
		\item If $D^2c_n/\min\{\gammal^2, \gammaq^2, \gammaq\} = o(1)$, then RIC is screening consistent:
			\begin{align}
				\tp\left(\sup_{\substack{S: S \supseteq S^* \\ |S| \leq D}}\textup{RIC}_n(S) < \inf_{\substack{S: S \not\supseteq S^* \\ |S| \leq D}}\textup{RIC}_n(S)\right) &\geq 1 - O\left(p^2\exp\left\{-Cn\left(\frac{\min\{\gammal^2, \gammaq^2, \gammaq\}}{D^2}\right)^2\right\}\right) \\
				&\rightarrow 1.
			\end{align}
		\item Further, if $c_n \gg D^2\sqrt{\frac{\log p}{n}}$, then RIC is weakly consistent:		
			\begin{equation}
				\tp\left(\textup{RIC}_n(S^*) = \inf_{S: |S| \leq D}\textup{RIC}_n(S)\right)  \geq 1 - O\left(p^2\exp\left\{-Cn\left(\frac{c_n}{D^2}\right)^2\right\}\right) \rightarrow 1.
			\end{equation}
	\end{enumerate}
	\
\end{theorem}
The proof  is available in Appendix \ref{appendix a: proof}. Note that the bound here is  tighter than the results  from Proposition \ref{prop: lda check} and Theorem \ref{thm: consistency}.

Based on the consistency of RIC, in the next section, we will construct an upper bound for the expectation of  the misclassification rate $R(C_n^{RaSE})$.

\subsection{Misclassification Rate of the RaSE Classifier}\label{subsec: error rase}
In the following theorem, we present an upper bound on the misclassification rate for the RaSE classifier, which holds for any criterion to choose optimal subspaces. 

\begin{theorem}[General misclassification rate]\label{thm: error bound using RIC}
	For the RaSE classifier with threshold $\alpha$ and any criterion to choose optimal subspaces, it holds that
	\begin{equation}\label{eq: error rate}
		\te\{\be[R(C_n^{RaSE}) - R(C_{Bayes})]\} \leq \frac{\te\supps\limits[R(C_n^{S}) - R(C_{Bayes})] + \p(S_{1*} \not\supseteq S^*)}{\min(\alpha, 1 - \alpha)}.
	\end{equation}
\end{theorem}

Here, the upper bound consists of two terms. The first term involving  $\mathbb{E}\supps\limits[R(C_n^{S^*}) - R(C_{Bayes})]$ can be seen as the maximum discrepancy between the risk of models trained in any subspace covering $S^*$ based on finite samples with the Bayes risk, which will be investigated in detail in the next subsection. This term shrinks to zero under certain conditions (see details in Section \ref{subsec::theory:detailed}). The second term corresponds to the event that at least one signal is missed in the selected subspace.  Corollary \ref{cor: coverage} in Section \ref{subsec: weak cosis} shows that $B_2$ needs to be sufficiently large to ensure this term goes to 0. We will show in Section \ref{subsec: iterative rase} that the iterative RaSE could relax the requirement on $B_2$ under certain scenarios. 

 Specifically, if we use the criterion of minimizing training error (misclassification rate on the training set) or leave-one-out cross-validation error, a similar guarantee of performance can be arrived as follows.

\begin{theorem}[Misclassification rate when minimizing training error or leave-one-out cross-validation error]\label{thm: error bound using training error} If the criterion of minimal training error or leave-one-out cross-validation error is applied for the RaSE classifier with threshold $\alpha$, it holds that
	\begin{align}
	&\te\{\be[R(C_n^{RaSE}) - R(C_{Bayes})]\}\\
	&\leq \frac{\te\supps\limits[R(C_n^{S}) - R(C_{Bayes})] + \left[\mathbb{E}(\epsilon_n) + \mathbb{E}\supps\limits|\epsilon_n^S|\right] + (1 - p_{S^*})^{B_2}}{\min(\alpha, 1 - \alpha)},\label{eq::mis_rate_tr_error}
	\end{align}
	where $\epsilon_n^S = R(C_n^{S}) - R_n(C_n^{S}), \epsilon_n = \be[R(C_n^{S_{1*}}) - R_n(C_n^{S_{1*}})]$. Here $R_n(C)$ is the training error or leave-one-out cross-validation error of classifier $C$.
\end{theorem}

This theorem is closely related to Theorem 3 in \cite{cannings2017random} and derived along similar lines. The merit of Theorem \ref{thm: error bound using training error} compared with Theorem \ref{thm: error bound using RIC} is that we don't have the term $\p(S_{1*} \not\supseteq S^*)$ in the bound, which can be difficult to quantify when minimizing training error or leave-one-out cross-validation error. Regarding the upper bound in \eqref{eq::mis_rate_tr_error}, the first term is the same as the first term of the bound in Theorem \ref{thm: error bound using RIC}. The second term involving $ \te(\epsilon_n) + \mathbb{E}\supps\limits|\epsilon_n^S| $ is relative to the distance between the training error and test error, which usually shrinks to zero for some specific classifiers under certain conditions (see details in Section \ref{subsec::theory:detailed}).  The third term involving $(1 - p_{S^*})^{B_2}$ reflects the possibility that $S^*$ is not selected in any of the $B_2$ subspaces we generate, which is similar to the second term in the bound given by Theorem \ref{thm: error bound using RIC}. This term shrinks to zero under the condition of Corollary \ref{cor: coverage}.

\subsection{Detailed Analysis for Several Base Classifiers}\label{subsec::theory:detailed}
In this section, we work out the technical details for the RaSE classifier when the base classifier is chosen to be LDA, QDA, and $k$NN.
\subsubsection{Linear Discriminant Analysis (LDA)}
LDA was proposed by \cite{fisher1936use} and corresponds to model \eqref{eq: model} where $f^{(r)} \sim N(\bm{\mu}^{(r)}, \Sigma), r = 0,1$. For given training data $\{\bm{x}_i, y_i\}_{i = 1}^n$ in subspace $S$, using  the MLEs given in Proposition \ref{prop: RIC for LDA case}, the classifier can be constructed as
\begin{equation}
	C_n^{S-LDA}(\bm{x}) = \begin{cases}
		1, & L_S(\bm{x}_S| \hat{\pi}_0, \hat{\pi}_1, \hat{\bm{\mu}}_S^{(0)}, \hat{\bm{\mu}}_S^{(1)}, \hat{\Sigma}_{S, S}) > 0,\\
		0, & \textrm{otherwise},
	\end{cases}
\end{equation}
where the decision function
\begin{equation}\label{eq: deci lda}
	L_S(\bm{x}_S| \hat{\pi}_0, \hat{\pi}_1, \hat{\bm{\mu}}_S^{(0)}, \hat{\bm{\mu}}_S^{(1)}, \hat{\Sigma}_{S, S}) = \log(\hat{\pi}_1/\hat{\pi}_0) +  (\bm{x}_S - (\hat{\bm{\mu}}_S^{(0)} + \hat{\bm{\mu}}_S^{(1)})/2)^T(\hat{\Sigma}_{S, S})^{-1} (\hat{\bm{\mu}}_S^{(1)} - \hat{\bm{\mu}}_S^{(0)}).
\end{equation}
And the degree of freedom of the LDA model with feature subspace $S$ is $\deg(S) = |S|+1$.
\cite{efron1975efficiency} derived that 
\begin{align}
	R(C_n^{S-LDA}) - R(C_{Bayes})  &= \pi_1\left[\Phi\left(-\frac{\hat{\Delta}_S}{2} + \hat{\tau}_S\right)  - \Phi\left(-\frac{\Delta_S}{2} + \tau_S\right)\right]\\
	&\quad +\pi_0\left[\Phi\left(-\frac{\hat{\Delta}_S}{2} - \hat{\tau}_S\right)- \Phi\left(-\frac{\Delta_S}{2} - \tau_S\right)\right],\label{eq: error diff lda}
\end{align}
where $\Delta_S = \sqrt{\left(\bmusr{1} - \bmusr{0}\right)^T \Sigma_{S, S}^{-1}\left(\bmusr{1} - \bmusr{0}\right)}$, $\hat{\Delta}_S = \sqrt{\left(\hmusr{1} - \hmusr{0}\right)^T \hat{\Sigma}_{S, S}^{-1}\left(\hmusr{1} - \hmusr{0}\right)}$, $\tau_S = \log(\pi_1/\pi_0)/\Delta_S, \hat{\tau}_S = \log(\hpi_1/\hpi_0)/\hat{\Delta}_S$. 

\begin{proposition}\label{prop: lda error rate}
	If Assumption \ref{asmp: lda} holds, then we have
	\begin{equation}
		 \mathbb{E}\supps[R(C_n^{S-LDA}) - R(C_{Bayes})] \lesssim D^2\sqrt{\frac{\log p}{n}}\cdot \max\{(p^*)^{-\frac{1}{2}}\gamma^{-1}, (p^*)^{-\frac{3}{2}}\gamma^{-3}\}.	\end{equation}
\end{proposition}


Regarding the second term in the upper bound in Theorem \ref{thm: error bound using training error},  due to Theorem 23.1 in \cite{devroye2013probabilistic}, for any subset $S$, we have
\begin{equation}
	\tp\left(|\epsilon_n^{S}| > \epsilon\right)\leq8n^D\exp\left\{-\frac{1}{32}n\epsilon^2\right\},
\end{equation}
which yields
\begin{equation}
	\tp\left(\supps|\epsilon_n^{S}| > \epsilon\right)\leq \sum_{\substack{S:S \supseteq S^*\\|S|\leq D}}\tp\left(|\epsilon_n^{S}| > \epsilon\right) \leq 8n^D p^{D-p^*}\exp\left\{-\frac{1}{32}n\epsilon^2\right\}.
\end{equation}
By Lemma \ref{lem: devroye} in Appendix \ref{appendix a: proof}, it follows
\begin{equation}\label{eq: lda1}
	\te \supps|\epsilon_n^{S}| \leq \sqrt{\frac{32[(D-p^*)\log p + D\log n + 3\log 2 + 1]}{n}}.
\end{equation}
Also since \[
	\mathbb{E}|\epsilon_n| \leq \be\left[\te \left(\sup_{k = 1,...,B_2} |\epsilon_n^{S_{1k}}|\right)\right],
\]and\[
	\tp\left(\sup_{k = 1,...,B_2}|\epsilon_n^{S}| > \epsilon\right)\leq \sum_{k=1}^{B_2}\tp\left(|\epsilon_n^{S_{1k}}| > \epsilon\right) \leq 8B_2n^D\exp\left\{-\frac{1}{32}n\epsilon^2\right\},
\]again by Lemma \ref{lem: devroye}, we have
\begin{equation}\label{eq: lda2}
	\te |\epsilon_n| \leq \sqrt{\frac{32[\log B_2 + D\log n + 3\log 2 + 1]}{n}}.
\end{equation}

By plugging these bounds in the right-hand side of \eqref{eq: error rate} and \eqref{eq::mis_rate_tr_error}, we can get the explicit upper bound of the misclassification rate for the LDA model.
\subsubsection{Quadratic Discriminant Analysis (QDA)}
QDA considers the model \eqref{eq: model} analogous to LDA while $\bm{x}|y = r \sim N(\bm{\mu}^{(r)}, \Sigma^{(r)}), r = 0,1$, where $\sigr{0}$ can be different from $\sigr{1}$. On the basis of training data $\{(\bm{x}_i, y_i)\}_{i = 1}^n$ in subspace $S$, it
admits the following form of classifier based on the MLEs given in Proposition \ref{prop: RIC for QDA case}:

\begin{equation}
	C_n^{S-QDA}(\bm{x}) = \begin{cases}
		1, & Q_S(\bm{x}_S| \hat{\pi}_0, \hat{\pi}_1, \hat{\bm{\mu}}_S^{(0)}, \hat{\bm{\mu}}_S^{(1)}, \hat{\Sigma}_{S, S}^{(0)}, \hat{\Sigma}_{S, S}^{(1)}) > 0,\\
		0, & \textrm{otherwise},
	\end{cases}
\end{equation}
where the decision function
\begin{align}
	Q_S(\bm{x}_S| \hat{\pi}_0, \hat{\pi}_1, \hat{\bm{\mu}}_S^{(0)}, \hat{\bm{\mu}}_S^{(1)}, \hat{\Sigma}_{S, S}^{(0)}, \hat{\Sigma}_{S, S}^{(1)}) &= \log(\hat{\pi}_1/\hat{\pi}_0) -\frac{1}{2}\bm{x}_S^T[(\hat{\Sigma}_{S, S}^{(1)})^{-1} - (\hat{\Sigma}_{S, S}^{(0)})^{-1}]\bm{x}_S \\
	&\quad+ \bm{x}_S^T[(\hat{\Sigma}_{S, S}^{(1)})^{-1}\hat{\bm{\mu}}_S^{(1)} - (\hat{\Sigma}_{S, S}^{(0)})^{-1}\hat{\bm{\mu}}_S^{(0)}] \nonumber\\
	&\quad - \frac{1}{2}(\hat{\bm{\mu}}_S^{(1)})^T(\hat{\Sigma}_{S, S}^{(1)})^{-1}\hat{\bm{\mu}}_S^{(1)} + \frac{1}{2}(\hat{\bm{\mu}}_S^{(0)})^T(\hat{\Sigma}_{S, S}^{(0)})^{-1}\hat{\bm{\mu}}_S^{(0)}.
\end{align}
And the degree of freedom of QDA model with feature subspace $S$ is $\deg(S) = |S|(|S|+1)/2 + |S| + 1$.  

To analyze the first term in \eqref{eq: error rate} and \eqref{eq::mis_rate_tr_error}, as in \cite{jiang2018direct}, for any constant $c$, we define \[
	u_c = \max\left\{\mathop{\textup{ess\,sup}}\limits_{z \in [-c,c]} h^{(r)}(z), r = 0, 1\right\},
\]where $\textup{ess\,sup}$ represents the essential supremum defined as the supremum except a set with measure zero and $h^{(r)}(z)$ is the density of $Q_{S^*}(\bx_{S^*}|\pi_1, \pi_0, \bmur{0}_{S^*}, \bmur{1}_{S^*}, \sigr{0}_{S^*, S^*}, \sigr{1}_{S^*, S^*})$ given that $y = r$.
\begin{proposition}\label{prop: qda error rate}
	If Assumption \ref{asmp: qda} and the following conditions hold:
	\begin{enumerate}[label=(\roman*)]
		\item There exist positive constants $c, U_c$ such that $u_c \leq U_c < \infty$;
		\item There exists a positive number $\varpi_0 \in (0, 1)$, $p \lesssim \exp\{n^{\varpi_0}\}$;
	\end{enumerate}
	then we have
	\begin{equation}
		\mathbb{E}\supps\left[R(C_n^{S-QDA}) - R(C_{Bayes})\right] \lesssim D^2\left(\frac{\log p}{n}\right)^{\frac{1-2\varpi}{2}}.
	\end{equation}
	for any $\varpi \in (0, 1/2)$.
\end{proposition} 

Also, by applying Theorem 23.1 in \cite{devroye2013probabilistic} and Lemma \ref{lem: devroye}, we have similar conclusions as \eqref{eq: lda1} and \eqref{eq: lda2} in the following
\begin{align}
	\mathbb{E}\supps|\epsilon_n^{S}| &\leq \sqrt{\frac{32[(D-p^*)\log p + D(D+3)/2\cdot \log n + 3\log2 + 1]}{n}}, \\
	\te|\epsilon_n| &\leq \sqrt{\frac{32[\log B_2 + D(D+3)/2\cdot \log n + 3\log2 + 1]}{n}}.
\end{align}
The explicit upper bound of misclassification rate for the QDA model follows when we plug these inequalities into the right-hand sides of \eqref{eq: error rate} and \eqref{eq::mis_rate_tr_error}.

\subsubsection{$k$-nearest Neighbor ($k$NN)}
$k$NN method was firstly proposed by \cite{fix1951discriminatory}. Given $\bm{x}$, it tries to mimic the regression function $E[y|\bm{x}]$ in the local region around $\bm{x}$ by using the average of $k$ nearest neighbors. $k$NN is a non-parametric method and its success has been witnessed in a wide range of applications. 

Given training data $\{\bm{x}_i, y_i\}_{i = 1}^n$, for the new observation $\bm{x}$ and subspace $S$, rank the training data by the increasing $\ell^2$ distance in Euclidean  space to $\bm{x}_S$ as $\{\bm{x}_{m_i,S}\}_{i = 1}^n$ such that\[
	\twonorm{\bm{x}_S - \bm{x}_{m_1,S}} \leq \twonorm{\bm{x}_S - \bm{x}_{m_2,S}} \leq \ldots \leq\twonorm{\bm{x}_S - \bm{x}_{m_n,S}},
\]where $\twonorm{\cdot}$ represents the $\ell^2$ norm in the corresponding Euclidean space. Then the $k$NN classifier admits the following form:
\begin{equation}
	C_n^{S-k\textup{NN}}(\bm{x}) = \begin{cases}
		1, & \frac{1}{k}\sum_{i=1}^k y_{m_i} > 0.5,\\
		0, & \textrm{otherwise}.
	\end{cases}
\end{equation}

By \cite{devroye1979distribution} and \cite{cannings2017random}, it holds the following tail bound:
\begin{equation}
	\tp\left(\supps |\epsilon_n^S| > \epsilon\right) \leq \sum_{\substack{S:S \supseteq S^*\\|S|\leq D}}\tp\left(|\epsilon_n^S| > \epsilon\right)\leq 8p^{D-p^*}\exp\left\{-\frac{n\epsilon^3}{108k(3^D + 1)}\right\}.
\end{equation}
Then by Lemma \ref{lem: devroye} and similar to the analysis for deriving \eqref{eq: lda2}, it follows
\begin{align}
	\te \supps |\epsilon_n^S| &\leq [108k(3^D + 1)]^{\frac{1}{3}}\cdot \left(\frac{3 \log 2 + (D-p^*)\log p + 1}{n}\right)^{\frac{1}{3}},\\
	\te |\epsilon_n| &\leq [108k(3^D + 1)]^{\frac{1}{3}}\cdot \left(\frac{3 \log 2 + \log B_2 + 1}{n}\right)^{\frac{1}{3}}.
\end{align}
However, for $k$NN, due to its lack of parametric form, it is much more involved to derive a similar upper bound as those presented in Propositions \ref{prop: lda error rate} and \ref{prop: qda error rate}. We decide to leave this analysis as future work.


\subsection{Theoretical Analysis of Iterative RaSE}\label{subsec: iterative rase}
Recall that to control the misclassification rate when minimizing RIC, we showed in Section \ref{subsec: error rase} that to control $\p(S_{1*}\not\supseteq S^*)$, $B_2$ needs to be sufficiently large. In particular, a sufficient condition regarding $B_2$ was presented in Corollary \ref{cor: coverage}. Since  
	\begin{equation}
		\frac{\binom{p-p^*}{d-p^*}}{\binom{p}{d}} \leq \frac{\binom{p-p^*}{D-p^*}}{\binom{p}{D}}\leq \left(\frac{D}{p-p^*+1}\right)^{p^*},
	\end{equation}
	condition (\rom{2}) in Corollary \ref{cor: coverage} implies $B_2 \gg \left(\frac{p-p^*+1}{D}\right)^{p^*}$, which could be very large for high-dimensional settings. Next, we show the iterative RaSE in Algorithm \ref{algo_iteration} can sometimes relax the requirement on $B_2$ substantially. 
	
Different from the hierarchical uniform distribution over the subspaces in RaSE (Algorithm \ref{algo}), the iterative RaSE in Algorithm \ref{algo_iteration} uses a non-uniform distribution from the second iteration.  The non-uniform distribution works by assigning higher probabilities to subspaces that include the more frequently appeared variables among the $B_1$ subspaces chosen in the previous step. We will show that the subspaces generated from such non-uniform distributions require a smaller $B_2$ to cover $S^*$. 



In the following analysis,  we study the iterative RaSE algorithm that minimizes a general criterion $\cri$, which is a real-valued function on any subspace with its sample version denoted as $\cri_n$. 

To guarantee the success of Algorithm \ref{algo_iteration}, we need the following conditions.
\begin{assumption}\label{asmp: iterative rase}
Suppose the following conditions are satisfied:
	\begin{enumerate}[label=(\roman*)]
		\item There exists a positive function of $n, p, D$ called $\nu$ satisfying $\nu(n, p, D) = o(1)$, such that
			\begin{equation}
				\supp \norma{\cri_n(S) - \cri(S)} = O_p(\nu(n, p, D)).
			\end{equation}
		\item (Stepwise detectable condition) There exists a series $\{M_n\}_{n=1}^{\infty} \rightarrow \infty$ and a specific positive integer $\bar{p}^* \leq p^*$ such that
		\begin{enumerate}
			\item for any feature subset $\tilde{S}^{(1)}_* \subseteq S^*$ where $|\tilde{S}^{(1)}_*| \leq p^* - \barp$, there exists $S^{(2)}_* \subseteq S^*\backslash \tilde{S}^{(1)}_*$ and $|S^{(2)}_*| = \bar{p}^*$ such that
			\begin{equation}
				\cri(S') - \cri(S) > M_n\nu(n, p, D)
			\end{equation}
			holds for any $n$, any $S$ and $S'$ satisfying $S \cap S^* \neq S' \cap S^*$, $S \cap S^* = \tilde{S}^{(1)}_* \cup S^{(2)}_*$, $S' \cap S^* = S_1' \cup S_2'$ where $S_1' \subseteq \tilde{S}^{(1)}_*,S_2' \subseteq S^* \backslash \tilde{S}^{(1)}_*, |S_2'| \leq \bar{p}^*$;
			\item the criterion satisfies
			\begin{equation}
				\inf\limits_{\substack{S: |S| \leq D \\ S\not\supseteq S^*}} \cri(S) - \sup\limits_{\substack{S: |S| \leq D \\ S\supseteq S^*}} \cri(S) > M_n\nu(n, p, D).
			\end{equation}
		\end{enumerate}
		
		\item  We have \[
			D\log\log p \ll \log p.
		\]
	\end{enumerate}
\end{assumption}


In Assumption \ref{asmp: iterative rase}, condition (\rom{1}) provides a uniform bound for the specific criterion we use. Condition (\rom{2})(a) is introduced to make Algorithm \ref{algo_iteration} detect additional signals in each iteration until all signals are covered. Condition (\rom{2})(b) is imposed to help us find discriminative sets among $B_2$ subspaces, and it holds for RIC by previous analysis in Section \ref{subsec: weak cosis}. Condition (\rom{3}) characterizes the requirement on the dimension $p$ and the maximal subspace size $D$. 

\begin{theorem}\label{thm: iterative rase}
	For Algorithm \ref{algo_iteration}, $B_2$ in the first step is set as
	\begin{equation}
		Dp^{\barp}\lesssim B_2 \ll \left(\frac{p}{\barp D}\right)^{\barp + 1},
	\end{equation}
	and $B_2$ in the following steps is set as
	\begin{equation}
		\left(1+\frac{D}{C_0}\right)^D p^{\barp} (\log p)^{p^*}\lesssim B_2 \ll \left(\frac{p}{\barp D}\right)^{\barp + 1},
	\end{equation}
	where $C_0$ is the same as in Algorithm \ref{algo_iteration}. Also we set $B_1$ such that $B_1 \gg \log p^*$. If Assumption \ref{asmp: iterative rase} holds, then after $T$ iterations where $\frac{e^{B_1}}{p^*} \gg T \geq \lceil\frac{p^*}{\barp}\rceil$, as $n, B_1, B_2 \rightarrow \infty$ there holds
	\begin{equation}
		\p(S^{(T)}_{1*} \not\supseteq S^*) \rightarrow 0.
	\end{equation}
\end{theorem}

Next, we compare the requirements on $B_2$ for iterative RaSE with that for the vanilla RaSE. For simplicity, we assume $D$, $p^*$, and $\barp$ are constants. When $\barp<p^*$, iterative RaSE requires $B_2\gtrsim p^{\barp} (\log p)^{p^*}$, which is much weaker than the requirement $B_2 \gg p^{p^*}$ for vanilla RaSE implied by Corollary \ref{cor: coverage}.


Using the results in Theorem \ref{thm: error bound using RIC}, an upper bound for the error rate could be obtained. Also note that the rate of $p$ is constrained by Assumption \ref{asmp: iterative rase}.(\rom{1}), where we assume $\nu(n, p, D) = o(1)$. For example, with LDA and QDA model, by Lemmas \ref{lem: lda consis 5} and \ref{lem: qda consis 6} in Appendix \ref{appendix a: proof}, under Assumptions \ref{asmp: lda} and \ref{asmp: qda} respectively, we have
\begin{equation}
	\nu(n, p, D) = D^2\sqrt{\frac{\log p}{n}}.
\end{equation}
Therefore the constraint is $D^2\sqrt{\frac{\log p}{n}} = o(1)$.



\section{Computational and Practical Issues\label{sec::computation}}

\subsection{Tuning Parameter Selection}\label{subsec: tune parameter}
In Algorithm \ref{algo}, there are five tuning parameters, including the number of weak learners $B_1$, the number of candidate subspaces $B_2$ to explore for each weak learner, the distribution $\mathcal{D}$ of subspaces, the criterion $\mathcal{C}$ for selecting the optimal subspace for each weak learner, and the threshold $\hat{\alpha}$.

 If we set $\mathcal{C}$ as minimizing the RIC, the difference between the risk of RaSE and Bayes risk, as well as the MC variance of RaSE, vanish at an exponential rate when $B_1\to \infty$, except for a finite set of thresholds $\alpha$. This implies the RaSE classifier becomes more accurate and stable as $B_1$ increases. 
 Regarding the impact of $B_2$, by Corollary \ref{cor: coverage}, Theorem \ref{thm: consistency} and Theorem \ref{thm: iterative rase}, under the minimal RIC criterion with some conditions, as $B_2, n \rightarrow \infty$, the subspace chosen for each weaker learner recovers the minimal determinative set with high probability. By Theorem \ref{thm: error bound using RIC}, the expectation of the misclassification rate becomes closer to the Bayes error as the sample size $n$ and $B_2$ increase, which motivates us to use a large $B_2$ if we have sufficient computational power. However, when we choose ``minimizing training error" as the criterion $\mathcal{C}$ to select the optimal subspace,  Theorem \ref{thm: error bound using training error} shows that the influence of $B_2$ becomes more subtle. In our implementation, we set $B_1 = 200$ and $B_2 = 500$ as default. For LDA and QDA classifier, $\mathcal{C}$ is set to choose the optimal subspace by minimizing the RIC, while for $k$NN, the default setting is minimizing the leave-one-out cross-validation error.

Without prior information about the features, as we mentioned in Section \ref{sec: main}, $\mathcal{D}$ is set as the hierarchical uniform distribution over the subspaces. To generate the size $d$ of subspaces from the uniform distribution over $\{1, \ldots, D\}$, another parameter $D$ has to be determined. In practice, for QDA base classifier we set $D = \min(p, [\sqrt{n_0}], [\sqrt{n_1}])$ and for LDA, $k$NN and all the other base classifiers, we set $D = \min(p, [\sqrt{n}])$, where $[a]$ denotes the largest integer not larger than $a$. The threshold $\hat{\alpha}$ is chosen by \eqref{eq: thresholding} to minimize the training error. When using non-parametric estimate of RIC corresponding to \eqref{eq: nonpara kl}, following \cite{wang2009divergence} and \cite{ganguly2018nearest}, we set $k_0 = [\sqrt{n_0}]$ and $k_1 = [\sqrt{n_1}]$ to satisfy the conditions they presented for proving the consistency.


\subsection{Computational Cost}
RaSE  is an ensemble framework, generating $B_1B_2$ subspaces in total following distribution $\mathcal{D}$. If we use the uniform distribution introduced in the last section to generate one subspace, the time required equals to the time for sampling at most $D$ features from $p$ ones, which is $O(pD)$. And the time for training the base model is denoted as $T_{\textup{train}}$, which equals to $O(nD^2)$ for LDA and QDA base classifiers. Similarly, the time for predicting test data is denoted as $T_{\textup{test}}$, which equals to $O(n_{\textup{test}}D)$ for LDA base classifier, $O(n_{\textup{test}}D^2)$ for QDA base classifier, and $O(n\cdot n_{\textup{test}}D)$ for $k$-NN base classifier. In total, the computation cost for the training process is $O(B_1B_2T_{\textup{train}} + B_1B_2\log B_2)$ time. Here, $O(B_2 \log B_2)$ is the time needed to find the optimal subspace among $B_2$ ones based on the sorting of their scores calculated under some criterion. The computation cost for prediction process is $O(B_1T_{\textup{test}})$ and RaSE algorithm takes approximately $O(B_1B_2T_{\textup{train}} + B_1B_2\log B_2 + B_1T_{\textup{test}})$ for both model fitting and prediction.

In practice, this type of framework is very convenient to apply the parallel computing paradigm, making the computing quite fast. And for specific classifiers like LDA and QDA, we have simplified the RIC expression, which can be directly used to speed up calculation. Compared to the projection generation process in \cite{cannings2017random}, RaSE is more efficient since we only need to select features based on certain distribution without doing any complicated matrix operations.

\subsection{Feature Ranking}
There are many powerful tools in statistics and machine learning for variable selection and feature ranking. For the sparse classification approaches like sparse-LDA and sparse-QDA \citep{mai2012direct, jiang2018direct, fan2012road, shao2011sparse, hao2018model, fan2015innovated, mai2015multiclass, fan2016feature}, or independent regularization  approach like nearest shrunken centroids \citep{tibshirani2003class}, this is usually directly implied by the methodology. For model-free classification methods, however, it's not straightforward to rank features. For random forest, \cite{breiman2001random} proposed a feature ranking method by randomly permuting the value of each feature and calculating the misclassification rate on the out-of-bag data for the new random forest.

For RaSE, as an ensemble framework, it's quite natural to rank variables by their frequencies of appearing in $B_1$ subspaces corresponding to $B_1$ weak learners. Following Corollary \ref{cor: coverage}, as $n, B_2 \rightarrow \infty$, under some conditions for signal strength and the increasing rate of $B_2$, by applying the criterion of minimizing RIC, the chosen subspace tends to cover the minimal discriminative set $S^*$ with high probability, which intuitively illustrates why this idea works to rank variables. When we do not have sufficient computational resources to set a very large $B_2$, as Theorem \ref{thm: iterative rase} indicates, under some conditions, the iterative RaSE (Algorithm \ref{algo_iteration}) can cover $S^*$ with high probability after a few steps with a smaller $B_2$. In practice, the frequencies of signals in $S^*$ tend to increase after iterations, which can improve the performance of the RaSE classifier and provide a better ranking. We will demonstrate this via extensive simulation studies in the next section.


\section{Simulations and Real-data Experiments}\label{sec: numerical}
We use six simulation settings and four real data sets to demonstrate the effectiveness of the RaSE method, coupled with RIC and leave-one-out cross-validation as the minimizing criterion to choose the optimal subspace. The performance of RaSE classifiers with LDA, QDA, Gamma and $k$NN as base classifiers with different iteration numbers are compared with that of other competitors, including the standard LDA, QDA, and $k$NN classifiers, sparse LDA (sLDA) \citep{mai2012direct}, regularization algorithm under marginality principle (RAMP) \citep{hao2018model}, nearest shrunken centroids (NSC) \citep{tibshirani2003class}, random forests (RF) \citep{breiman2001random}, and random projection ensemble classifier (RPEnsemble) \citep{cannings2017random}. For the LDA model, we also implemented the non-parametric estimate for RIC to show its effectiveness and robustness.

The standard LDA and QDA methods are implemented by using R package \texttt{MASS}. And the $k$NN classifier is implemented by \texttt{knn}, \texttt{knn.cv} in R package \texttt{class} and function \texttt{knn3} in R package \texttt{caret}. We utilize package \texttt{dsda} to fit sLDA model. RAMP is implemented through package \texttt{RAMP}. For the RF, we use R package \texttt{RandomForest}; the number of trees are set as 500 (as default) and $[\sqrt{p}]$ variables are randomly selected when training each tree (as default). And the NSC model is fitted by calling function \texttt{pamr.train} in package \texttt{pamr}. RPEnsemble is implemented by R package \texttt{RPEnsemble}. To obtain the MLE of parameters in the Gamma distribution, the Newton's method is applied via function \texttt{nlm} in package \texttt{stats} and the initial point is chosen to be the moment estimator. To get the non-parametric estimate of KL divergence and RIC, we call function \texttt{KL.divergence} in package \texttt{FNN}.

When fitting the standard $k$NN classifier, and the $k$NN base classifier in RPEnsemble and RaSE method, the number of neighbors $k$ is chosen from $\{3, 5, 7, 9, 11\}$ via leave-one-out cross-validation, following \cite{cannings2017random}. For RAMP, the response type is set as ``binomial", for which the logistic regression with interaction effects is considered. In sLDA, the penalty parameter $\lambda$ is chosen to minimize cross-validation error. In RPEnsemble method, LDA, QDA, $k$NN, are set as base classifiers with default parameters, and the number of weak learner $B_1 = 500$ and the number of projection candidates for each weak learner $B_2 = 50$ and the dimension of projected space $d = 5$. The projection generation method is set to be ``Haar". The criterion of choosing optimal projection is set to minimize training error for LDA and QDA and minimize leave-one-out cross-validation error for $k$NN. For the RaSE method, for LDA, QDA, and independent Gamma classifier (to be illustrated later in Section \ref{subsubsec: gamma}), the criterion is set to be minimizing RIC, and for $k$-NN, the strategy of minimizing leave-one-out cross-validation error is applied.  Other parameter settings in RaSE are the same as in the last section. For simulations, the number of iterations $T$ in Algorithm \ref{algo_iteration} is set to be 0, 1, 2, while for real-data experiments, we only consider RaSE methods with 0 or 1 iteration. We write the iteration number on the subscript, and if it is zero, the subscript will be omitted. For example, RaSE\textsubscript{1}-LDA represents the RaSE classifier with $T$ = 1 iteration and LDA base classifier. In addition, we use LDAn to denote the LDA base classifier with the non-parametric estimate of RIC.

For all experiments, 200 replicates are considered, and the average test errors (percentage) are calculated based on them. The standard deviation of the test errors over 200 replicates is also calculated for each approach and written on the subscript. The approach with minimal test error for each setting is highlighted in bold, and methods that achieve test error within one standard deviation of the best one are highlighted in italics. Also, for all simulations and the madelon data set, the average selected percentage of features in $B_1$ subspaces for the largest sample size setting in the RaSE method in 200 replicates are presented, which provides a natural way for feature ranking. For the average selected percentage in the case of the smallest sample size, refer to Appendix \ref{appendix b: additional fig}. To highlight the different behaviors of signals and noises, we present the average selected percentages of all noise features as a box-plot marked with ``N".

The RaSE classifier competes favorably with existing classification methods. Its misclassification rate is the lowest in 27 out of 30 (simulation and real-data) settings and within one standard deviation of the lowest in the remaining four settings.

\subsection{Simulations}
For the simulated data, model 1 follows model 1 in \cite{mai2012direct}, which is a sparse LDA-adapted setting. In model 2, for each class, a Gamma distribution with independent components is used. Model 3 follows from the setting of model 3 in \cite{fan2015innovated}, which is a QDA-adapted model. The marginal distribution for two classes is set to be $\pi_0 = \pi_1 = 0.5$ for the first three simulation models. Model 4 is motivated by the $k$NN algorithm, and the data generation process will be introduced below. To test the robustness of RaSE, two non-sparse settings, model 1' and 4' are investigated as well, where model 1' has decaying signals in the LDA model while model 4' inherits from model 4 by increasing the number of signals to 30 with the signal strength decreased.

For simulations, we consider the ``signal" model as a benchmark. These models use the correct model on the minimal discriminative set $S^*$, mimicking the behavior of the Bayes classifier when $S^*$ is sparse.

\subsubsection{Models 1 and 1' (LDA)}\label{subsubsec: model 1}
First we consider a sparse LDA setting (model 1). Let $\bm{x}|y = r \sim N(\bm{\mu}^{(r)}, \Sigma), r = 0, 1$, where $\Sigma = (\Sigma_{ij})_{p \times p} = (0.5^{|i-j|})_{p \times p} , \bm{\mu}^{(0)} = \bm{0}_{p \times 1}, \bm{\mu}^{(1)} = \Sigma\times 0.556(3, 1.5, 0, 0, 2, \bm{0}_{1 \times (p-5)})^T$. Here $p = 400$, and the training sample size $n \in \{200, 400, 1000\}$. Test data of size 1000 is independently generated from the same model.

As analyzed in Example \ref{exp: lda}, feature subset $\{1, 2, 5\}$ is the minimal discriminative set $S^*$. On the left panel of Table \ref{table: S1_S2}, the performance of various methods on model 1 for different sample sizes are presented. As we could see, RaSE\textsubscript{1}-LDAn performs the best when the sample size $n = 200$ and 400. sLDA achieves similar performances to the best classifiers for each setting, and RaSE\textsubscript{2}-QDA ranks the top when $n = 1000$. Also, since this model is very sparse, the default value of $B_2$ cannot guarantee that the minimal discriminative set can be selected. Therefore the iterative version of RaSE improves the performance of RaSE a lot. And NSC also achieves a comparably small misclassification rate when $n = 200$.

In Figure \ref{figure: model1_sparse_large}, the average selected percentages of 400 features in 200 replicates when $n = 1000$ are presented. Note that after two iterations, the three signals can be captured by almost all $B_1 = 200$ subspaces for all three base classifiers, and all noises are rarely selected across $B_1$ subspaces except when the non-parametric estimate of RIC is applied.

Next, we consider a non-sparse LDA model (model 1'). Let $\bm{\mu}^{(1)} = \Sigma\cdot (0.9, 0.9^2,\ldots, 0.9^{50},\\ \bm{0}_{1 \times (p-50)})^T$ and keep other parameters the same as above. Now $S^*$ contains the first 50 features. Under this non-sparse setting, as the right panel of Table \ref{table: S1_S2} shows, although most methods obtain similar error rates, RaSE\textsubscript{1}-LDAn achieves the best performance when $n = 200$ and 400. RaSE\textsubscript{1}-$k$NN performs the best when $n = 1000$. From the table, it can be seen that despite the non-sparse design, the iterations can still improve the performance of RaSE.

An interesting phenomenon is observed from Figure \ref{figure: model1_nonsparse_large}, which exhibits the average selected percentages for model 1'. Note that the selected percentages are decaying as the signal strength decreases except for the first feature. One possible reason is that the marginal discriminative powers of feature 2 and 3 are the strongest among all features due to the specific correlation structure in our setting.

\begin{table}[!h]
\centering
\begin{threeparttable}
\setlength{\tabcolsep}{6pt}
\begin{tabular}{Vl|ccc|cccV}
\Xhline{1pt}
\multirow{2}{*}{Method} & \multicolumn{3}{c|}{Results for model 1} & \multicolumn{3}{cV}{Results for model 1'} \\ \cline{2-7}
 & \multicolumn{1}{r}{$n=200$} & $n=400$ & $n=1000$ & $n=200$ & $n=400$ & $n=1000$ \\ \hline
RaSE-LDA& \multicolumn{1}{r}{12.99\textsubscript{1.42} }& 12.38\textsubscript{1.20} & \textit{11.16}\textsubscript{1.10} & 21.43\textsubscript{4.64} & \textit{19.56}\textsubscript{3.67} & \textit{17.98}\textsubscript{2.74} \\
RaSE-LDAn& \multicolumn{1}{r}{13.40\textsubscript{1.31} }& 13.23\textsubscript{1.20} & 12.64\textsubscript{1.11} & 21.64\textsubscript{4.22} & 20.25\textsubscript{3.01} & 19.42\textsubscript{2.59} \\
RaSE-QDA& \multicolumn{1}{r}{13.78\textsubscript{1.26} }& 13.69\textsubscript{1.19} & 13.23\textsubscript{1.04} & 22.36\textsubscript{3.96} & 20.46\textsubscript{2.83} & 19.83\textsubscript{2.51} \\
RaSE-$k$NN& \multicolumn{1}{r}{13.21\textsubscript{1.48} }& 12.57\textsubscript{1.23} & 11.19\textsubscript{1.10} & \textit{19.98}\textsubscript{3.74} & \textit{18.21}\textsubscript{2.89} & \textit{16.97}\textsubscript{2.10} \\
RaSE\textsubscript{1}-LDA& \multicolumn{1}{r}{\textit{11.48}\textsubscript{1.26} }& \textit{10.42}\textsubscript{1.07} & \textit{10.25}\textsubscript{1.08} & \textit{20.15}\textsubscript{3.65} & \textit{18.39}\textsubscript{2.69} & \textit{17.20}\textsubscript{2.14} \\
RaSE\textsubscript{1}-LDAn& \multicolumn{1}{r}{\textbf{10.58}\textsubscript{1.15} }& \textbf{10.28}\textsubscript{1.01} & \textit{10.16}\textsubscript{1.04} & \textbf{18.21}\textsubscript{2.92} & \textbf{17.55}\textsubscript{2.26} & \textit{16.86}\textsubscript{1.74} \\
RaSE\textsubscript{1}-QDA& \multicolumn{1}{r}{\textit{11.28}\textsubscript{1.47} }& \textit{10.76}\textsubscript{1.25} & \textit{10.40}\textsubscript{1.22} & 21.24\textsubscript{4.75} & \textit{19.65}\textsubscript{3.56} & \textit{18.23}\textsubscript{2.51} \\
RaSE\textsubscript{1}-$k$NN& \multicolumn{1}{r}{\textit{11.11}\textsubscript{1.25} }& \textit{10.58}\textsubscript{1.09} & \textit{10.33}\textsubscript{1.05} & \textit{18.90}\textsubscript{3.16} & \textit{17.75}\textsubscript{2.51} & \textbf{16.80}\textsubscript{1.98} \\
RaSE\textsubscript{2}-LDA& \multicolumn{1}{r}{12.62\textsubscript{1.45} }& 11.41\textsubscript{1.16} & \textit{10.13}\textsubscript{1.03} & 21.78\textsubscript{3.44} & \textit{19.22}\textsubscript{2.54} & \textit{17.53}\textsubscript{1.98} \\
RaSE\textsubscript{2}-LDAn& \multicolumn{1}{r}{\textit{10.90}\textsubscript{1.16} }& \textit{10.42}\textsubscript{0.96} & \textit{10.17}\textsubscript{1.06} & \textit{18.75}\textsubscript{2.87} & \textit{17.95}\textsubscript{2.19} & \textit{17.30}\textsubscript{1.70} \\
RaSE\textsubscript{2}-QDA& \multicolumn{1}{r}{11.74\textsubscript{1.45} }& \textit{10.39}\textsubscript{1.01} & \textbf{10.09}\textsubscript{1.07} & 21.84\textsubscript{5.03} & 20.00\textsubscript{3.62} & 19.04\textsubscript{2.57} \\
RaSE\textsubscript{2}-$k$NN& \multicolumn{1}{r}{\textit{11.17}\textsubscript{1.31} }& \textit{10.60}\textsubscript{0.99} & \textit{10.34}\textsubscript{1.02} & \textit{18.96}\textsubscript{3.10} & \textit{17.77}\textsubscript{2.29} & \textit{17.19}\textsubscript{1.78} \\
RP-LDA& \multicolumn{1}{r}{17.26\textsubscript{1.53} }& 15.03\textsubscript{1.24} & 13.37\textsubscript{1.09} & 25.66\textsubscript{1.98} & 23.85\textsubscript{1.79} & 21.99\textsubscript{1.50} \\
RP-QDA& \multicolumn{1}{r}{17.96\textsubscript{1.60} }& 15.26\textsubscript{1.34} & 13.50\textsubscript{1.12} & 26.41\textsubscript{2.04} & 24.03\textsubscript{1.88} & 22.06\textsubscript{1.51} \\
RP-$k$NN& \multicolumn{1}{r}{18.54\textsubscript{1.60} }& 16.15\textsubscript{1.21} & 14.23\textsubscript{1.21} & 27.03\textsubscript{2.27} & 25.06\textsubscript{1.76} & 23.01\textsubscript{1.55} \\
LDA& ---$\dagger$ & 46.39\textsubscript{2.62} & 18.62\textsubscript{1.53} & ---$\dagger$ & 47.81\textsubscript{2.87} & 27.44\textsubscript{1.95} \\
QDA& ---$\dagger$ & ---$\dagger$ & 47.74\textsubscript{1.73} & ---$\dagger$ & ---$\dagger$ & 48.90\textsubscript{1.65} \\
$k$NN& \multicolumn{1}{r}{29.08\textsubscript{2.77} }& 26.73\textsubscript{1.97} & 24.53\textsubscript{1.60} & 35.67\textsubscript{2.59} & 34.07\textsubscript{2.33} & 32.36\textsubscript{1.72} \\
sLDA& \multicolumn{1}{r}{\textit{10.80}\textsubscript{1.26} }& \textit{10.48}\textsubscript{1.16} & \textit{10.23}\textsubscript{1.07} & \textit{18.80}\textsubscript{3.19} & \textit{17.62}\textsubscript{2.28} & \textit{17.24}\textsubscript{1.81} \\
RAMP& \multicolumn{1}{r}{14.09\textsubscript{2.67} }& \textit{10.56}\textsubscript{1.33} & \textit{10.10}\textsubscript{1.04} & 21.91\textsubscript{5.49} & \textit{19.22}\textsubscript{3.25} & \textit{17.55}\textsubscript{2.04} \\
NSC& \multicolumn{1}{r}{\textit{11.50}\textsubscript{1.13} }& 11.50\textsubscript{0.97} & 11.67\textsubscript{1.09} & \textit{18.49}\textsubscript{1.99} & \textit{19.02}\textsubscript{1.78} & 19.41\textsubscript{1.39} \\
RF& \multicolumn{1}{r}{12.66\textsubscript{1.43} }& 11.77\textsubscript{1.04} & 11.25\textsubscript{1.10} & 21.33\textsubscript{3.01} & 20.00\textsubscript{2.10} & 19.25\textsubscript{1.51} \\
Sig-LDA& \multicolumn{1}{r}{10.07\textsubscript{0.94} }& 10.09\textsubscript{0.95} & 10.07\textsubscript{1.03} & 23.70\textsubscript{3.14} & 20.90\textsubscript{2.27} & 18.95\textsubscript{1.64} \\ \Xhline{1pt}
\end{tabular}
\begin{tablenotes}
\item $\dagger$ Not applicable.
\end{tablenotes}
\end{threeparttable}
\caption{Error rates for models 1 and 1'}
\label{table: S1_S2}
\end{table}

\begin{figure}[!h]
  \includegraphics[width = \textwidth]{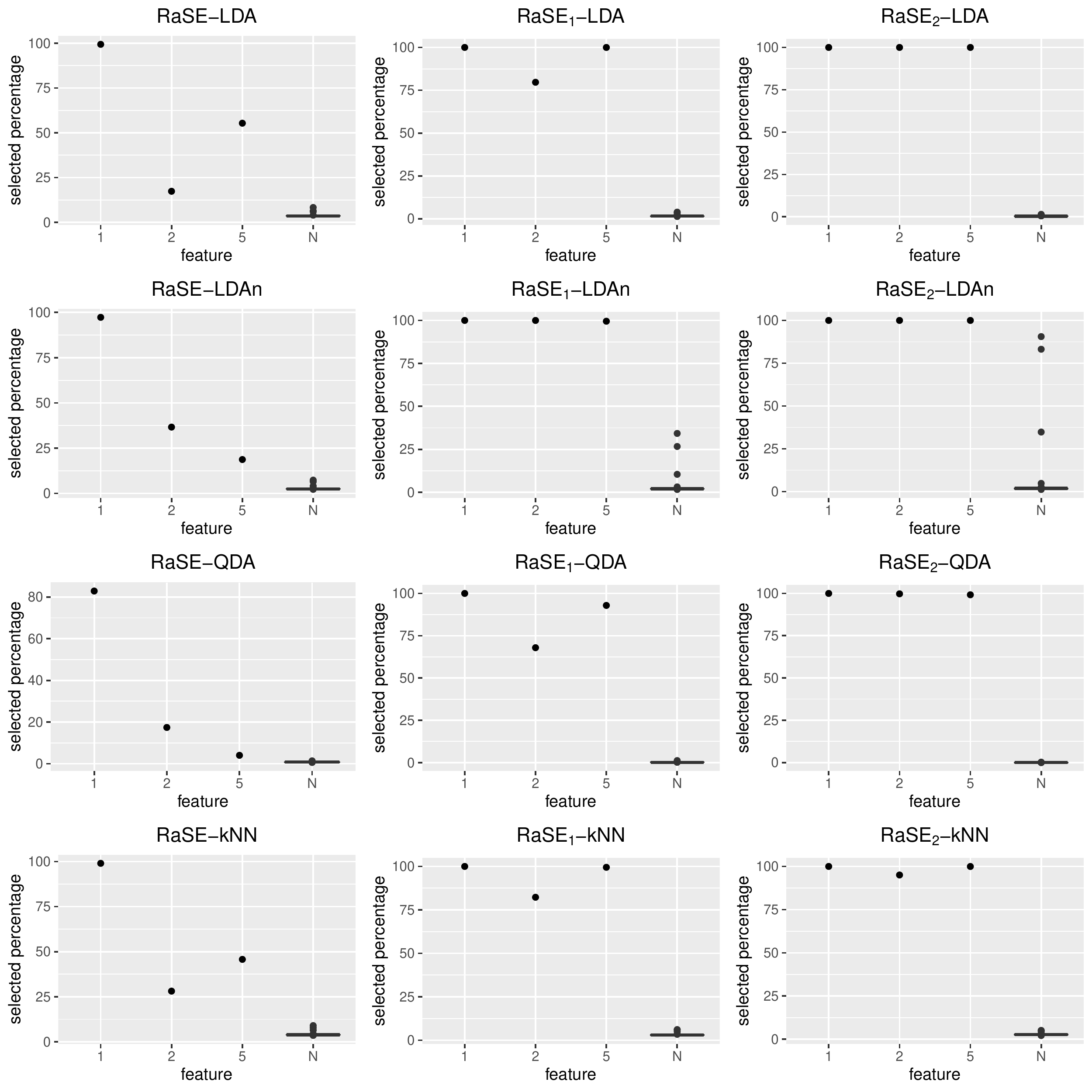}
  \caption{Average selected percentages of features for model 1 in 200 replicates when $n = 1000$}
  \label{figure: model1_sparse_large}
\end{figure}

\begin{figure}[!h]
  \includegraphics[width = \textwidth]{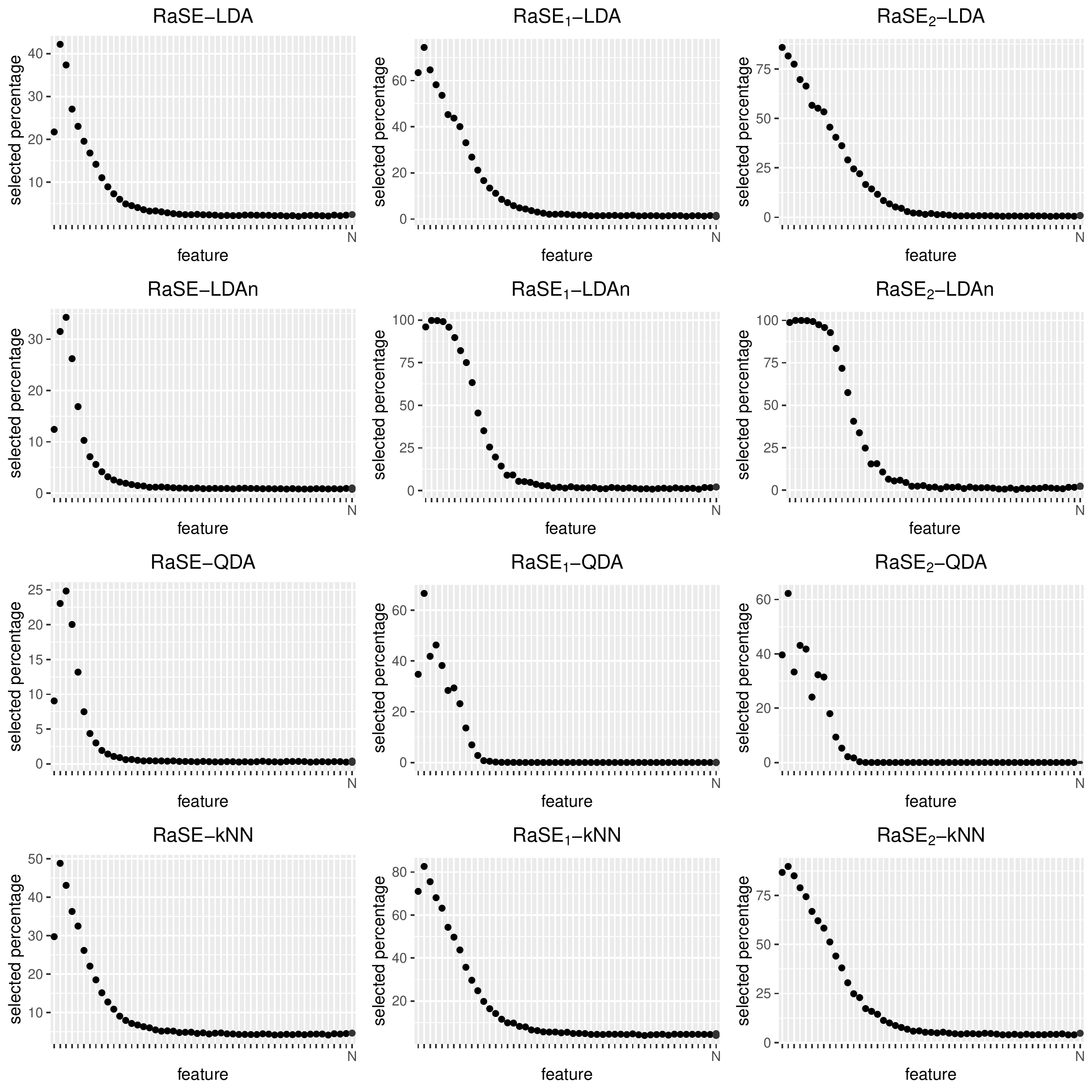}
  \caption{Average selected percentages of features for model 1' in 200 replicates when $n = 1000$}
  \label{figure: model1_nonsparse_large}
\end{figure}
\subsubsection{Model 2 (Gamma distribution)}\label{subsubsec: gamma}
In this model we investigate the Gamma distribution with independent features, which is rarely studied in the literature. $\bm{x}|y = r$ at $j$-th coordinate follows Gamma distribution $Gamma(\alpha_j^{(r)}, \beta_j^{(r)})$, which has the density function
\begin{equation}\label{eq: gamma distribution}
	f_j^{(r)}(x;\alpha_j^{(r)}, \beta_j^{(r)}) = \frac{1}{(\beta_j^{(r)})^{\alpha_j^{(r)}}\Gamma(\alpha_j^{(r)})}x^{\alpha_j^{(r)} - 1}\exp\{-x/\beta_j^{(r)}\}\mathds{1}(x \geq 0), j = 1, \ldots, p, r = 0, 1.
\end{equation}
Denote $\bm{\alpha}^{(r)} = (\alpha^{(r)}_1, \ldots, \alpha^{(r)}_p)^T, \bm{\beta}^{(r)} = (\beta^{(r)}_1, \ldots, \beta^{(r)}_p)^T$. Here, we let $\bm{\alpha}^{(0)} = (2,1.5,1.5,2,2, \\ \bm{1}_{1 \times (p-5)})^T$, $\bm{\alpha}^{(1)} = (2.5,1.5,1.5,1,1, \bm{1}_{1 \times (p-5)})^T, \bm{\beta}^{(0)} = (1.5,3,1,1,1, \bm{1}_{1\times (p-5)})^T, \bm{\beta}^{(1)} = (2,1,3, 1, 1, \bm{1}_{1 \times (p-5)})^T, p = 400, n \in \{100, 200, 400\}$. Hence, the minimal discriminative set $S^*$ is $\{1, 2, 3, 4, 5\}$, due to Proposition \ref{prop: discriminative set}. 

MLEs of $\bm{\alpha}^{(0)}, \bm{\alpha}^{(1)}, \bm{\beta}^{(0)}, \bm{\beta}^{(1)}$ can be obtained by numerical approaches like the gradient descent or Newton's method. And the marginal probabilities are estimated by the proportion of two classes in training samples. Then the Bayes classifier is estimated by these MLEs and applied to classify new observations, which is denoted as an independent Gamma classifier in Table \ref{table: S3_S4}. For this example, we also apply RaSE with the independent Gamma classifier as one of the base classifiers. According to \eqref{eq: bayes classifier} and \eqref{eq: gamma distribution}, the decision function of independent Gamma classifier estimated in subspace $S$ is

\begin{equation}
	C_n^{S-Gamma} (\bm{x}) = \mathds{1}\left(\frac{\hat{\pi}_1}{\hat{\pi}_0}\cdot \prod_{j \in S}\frac{f_j^{(1)}(x_j; \hat\alpha_j^{(1)}, \hat\beta_j^{(1)})}{f_j^{(0)}(x_j; \hat\alpha_j^{(0)}, \hat\beta_j^{(0)})} > 0.5\right),
\end{equation}
where $\hat{\pi}_1, \hat{\pi}_0, \hat{\alpha}_j^{(1)}, \hat{\alpha}_j^{(0)}, \hat{\beta}_j^{(1)},\hat{\beta}_j^{(0)}$ are corresponding MLEs.

This is also a very sparse model. Therefore the iteration process can improve the RaSE method with a lower misclassification rate. The left panel of Table \ref{table: S3_S4} shows us the performance of various methods on this model. It demonstrates that RaSE\textsubscript{1}-Gamma performs the best when $n = 100$. RaSE\textsubscript{2}-Gamma incurs the lowest misclassification rate and low standard deviation for the other two settings. 

Figure \ref{figure: model2_large} shows us the average selected percentage of features when $n = 400$, from which we can see that due to the high sparsity, the default $B_2$  is not sufficient and makes it hard for the vanilla RaSE classifier to capture all the features in $S^*$. After iterations, the frequencies of the minimal discriminative set increase significantly, and $S^*$ can be easily identified for all three base classifiers after two iterations. 

\begin{figure}[!h]
  \centering
  \includegraphics[width = \textwidth]{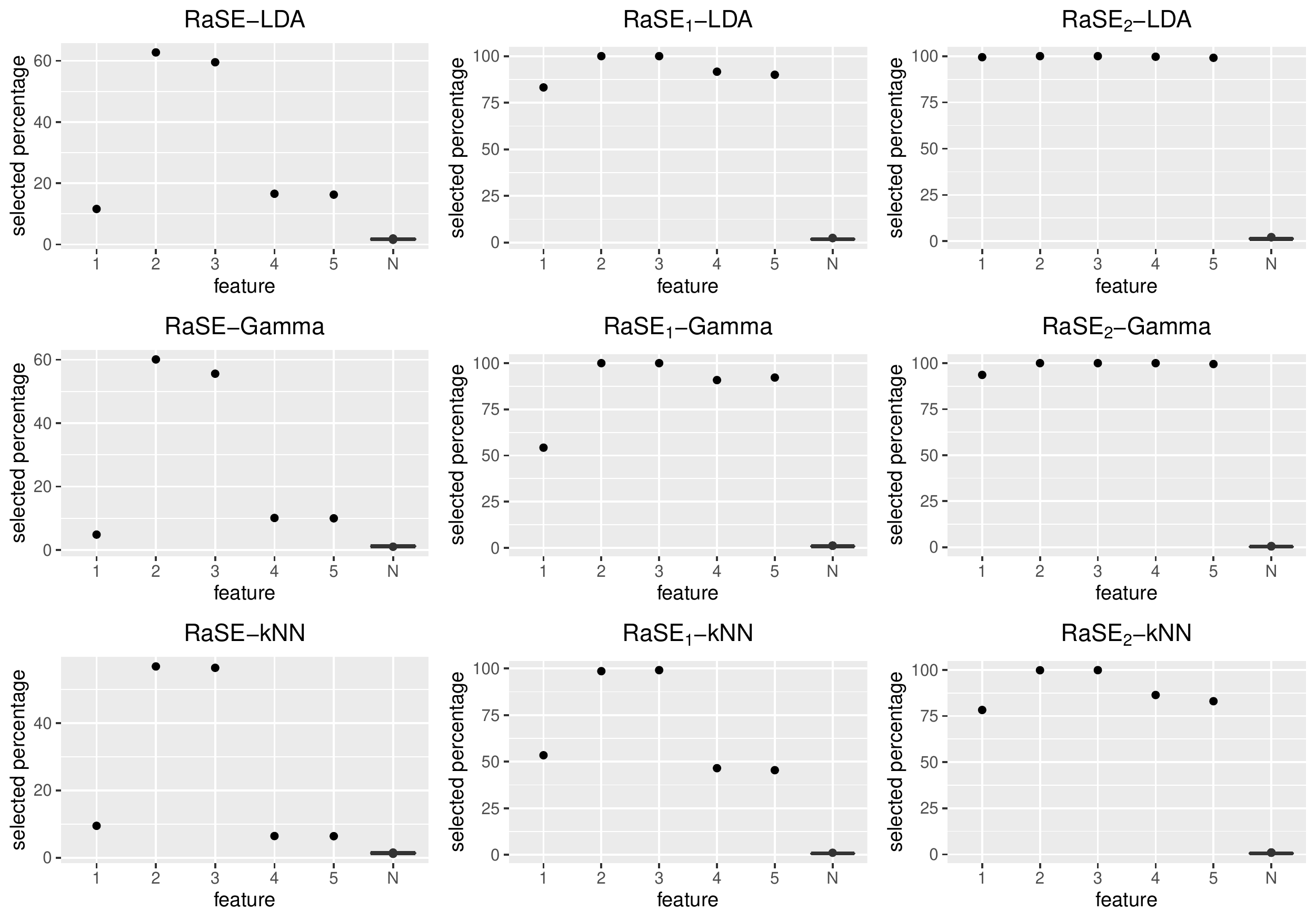}
  \caption{Average selected percentages of features for model 2 in 200 replicates when $n = 400$}
  \label{figure: model2_large}
\end{figure}

\subsubsection{Model 3 (QDA)}
$\bx|y = r \sim N(\bm{\mu}^{(r)}, \Sigma^{(r)}), r = 0, 1$, where $\Omega^{(0)} = (\Sigma^{(0)})^{-1}$ is a $p \times p$ band matrix with $(\Omega^{(0)})_{ii} = 1$ and $(\Omega^{(0)})_{ik} = 0.3$ for $|i - k| = 1$. All the other entries are zero. $\Omega^{(1)} = \Omega^{(0)} + \Omega$, where $\Omega$ is a $p \times p$ sparse symmetric matrix with $\Omega_{10, 10} = -0.3758, \Omega_{10, 30} = 0.0616, \Omega_{10, 50} = 0.2037, \Omega_{30, 30} = -0.5482, \Omega_{30, 50} = 0.0286, \Omega_{50, 50} = -0.4614$ and all the other entries are zero. Here $p = 200, n \in \{200, 400, 1000\}$.

As analyzed in Example \ref{exp: qda},  the minimal discriminative set $S^* = \{1, 2, 10, 30, 50\}$, where features 1 and 2 represent linear signals or main effects and features 10, 30, and 50 are quadratic signals. The right panel of Table \ref{table: S3_S4} shows us the results. From it we can see that RaSE\textsubscript{1}-QDA achieves the best performance when $n = 200$, and RaSE\textsubscript{2}-QDA has the lowest test misclassification rate when $n = 400, 1000$, which is reasonable since data is generated from a QDA-adapted model. RaSE\textsubscript{1}-$k$NN and RaSE\textsubscript{2}-$k$NN also do a good job when $n$ is large.

Figure \ref{figure: model3_large} represents the average selected percentage of features when $n = 1000$,  which exhibits that the frequencies of the elements in the minimal discriminative set are increasing after iterations. When the base classifier is QDA or $k$NN, all the five features stand out. On the other hand, RaSE, with the LDA base classifier, only captures the two linear signals, which is expected since LDA does not consider the quadratic terms.

\begin{table}[!h]
\centering
\begin{threeparttable}
\setlength{\tabcolsep}{6pt}
\begin{tabular}{Vl|ccc|cccV}
\Xhline{1pt}
\multirow{2}{*}{Method\tablefootnote{To save space, for the rows involving ``Gamma/QDA", it represents the independent Gamma classifier in model 2 and the QDA in model 3.}} & \multicolumn{3}{c|}{Results for model 2} & \multicolumn{3}{cV}{Results for model 3} \\ \cline{2-7}
 & \multicolumn{1}{r}{$n=100$} & $n=200$ & $n=400$ & $n=200$ & $n=400$ & $n=1000$ \\ \hline
RaSE-LDA& \multicolumn{1}{r}{21.51\textsubscript{3.42} }& 18.52\textsubscript{2.75} & 17.43\textsubscript{2.00} & 37.30\textsubscript{3.17} & 36.11\textsubscript{1.97} & 35.67\textsubscript{1.73} \\
RaSE-Gamma/QDA& \multicolumn{1}{r}{23.57\textsubscript{3.65} }& 21.41\textsubscript{3.17} & 20.29\textsubscript{2.43} & 32.52\textsubscript{2.90} & 30.44\textsubscript{2.60} & 29.00\textsubscript{1.97} \\
RaSE-$k$NN& \multicolumn{1}{r}{22.83\textsubscript{3.09} }& 21.07\textsubscript{3.07} & 20.47\textsubscript{2.57} & 31.10\textsubscript{3.23} & 27.83\textsubscript{2.41} & 25.22\textsubscript{1.56} \\
RaSE\textsubscript{1}-LDA& \multicolumn{1}{r}{19.01\textsubscript{2.83} }& 15.14\textsubscript{1.75} & 13.55\textsubscript{1.20} & 36.09\textsubscript{2.87} & 32.82\textsubscript{1.74} & 32.68\textsubscript{1.49} \\
RaSE\textsubscript{1}-Gamma/QDA& \multicolumn{1}{r}{\textbf{15.05}\textsubscript{2.31} }& \textit{13.02}\textsubscript{1.51} & \textit{12.50}\textsubscript{1.18} & \textbf{26.83}\textsubscript{2.47} & \textit{25.07}\textsubscript{1.89} & \textit{23.53}\textsubscript{1.50} \\
RaSE\textsubscript{1}-$k$NN& \multicolumn{1}{r}{19.84\textsubscript{2.90} }& 16.64\textsubscript{1.86} & 15.33\textsubscript{1.39} & \textit{28.76}\textsubscript{2.60} & \textit{25.88}\textsubscript{1.98} & \textit{24.18}\textsubscript{1.47} \\
RaSE\textsubscript{2}-LDA& \multicolumn{1}{r}{20.88\textsubscript{3.03} }& 16.92\textsubscript{2.05} & 13.66\textsubscript{1.14} & 38.09\textsubscript{2.48} & 33.69\textsubscript{1.83} & 32.71\textsubscript{1.55} \\
RaSE\textsubscript{2}-Gamma/QDA& \multicolumn{1}{r}{\textit{16.27}\textsubscript{2.22} }& \textbf{12.64}\textsubscript{1.31} & \textbf{11.83}\textsubscript{1.02} & \textit{26.99}\textsubscript{2.68} & \textbf{24.87}\textsubscript{1.99} & \textbf{23.11}\textsubscript{1.60} \\
RaSE\textsubscript{2}-$k$NN& \multicolumn{1}{r}{22.39\textsubscript{3.13} }& 17.40\textsubscript{2.27} & 14.59\textsubscript{1.44} & \textit{28.73}\textsubscript{2.56} & \textit{25.46}\textsubscript{1.82} & \textit{23.76}\textsubscript{1.54} \\
RP-LDA& \multicolumn{1}{r}{38.89\textsubscript{1.96} }& 35.00\textsubscript{1.97} & 30.89\textsubscript{1.76} & 44.90\textsubscript{1.86} & 42.82\textsubscript{1.76} & 40.38\textsubscript{1.74} \\
RP-QDA& \multicolumn{1}{r}{43.62\textsubscript{3.62} }& 37.62\textsubscript{2.35} & 33.02\textsubscript{2.14} & 43.02\textsubscript{2.07} & 39.87\textsubscript{1.88} & 36.38\textsubscript{1.78} \\
RP-$k$NN& \multicolumn{1}{r}{41.48\textsubscript{2.26} }& 38.90\textsubscript{2.06} & 36.69\textsubscript{1.85} & 44.32\textsubscript{1.81} & 42.46\textsubscript{1.58} & 40.80\textsubscript{2.12} \\
LDA& ---$\dagger$ & ---$\dagger$ & 47.50\textsubscript{2.35} & 49.03\textsubscript{1.94} & 42.88\textsubscript{1.82} & 38.68\textsubscript{1.70} \\
QDA& \multicolumn{1}{r}{32.06\textsubscript{2.40} }& 26.22\textsubscript{1.80} & 21.56\textsubscript{1.44} & ---$\dagger$ & ---$\dagger$ & 45.13\textsubscript{1.58} \\
$k$NN& \multicolumn{1}{r}{45.48\textsubscript{2.24} }& 44.68\textsubscript{2.14} & 44.07\textsubscript{2.00} & 45.67\textsubscript{1.78} & 44.63\textsubscript{2.02} & 43.43\textsubscript{1.63} \\
sLDA& \multicolumn{1}{r}{22.26\textsubscript{3.52} }& 18.64\textsubscript{2.12} & 15.34\textsubscript{1.55} & 36.41\textsubscript{3.15} & 33.87\textsubscript{2.01} & 32.99\textsubscript{1.52} \\
RAMP& \multicolumn{1}{r}{20.64\textsubscript{3.81} }& 16.72\textsubscript{2.25} & 13.31\textsubscript{1.21} & 36.94\textsubscript{5.87} & 32.65\textsubscript{1.89} & 32.42\textsubscript{1.78} \\
NSC& \multicolumn{1}{r}{26.00\textsubscript{6.49} }& 19.92\textsubscript{4.31} & 16.87\textsubscript{2.69} & 41.14\textsubscript{4.49} & 38.24\textsubscript{3.85} & 35.13\textsubscript{2.20} \\
RF& \multicolumn{1}{r}{24.97\textsubscript{5.74} }& 18.02\textsubscript{2.77} & 15.26\textsubscript{1.47} & 37.34\textsubscript{2.91} & 31.61\textsubscript{2.19} & 27.42\textsubscript{1.60} \\
Sig-Gamma/QDA& \multicolumn{1}{r}{12.65\textsubscript{1.12} }& 12.12\textsubscript{1.12} & 11.76\textsubscript{0.97} & 23.62\textsubscript{1.47} & 22.72\textsubscript{1.40} & 22.16\textsubscript{1.31} \\ \Xhline{1pt}
\end{tabular}
\begin{tablenotes}
\item $\dagger$ Not applicable.
\end{tablenotes}
\end{threeparttable}
\caption{Error rates for models 2 and 3}
\label{table: S3_S4}
\end{table}

\begin{figure}[!h]
  \includegraphics[width = \textwidth]{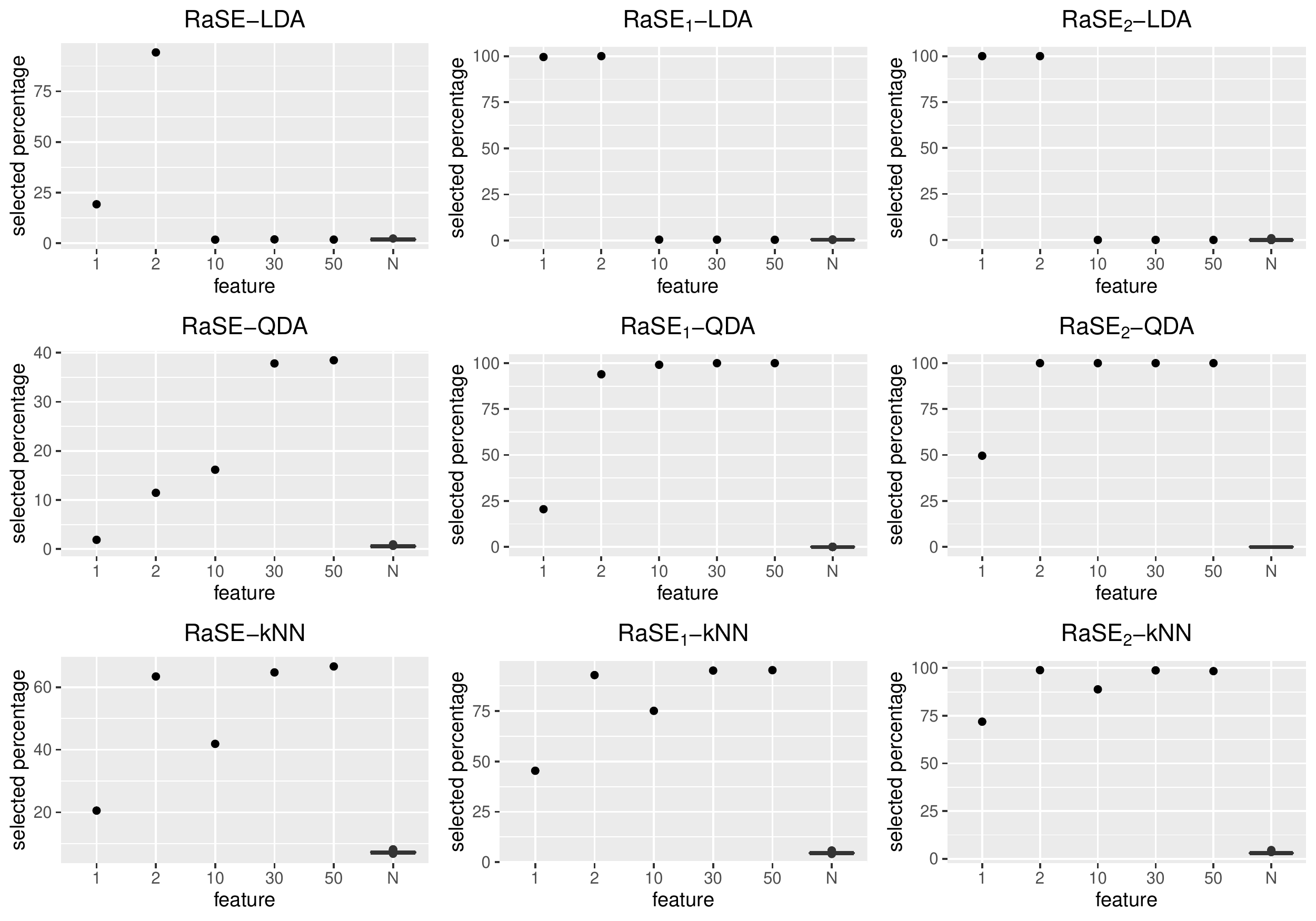}
  \caption{Average selected percentages of features for model 3 in 200 replicates when $n = 1000$}
  \label{figure: model3_large}
\end{figure}


\subsubsection{Models 4 and 4' ($k$NN)}
As in the LDA models, we first study a sparse setting (model 4) with the data generating process motivated by the $k$NN classifier. First, 10 initial points $\bm{z}_1, \ldots, \bm{z}_{10}$ are generated i.i.d. from $N(\bm{0}_{p \times 1}, I_p)$, five of which are labeled as 0 and the other five are labeled as 1. Then each time one of $\{\bm{z}_1, \ldots, \bm{z}_{10}\}$  is randomly selected (suppose $\bm{z}_{k_i}$) and we then generate $\bm{x}_i \sim N((\bm{z}_{k_i, S^*}^T, \bm{0}_{1 \times (p-5)})^T, 0.5^2 I_p)$. Here the minimal discriminative set is $S^* = \{1, 2, 3, 4, 5\}$, $p = 200$, and $n \in \{200, 400, 1000\}$.  The results are presented in Table \ref{table: S5_S6} and the average selected percentages of features in $B_1 = 200$ subspaces are presented in Figure \ref{figure: model4_sparse_large}.

From the left panel of Table \ref{table: S5_S6}, it can be seen that the performance of RaSE\textsubscript{2}-$k$NN is surprising. It outperforms all the other methods, and the difference between its misclassification rate and others is very prominent. Note that in this case, the Sig-$k$NN classifier is calculated by applying $k$NN on only the first five features and also uses leave-one-out cross-validation to choose $k$ from $\{3, 5, 7, 9, 11\}$. Note that it is not the optimal classifier due to the lack of a true model, which explains why RaSE\textsubscript{2}-$k$NN can even achieve a better performance when $n = 400, 1000$. 

 Figure \ref{figure: model4_sparse_large} shows that RaSE, based on all the three base classifiers, can capture features in the minimal discriminative set.

Now, we study a non-sparse setting (model 4'), where the number of signals are increased to 30 and each $\bm{x}_i \sim N((\bm{z}_{k_i, S^*}^T, \bm{0}_{1 \times (p-30)})^T, 2 I_p)$. The other parameters and the data generation mechanism are the same as model 4. From the right panel of Table \ref{table: S5_S6}, we observe that RaSE\textsubscript{2}-$k$NN still achieves the best performance, and the iterations improve the performance of RaSE classifiers under this non-sparse setting. In addition, Figure \ref{figure: model4_nonsparse_large} shows that RaSE can capture the signals while the noises keep a low selected percentage, which again verifies the robustness of RaSE.

\begin{table}[!h]
\centering
\begin{threeparttable}
\setlength{\tabcolsep}{6pt}
\begin{tabular}{Vl|ccc|cccV}
\Xhline{1pt}
\multirow{2}{*}{Method} & \multicolumn{3}{c|}{Results for model 4} & \multicolumn{3}{cV}{Results for model 4'} \\ \cline{2-7}
 & \multicolumn{1}{r}{$n=200$} & $n=400$ & $n=1000$ & $n=200$ & $n=400$ & $n=1000$ \\ \hline
RaSE-LDA& \multicolumn{1}{r}{27.38\textsubscript{9.53} }& 25.08\textsubscript{8.40} & 24.92\textsubscript{9.05} & 26.50\textsubscript{5.78} & 23.84\textsubscript{4.98} & 20.92\textsubscript{5.20} \\
RaSE-QDA& \multicolumn{1}{r}{24.24\textsubscript{7.20} }& 22.62\textsubscript{6.77} & 22.04\textsubscript{6.73} & 29.28\textsubscript{4.34} & 26.73\textsubscript{3.63} & 24.77\textsubscript{3.35} \\
RaSE-$k$NN& \multicolumn{1}{r}{13.26\textsubscript{5.03} }& 10.67\textsubscript{4.44} & 8.85\textsubscript{4.05} & 20.83\textsubscript{4.50} & 15.70\textsubscript{3.74} & 10.33\textsubscript{2.90} \\
RaSE\textsubscript{1}-LDA& \multicolumn{1}{r}{25.89\textsubscript{10.13} }& 23.05\textsubscript{8.49} & 23.42\textsubscript{8.95} & 21.28\textsubscript{4.76} & 18.59\textsubscript{3.97} & 16.87\textsubscript{3.92} \\
RaSE\textsubscript{1}-QDA& \multicolumn{1}{r}{13.83\textsubscript{5.88} }& 12.54\textsubscript{5.79} & 12.70\textsubscript{5.51} & 21.97\textsubscript{5.37} & 19.32\textsubscript{4.65} & 15.84\textsubscript{4.36} \\
RaSE\textsubscript{1}-$k$NN& \multicolumn{1}{r}{\textit{7.51}\textsubscript{3.80} }& \textit{6.16}\textsubscript{3.48} & \textit{5.90}\textsubscript{3.12} & \textit{16.24}\textsubscript{3.52} & 11.09\textsubscript{2.84} & 7.08\textsubscript{2.04} \\
RaSE\textsubscript{2}-LDA& \multicolumn{1}{r}{27.49\textsubscript{10.13} }& 23.39\textsubscript{8.51} & 23.39\textsubscript{9.05} & 20.73\textsubscript{4.99} & 17.34\textsubscript{3.99} & 15.72\textsubscript{3.93} \\
RaSE\textsubscript{2}-QDA& \multicolumn{1}{r}{13.15\textsubscript{5.00} }& 11.90\textsubscript{5.38} & 12.15\textsubscript{5.22} & 20.94\textsubscript{5.06} & 18.61\textsubscript{4.61} & 14.96\textsubscript{4.11} \\
RaSE\textsubscript{2}-$k$NN& \multicolumn{1}{r}{\textbf{7.06}\textsubscript{3.62} }& \textbf{5.89}\textsubscript{3.32} & \textbf{5.74}\textsubscript{3.02} & \textbf{13.62}\textsubscript{3.31} & \textbf{8.64}\textsubscript{2.40} & \textbf{5.36}\textsubscript{1.58} \\
RP-LDA& \multicolumn{1}{r}{28.03\textsubscript{8.91} }& 25.48\textsubscript{7.80} & 24.83\textsubscript{8.09} & 22.16\textsubscript{4.49} & 18.84\textsubscript{4.42} & 16.84\textsubscript{4.10} \\
RP-QDA& \multicolumn{1}{r}{26.22\textsubscript{7.66} }& 23.93\textsubscript{6.97} & 22.75\textsubscript{7.00} & 21.37\textsubscript{4.29} & 17.58\textsubscript{3.84} & 15.53\textsubscript{3.73} \\
RP-$k$NN& \multicolumn{1}{r}{26.63\textsubscript{8.09} }& 24.49\textsubscript{7.15} & 23.32\textsubscript{7.35} & 22.37\textsubscript{4.69} & 18.69\textsubscript{4.19} & 16.50\textsubscript{3.95} \\
LDA& \multicolumn{1}{r}{47.51\textsubscript{2.66} }& 33.27\textsubscript{7.65} & 27.89\textsubscript{8.97} & 46.06\textsubscript{3.05} & 25.16\textsubscript{4.12} & 17.78\textsubscript{4.05} \\
QDA& ---$\dagger$ & ---$\dagger$ & 36.70\textsubscript{4.83} & ---$\dagger$ & ---$\dagger$ & 30.45\textsubscript{2.89} \\
$k$NN& \multicolumn{1}{r}{24.49\textsubscript{6.64} }& 21.04\textsubscript{6.60} & 19.07\textsubscript{6.50} & 24.73\textsubscript{4.35} & 20.06\textsubscript{3.95} & 15.91\textsubscript{3.59} \\
sLDA& \multicolumn{1}{r}{24.90\textsubscript{9.41} }& 22.80\textsubscript{8.19} & 23.22\textsubscript{8.79} & 19.78\textsubscript{4.92} & 16.41\textsubscript{3.98} & 14.59\textsubscript{3.69} \\
RAMP& \multicolumn{1}{r}{22.14\textsubscript{11.50} }& 15.01\textsubscript{7.83} & 12.82\textsubscript{7.13} & 24.79\textsubscript{6.49} & 18.59\textsubscript{4.56} & 13.13\textsubscript{3.13} \\
NSC& \multicolumn{1}{r}{27.70\textsubscript{9.44} }& 25.71\textsubscript{8.35} & 25.82\textsubscript{8.72} & 22.09\textsubscript{6.20} & 18.17\textsubscript{4.51} & 16.17\textsubscript{4.09} \\
RF& \multicolumn{1}{r}{23.64\textsubscript{8.05} }& 17.39\textsubscript{6.19} & 14.84\textsubscript{5.73} & 21.73\textsubscript{5.44} & 15.70\textsubscript{3.65} & 11.54\textsubscript{2.80} \\
Sig-$k$NN& \multicolumn{1}{r}{6.89\textsubscript{3.40} }& 6.03\textsubscript{3.37} & 6.01\textsubscript{3.21} & 7.60\textsubscript{2.19} & 5.40\textsubscript{1.76} & 4.06\textsubscript{1.43} \\ \Xhline{1pt}
\end{tabular}
\begin{tablenotes}
\item $\dagger$ Not applicable.
\end{tablenotes}
\end{threeparttable}
\caption{Error rates for models 4 and 4'}
\label{table: S5_S6}
\end{table}

\begin{figure}[!h]
  \includegraphics[width = \textwidth]{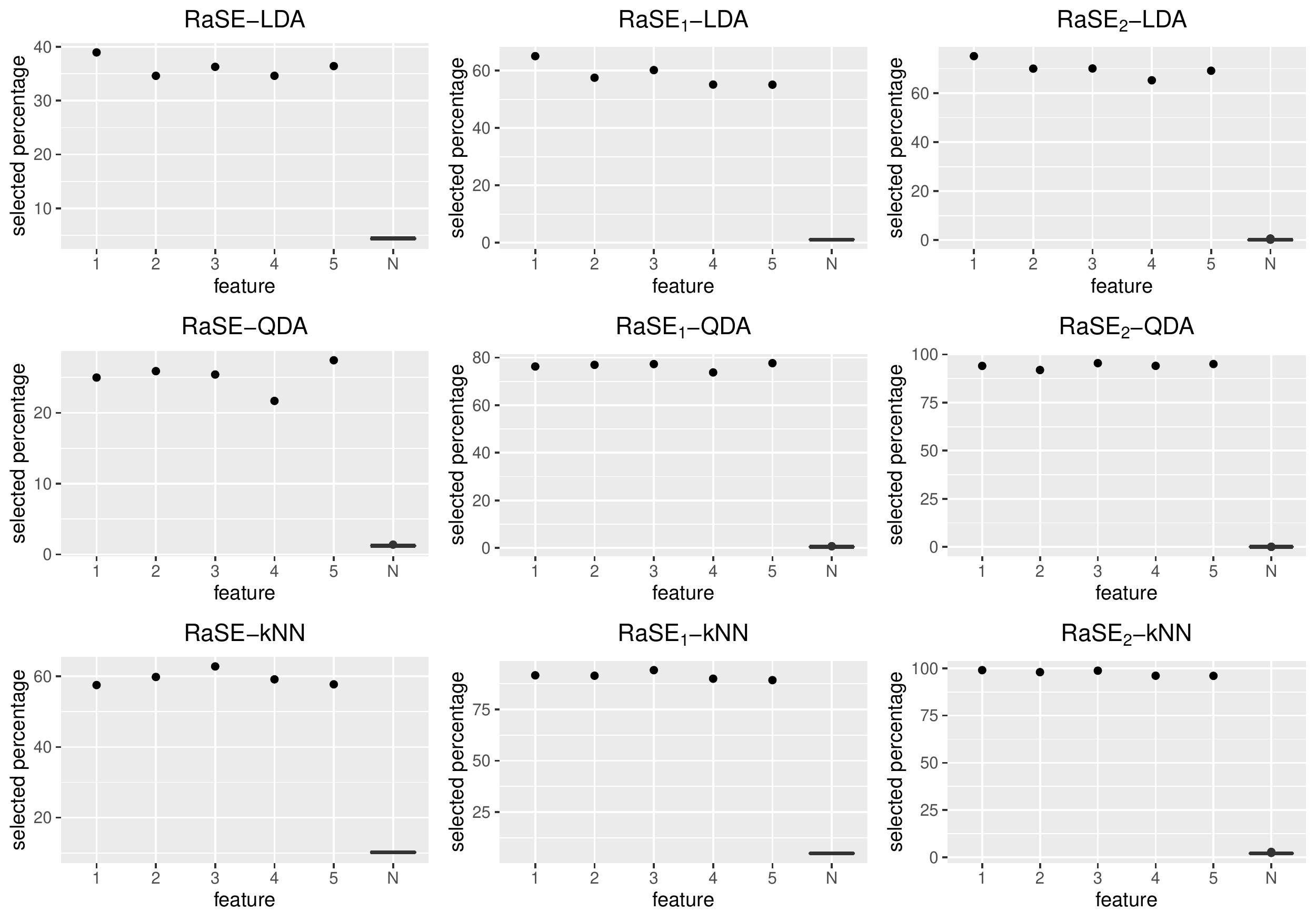}
  \caption{Average selected percentages of features for model 4 in 200 replicates when $n = 1000$}
  \label{figure: model4_sparse_large}
\end{figure}

\begin{figure}[!h]
  \includegraphics[width = \textwidth]{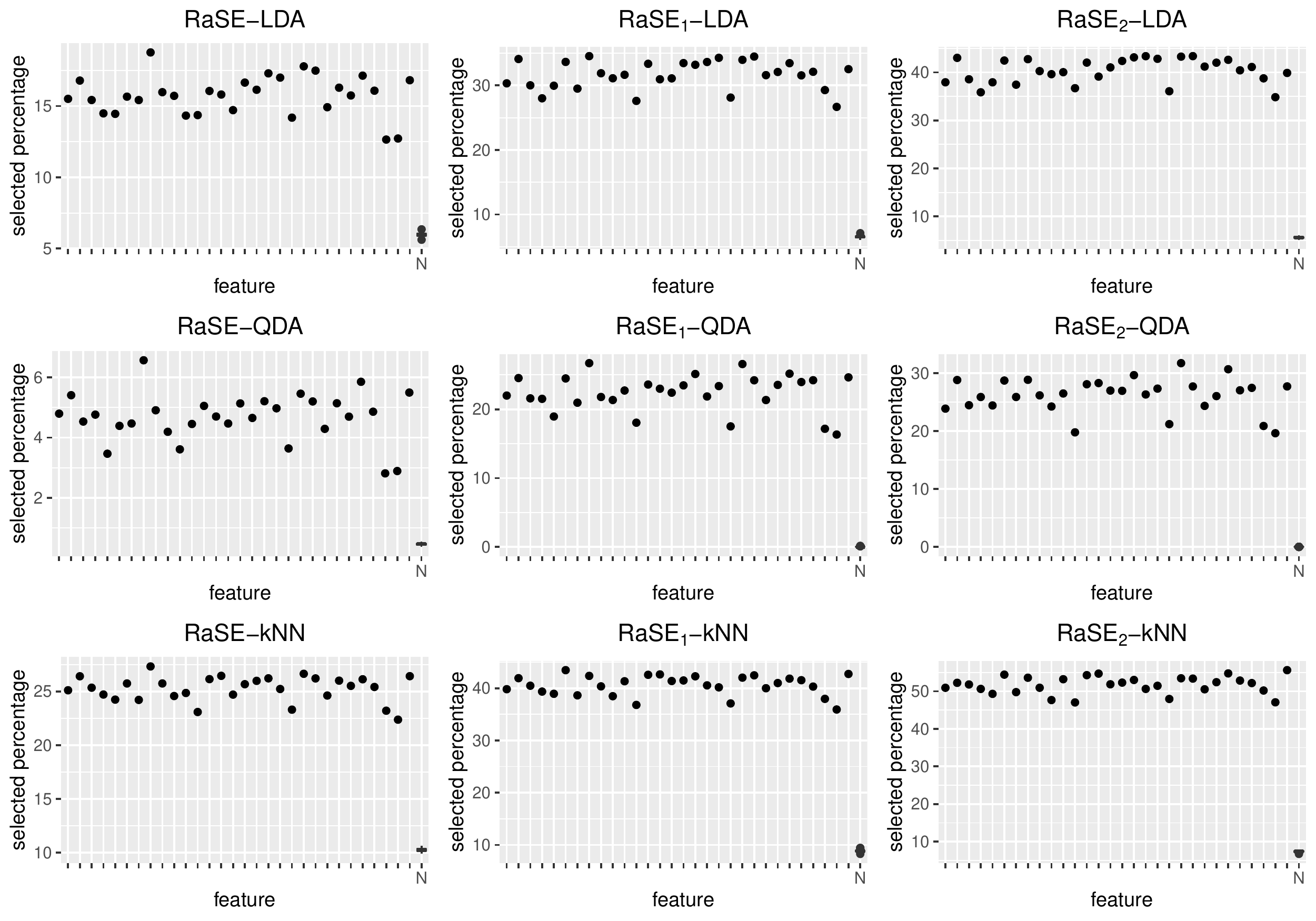}
  \caption{Average selected percentages of features for model 4 in 200 replicates when $n = 1000$}
  \label{figure: model4_nonsparse_large}
\end{figure}

\subsection{Real-data Experiments}

\subsubsection{Madelon}
The madelon data set (\url{http://archive.ics.uci.edu/ml/datasets/madelon}) is an artificial data set containing data points grouped in 32 clusters placed on the vertices of a five-dimensional hypercube and randomly labeled as 0 or 1 \citep{guyon2005result}. It consists of 2000 observations, 1000 of which are class 0, and the other 1000 are class 1. There are 500 features, among which only 20 are informative, and the others have no predictive power. The training sample size is set as $n \in \{200, 500, 1000\}$ for three different settings, and each time the remained data is used as test data. 200 replicates are applied, and the average misclassification rate with a standard deviation of all methods is reported in Table \ref{table: madelon_musk}. Figure \ref{figure: freq_madelon} represents the average selected percentage of features in different RaSE models when $n = 1000$.

From the left panel of Table \ref{table: madelon_musk}, we can see that the misclassification rate of RaSE\textsubscript{1}-$k$NN outperforms all the other methods in all the three settings. Figure \ref{figure: freq_madelon} shows us that the RaSE model leads to sparse solutions since most of the features have frequencies that are close to zero.

\begin{table}[!h]
\centering
\begin{threeparttable}
\setlength{\tabcolsep}{6pt}
\begin{tabular}{Vl|ccc|cccV}
\Xhline{1pt}
\multirow{2}{*}{Method} & \multicolumn{3}{c|}{Madelon} & \multicolumn{3}{cV}{Musk} \\ \cline{2-7}
 & \multicolumn{1}{r}{$n=200$} & $n=500$ & $n=1000$ & $n=200$ & $n=500$ & $n=1000$ \\ \hline
RaSE-LDA& \multicolumn{1}{r}{43.39\textsubscript{3.53} }& 39.87\textsubscript{1.63} & 39.16\textsubscript{1.29} & \textit{10.55}\textsubscript{1.22} & 8.86\textsubscript{0.77} & 7.87\textsubscript{0.46} \\
RaSE-QDA& \multicolumn{1}{r}{43.83\textsubscript{3.28} }& 40.58\textsubscript{1.73} & 40.01\textsubscript{1.57} & 12.55\textsubscript{7.02} & 9.16\textsubscript{0.97} & 7.77\textsubscript{0.73} \\
RaSE-$k$NN& \multicolumn{1}{r}{35.02\textsubscript{2.52} }& 26.78\textsubscript{2.71} & 21.45\textsubscript{1.91} & \textit{10.26}\textsubscript{1.81} & \textbf{7.33}\textsubscript{0.96} & \textit{5.76}\textsubscript{0.74} \\
RaSE\textsubscript{1}-LDA& \multicolumn{1}{r}{45.98\textsubscript{2.49} }& 39.76\textsubscript{1.81} & 38.57\textsubscript{1.11} & \textit{10.56}\textsubscript{1.19} & 8.88\textsubscript{0.81} & 7.82\textsubscript{0.46} \\
RaSE\textsubscript{1}-QDA& \multicolumn{1}{r}{44.46\textsubscript{5.36} }& 37.21\textsubscript{2.85} & 34.63\textsubscript{2.15} & 16.71\textsubscript{9.58} & 11.02\textsubscript{2.99} & 8.60\textsubscript{0.85} \\
RaSE\textsubscript{1}-$k$NN& \multicolumn{1}{r}{\textbf{26.05}\textsubscript{3.33} }& \textbf{16.66}\textsubscript{1.33} & \textbf{13.57}\textsubscript{1.00} & \textit{10.52}\textsubscript{1.95} & \textit{7.49}\textsubscript{0.97} & \textbf{5.71}\textsubscript{0.78} \\
RP-LDA& \multicolumn{1}{r}{41.18\textsubscript{1.77} }& 39.86\textsubscript{1.14} & 39.41\textsubscript{1.17} & 12.58\textsubscript{1.86} & 10.19\textsubscript{0.76} & 9.50\textsubscript{0.46} \\
RP-QDA& \multicolumn{1}{r}{39.97\textsubscript{1.53} }& 39.29\textsubscript{1.57} & 38.79\textsubscript{1.49} & \textit{10.70}\textsubscript{1.95} & 8.75\textsubscript{0.83} & 8.18\textsubscript{0.45} \\
RP-$k$NN& \multicolumn{1}{r}{40.10\textsubscript{1.66} }& 39.07\textsubscript{1.54} & 38.40\textsubscript{1.43} & \textbf{10.01}\textsubscript{1.66} & \textit{8.05}\textsubscript{1.06} & 6.89\textsubscript{0.84} \\
LDA& ---$\dagger$ & 49.82\textsubscript{1.24} & 47.60\textsubscript{1.49} & 25.60\textsubscript{3.76} & 9.06\textsubscript{0.80} & 6.97\textsubscript{0.40} \\
QDA& ---$\dagger$ & ---$\dagger$ & ---$\dagger$ & ---$\dagger$ & ---$\dagger$ & ---$\dagger$ \\
$k$NN& \multicolumn{1}{r}{35.89\textsubscript{1.42} }& 31.84\textsubscript{1.43} & 28.56\textsubscript{1.56} & \textit{11.35}\textsubscript{1.55} & \textit{8.19}\textsubscript{0.86} & 6.53\textsubscript{0.55} \\
sLDA& \multicolumn{1}{r}{43.01\textsubscript{2.98} }& 40.40\textsubscript{1.94} & 39.57\textsubscript{1.52} & \textit{10.68}\textsubscript{1.54} & \textit{7.81}\textsubscript{0.68} & 6.85\textsubscript{0.40} \\
RAMP& \multicolumn{1}{r}{49.09\textsubscript{3.46} }& 41.88\textsubscript{5.16} & 38.58\textsubscript{1.36} & 14.35\textsubscript{1.93} & 11.50\textsubscript{1.40} & 9.72\textsubscript{1.10} \\
NSC& \multicolumn{1}{r}{41.49\textsubscript{2.32} }& 40.10\textsubscript{1.09} & 40.05\textsubscript{1.24} & 20.72\textsubscript{5.11} & 24.12\textsubscript{3.62} & 25.58\textsubscript{1.01} \\
RF& \multicolumn{1}{r}{43.91\textsubscript{2.49} }& 38.59\textsubscript{1.56} & 34.17\textsubscript{1.46} & \textit{10.83}\textsubscript{1.44} & \textit{7.60}\textsubscript{0.66} & \textbf{5.71}\textsubscript{0.48} \\ \Xhline{1pt}
\end{tabular}
\begin{tablenotes}
\item $\dagger$ Not applicable.
\end{tablenotes}
\end{threeparttable}
\caption{Error rates for madelon and musk data sets}
\label{table: madelon_musk}
\end{table}

\begin{figure}[!h]
  \includegraphics[width = \textwidth]{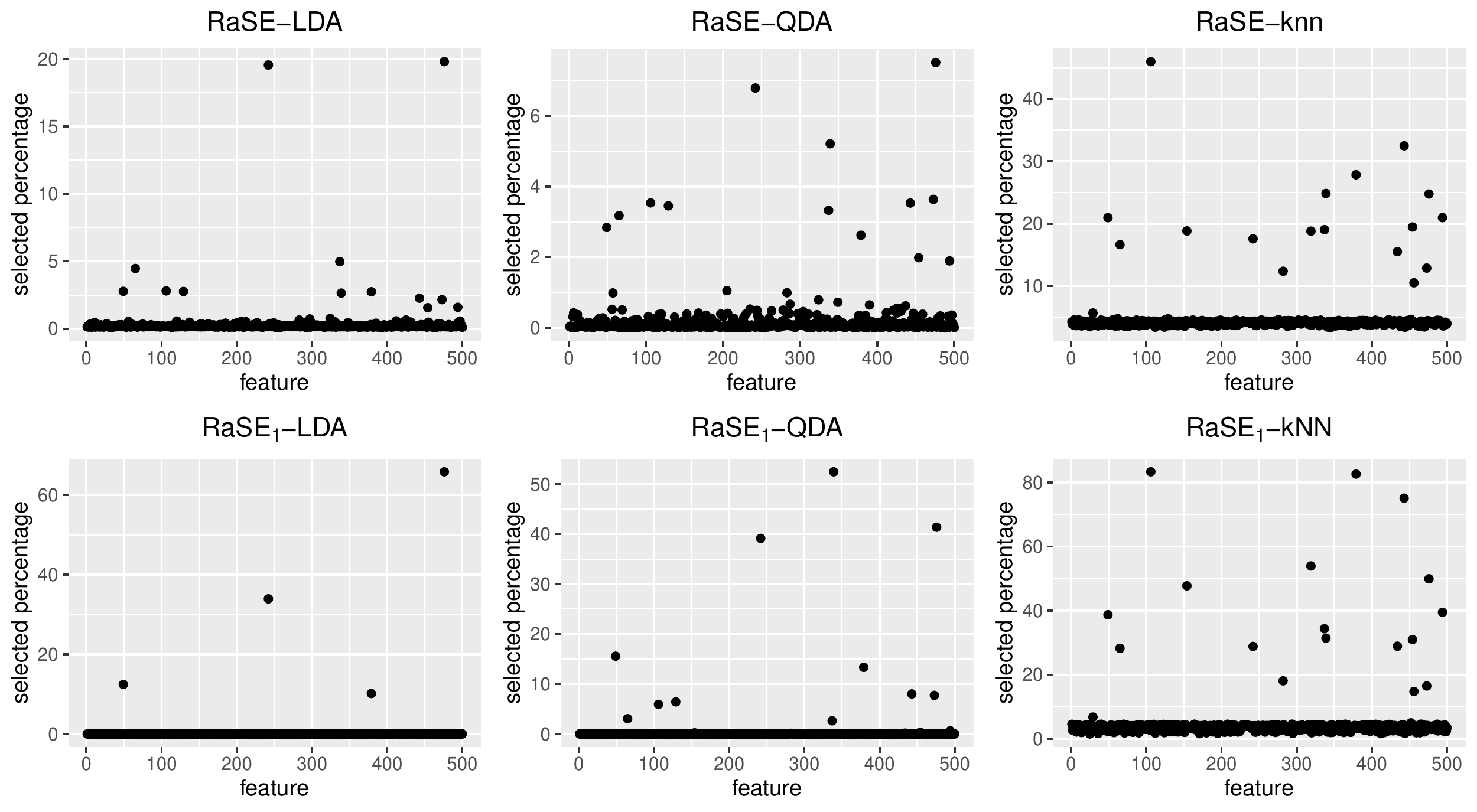}
  \caption{Average selected percentages of features for madelon data set in 200 replicates when $n = 1000$}
  \label{figure: freq_madelon}
\end{figure}

\subsubsection{Musk}

The musk data set (\url{https://archive.ics.uci.edu/ml/datasets/Musk+(Version+2)}) contains 6598 observations with 5581 non-musk (class 0) and 1017 musk (class 1) molecules. The molecule needs to be classified based on $p = 166$ shape measurements \citep{dua2019uci}. The training sample size is set to be 200, 500, 1000, for each setting, and the remaining observations are used as the test data. 200 replicates are considered, and the average misclassification rates and standard deviations are reported in Table \ref{table: madelon_musk}.

When the training sample size is 200, RP-$k$NN achieves the lowest misclassification rate, and RaSE-LDA, RaSE-$k$NN, RaSE\textsubscript{1}-LDA, RaSE\textsubscript{1}-$k$NN, sLDA, RP-QDA, and RF also have a good performance. As the sample size increases, RaSE-$k$NN turns to be the best one when $n = 500$ and RF yields a comparable performance. When $n = 1000$, RaSE\textsubscript{1}-$k$NN and RF outperform the other methods.

\subsubsection{Mice Protein Expression}

The mice protein expression data set (\url{https://archive.ics.uci.edu/ml/datasets/Mice+Protein+Expression}) contains 1080 instances with 570 healthy mice (class 0) and 510 mice with Down's syndrome (class 1). There are 77 features representing the expression of 77 different proteins \citep{higuera2015self}. Training samples of size 200, 500, 800 are considered, and the remaining observations are set as the test data.

The average of test misclassification rates and the standard deviations of 200 replicates are calculated with results reported in Table \ref{table: mice_num}. When $n = 200$, sLDA achieves the lowest error among all approaches, and RaSE-$k$NN achieves a similar performance. As the sample size increases to 500 and 800, the average misclassification rate of RaSE-$k$NN and RaSE\textsubscript{1}-$k$NN decrease significantly, and they become the best classifier when $n = 500$ and $800$, respectively. When $n = 800$, RP-$k$NN and RF also have a similar performance.

\begin{table}[!h]
\centering
\begin{threeparttable}
\setlength{\tabcolsep}{7pt}
\begin{tabular}{Vl|ccc|cccV}
\Xhline{1pt}
\multirow{2}{*}{Method} & \multicolumn{3}{c|}{Mice protein expression} & \multicolumn{3}{cV}{Hand-written digits recognition} \\ \cline{2-7}
 & \multicolumn{1}{r}{$n=200$} & $n=500$ & $n=800$ & $n=50$ & $n=100$ & $n=200$ \\ \hline
RaSE-LDA& \multicolumn{1}{r}{7.41\textsubscript{1.14} }& 5.70\textsubscript{0.93} & 4.65\textsubscript{1.24} & \textit{1.56}\textsubscript{0.85} & 1.13\textsubscript{0.59} & \textit{0.80}\textsubscript{0.54} \\
RaSE-QDA& \multicolumn{1}{r}{9.14\textsubscript{2.58} }& 4.81\textsubscript{1.17} & 3.44\textsubscript{1.23} & 2.50\textsubscript{1.47} & 1.89\textsubscript{0.91} & 1.47\textsubscript{0.96} \\
RaSE-$k$NN& \multicolumn{1}{r}{6.80\textsubscript{1.88} }& \textbf{1.55}\textsubscript{0.88} & \textit{0.62}\textsubscript{0.55} & 1.86\textsubscript{0.96} & 1.12\textsubscript{0.66} & \textit{0.75}\textsubscript{0.45} \\
RaSE\textsubscript{1}-LDA& \multicolumn{1}{r}{7.24\textsubscript{1.10} }& 5.53\textsubscript{1.02} & 4.49\textsubscript{1.23} & \textbf{1.06}\textsubscript{0.63} & \textit{0.70}\textsubscript{0.35} & \textbf{0.53}\textsubscript{0.40} \\
RaSE\textsubscript{1}-QDA& \multicolumn{1}{r}{9.38\textsubscript{2.21} }& 5.16\textsubscript{1.16} & 3.40\textsubscript{1.16} & 2.18\textsubscript{1.66} & 1.18\textsubscript{0.71} & \textit{0.85}\textsubscript{0.61} \\
RaSE\textsubscript{1}-$k$NN& \multicolumn{1}{r}{7.43\textsubscript{2.00} }& \textit{1.70}\textsubscript{0.87} & \textbf{0.60}\textsubscript{0.56} & 1.72\textsubscript{0.95} & \textit{1.02}\textsubscript{0.62} & \textit{0.60}\textsubscript{0.44} \\
RP-LDA& \multicolumn{1}{r}{24.84\textsubscript{2.91} }& 22.79\textsubscript{2.50} & 22.34\textsubscript{2.55} & 1.75\textsubscript{1.29} & 1.22\textsubscript{0.67} & 1.04\textsubscript{0.61} \\
RP-QDA& \multicolumn{1}{r}{18.31\textsubscript{2.57} }& 16.19\textsubscript{2.08} & 15.66\textsubscript{2.38} & 2.12\textsubscript{1.67} & 1.28\textsubscript{0.94} & \textit{0.92}\textsubscript{0.62} \\
RP-$k$NN& \multicolumn{1}{r}{11.77\textsubscript{2.54} }& 2.57\textsubscript{0.89} & \textit{0.92}\textsubscript{0.68} & \textit{1.68}\textsubscript{1.34} & \textit{1.03}\textsubscript{0.60} & \textit{0.84}\textsubscript{0.59} \\
LDA& \multicolumn{1}{r}{7.07\textsubscript{1.37} }& 3.88\textsubscript{0.85} & 3.13\textsubscript{1.08} & ---$\dagger$ & 1.82\textsubscript{0.96} & 1.01\textsubscript{0.56} \\
QDA& ---$\dagger$ & ---$\dagger$ & ---$\dagger$ & ---$\dagger$ & ---$\dagger$ & 3.25\textsubscript{2.32} \\
$k$NN& \multicolumn{1}{r}{20.53\textsubscript{2.47} }& 7.75\textsubscript{1.44} & 2.80\textsubscript{1.21} & \textit{1.42}\textsubscript{1.32} & \textbf{0.67}\textsubscript{0.41} & \textit{0.60}\textsubscript{0.47} \\
sLDA& \multicolumn{1}{r}{\textbf{5.70}\textsubscript{1.10} }& 3.95\textsubscript{0.91} & 3.13\textsubscript{1.05} & 2.30\textsubscript{1.36} & 1.71\textsubscript{1.27} & 1.15\textsubscript{0.95} \\
RAMP& \multicolumn{1}{r}{11.76\textsubscript{2.42} }& 8.52\textsubscript{1.69} & 7.02\textsubscript{1.89} & 3.31\textsubscript{1.75} & 2.26\textsubscript{1.19} & 1.70\textsubscript{0.87} \\
NSC& \multicolumn{1}{r}{30.49\textsubscript{3.31} }& 29.70\textsubscript{2.76} & 29.88\textsubscript{3.02} & 3.22\textsubscript{1.44} & 3.53\textsubscript{1.23} & 3.59\textsubscript{1.50} \\
RF& \multicolumn{1}{r}{8.32\textsubscript{1.71} }& 2.62\textsubscript{0.94} & \textit{1.04}\textsubscript{0.73} & 2.34\textsubscript{1.24} & 1.63\textsubscript{0.73} & 1.37\textsubscript{0.74} \\ \Xhline{1pt}
\end{tabular}
\begin{tablenotes}
\item $\dagger$ Not applicable.
\end{tablenotes}
\end{threeparttable}
\caption{Error rates for mice protein expression and hand-written digits recognition data sets}
\label{table: mice_num}
\end{table}

\subsubsection{Hand-written Digits Recognition}

The hand-written digits recognition data set (\url{https://archive.ics.uci.edu/ml/datasets/Multiple+Features}) consists of features of hand-written numerals (0-9) extracted from a collection of Dutch utility maps \citep{dua2019uci}. Here we use the mfeat-fou data set, which records 76 Fourier coefficients of the character shapes. We extract the observations corresponding to number 7 (class 0) and 9 (class 1) from the original data. There are 400 observations, 200 of which belong to class 0, and the remaining 200 belong to class 1. The training samples of size 50, 100, 200 are used, and the remained data is used as the test data. 

Average of test misclassification rates and standard deviations are reported in Table \ref{table: mice_num}, from which it can be seen that when $n = 50, 200$, RaSE\textsubscript{1}-LDA enjoys the minimal test misclassification rate while standard $k$NN method is the best when $n = 100$. And we also note that all the RaSE classifiers get improved after 1 iteration for all three settings, implying that the underlying classification problem may be a sparse one.

\section{Discussion\label{sec::discussion}}

\subsection{Summary}

In this work, we introduce a flexible ensemble classification framework named RaSE, which is designed to solve the sparse classification problem. To select the optimal subspace for each weak learner, we define a new information criterion, ratio information criterion (RIC), based on Kullback-Leibler divergence, and it is shown to achieve screening consistency and weak consistency under some general model conditions. This guarantees that each weak learner is trained using features in the minimal discriminative set with a high probability for a sufficiently large sample size and the number of random subspaces. We also investigate the consistency of RIC for LDA and QDA models under some specific conditions. The theoretical analysis of RaSE classifiers with specific base classifiers is conducted. In addition, we present two versions of RaSE algorithms, that is, the vanilla version (RaSE) and the iterative version (RaSE\textsubscript{$T$}). We also apply RaSE for feature ranking based on the average selected percentage of features in subspaces. Theoretical analysis shows that when the stepwise detectable condition holds and the signal is sufficiently sparse, the iterative RaSE can cover the minimal discriminative set with high probability after a few iterations, with the required number of random subspaces smaller than that for the vanilla RaSE.

Multiple numerical experiments, including simulations and real data, verify that RaSE is a favorable classification method for sparse classification problems.

The RaSE algorithms are available in R package \texttt{RaSEn} (\url{https://cran.r-project.org/web/packages/RaSEn/index.html}).

\subsection{Future Work}

There are many interesting directions along which RaSE can be extended and explored. An interesting question is how to extend RaSE and RIC into multi-class problems. For example, the pair-wise KL divergences can be used to define the multi-class RIC. Moreover, we can also apply RaSE for variable selection. We can conduct thresholding to the average selected percentage of features in $B_1$ subspaces to select variables. When the sample size is small, a bootstrap-type idea can be used, and each time we apply RaSE on a bootstrap sample and at the end, take the average for the selected percentage to do variable selection or feature ranking. Finally, we aggregate the classifiers by taking a simple average over all weak learners. However, the boosting-type idea can also be applied here to assign different weights for different weak learners according to the training error, which may further improve the performance of RaSE. In addition, the distribution for random subspaces can also be chosen to be different for each weak learner and can also be updated using a similar idea to boosting. 
\acks{The authors would like to thank the Action Editor and anonymous referees for many constructive comments which have greatly improved the paper. This work was partially supported by National Science Foundation CAREER grant DMS-2013789.}


\newpage

\appendix

\section{Additional Figures of Simulations}\label{appendix b: additional fig}

We present figures of the selected percentage for each feature when $n$ equals to the smallest value among three settings.

\begin{figure}[!htb]
  \includegraphics[width = \textwidth]{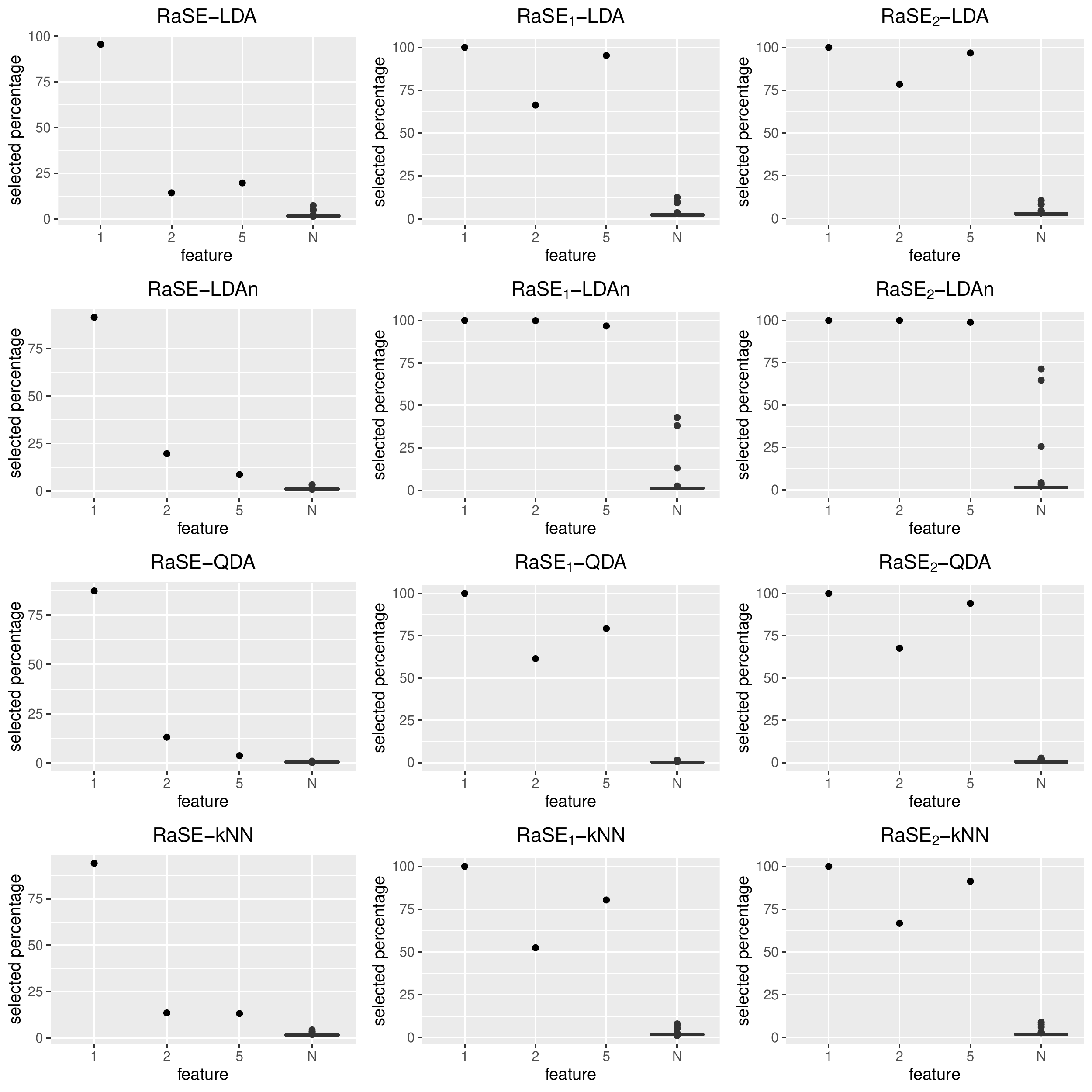}
  \caption{Average selected percentages of features for model 1 in 200 replicates when $n = 200$}
  \label{figure: model1_sparse_small}
\end{figure}

\FloatBarrier
\begin{figure}[!htb]
  \includegraphics[width = \textwidth]{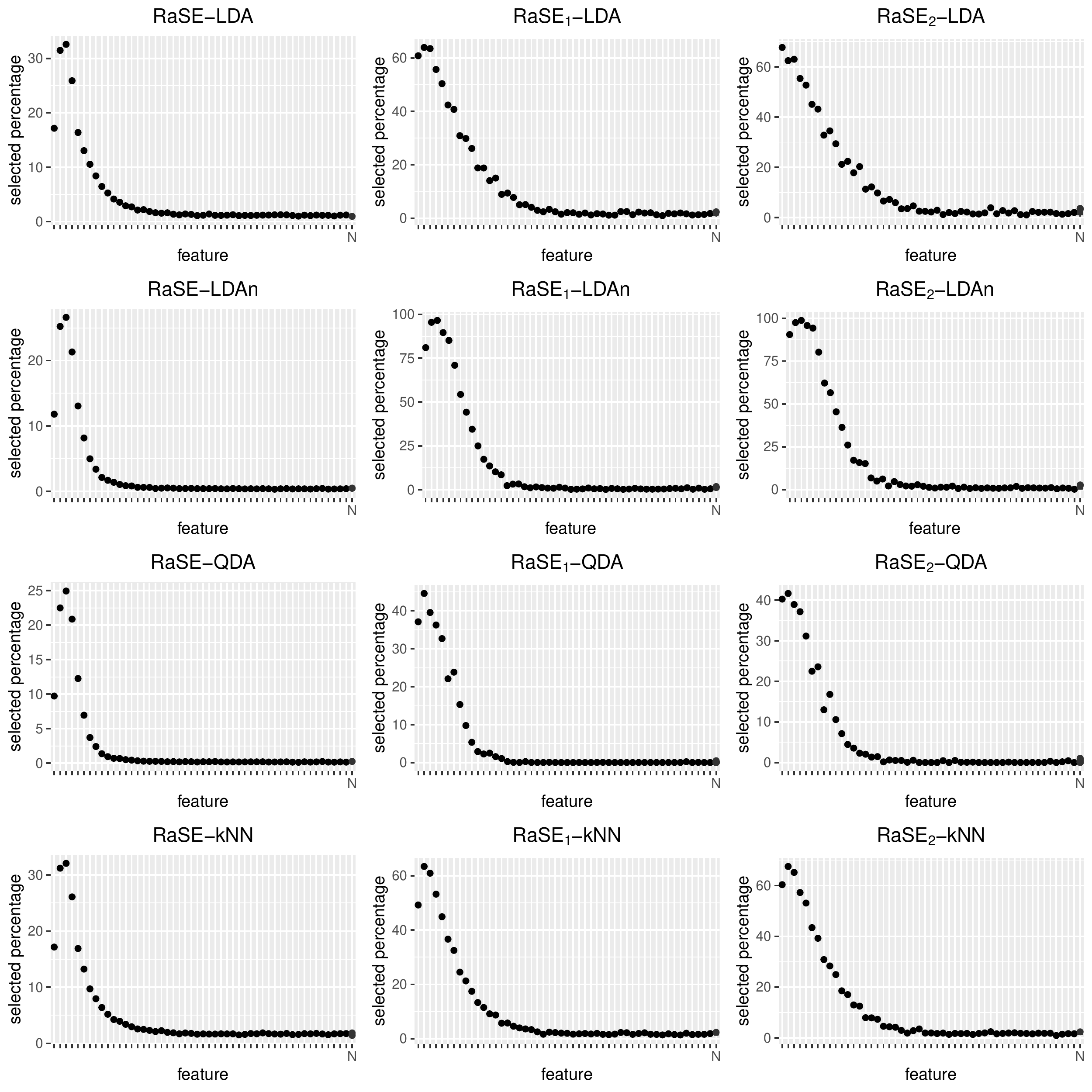}
  \caption{Average selected percentages of features for model 1' in 200 replicates when $n = 200$}
  \label{figure: model1_nonsparse_small}
\end{figure}

\FloatBarrier
\begin{figure}[!htb]
  \includegraphics[width = \textwidth]{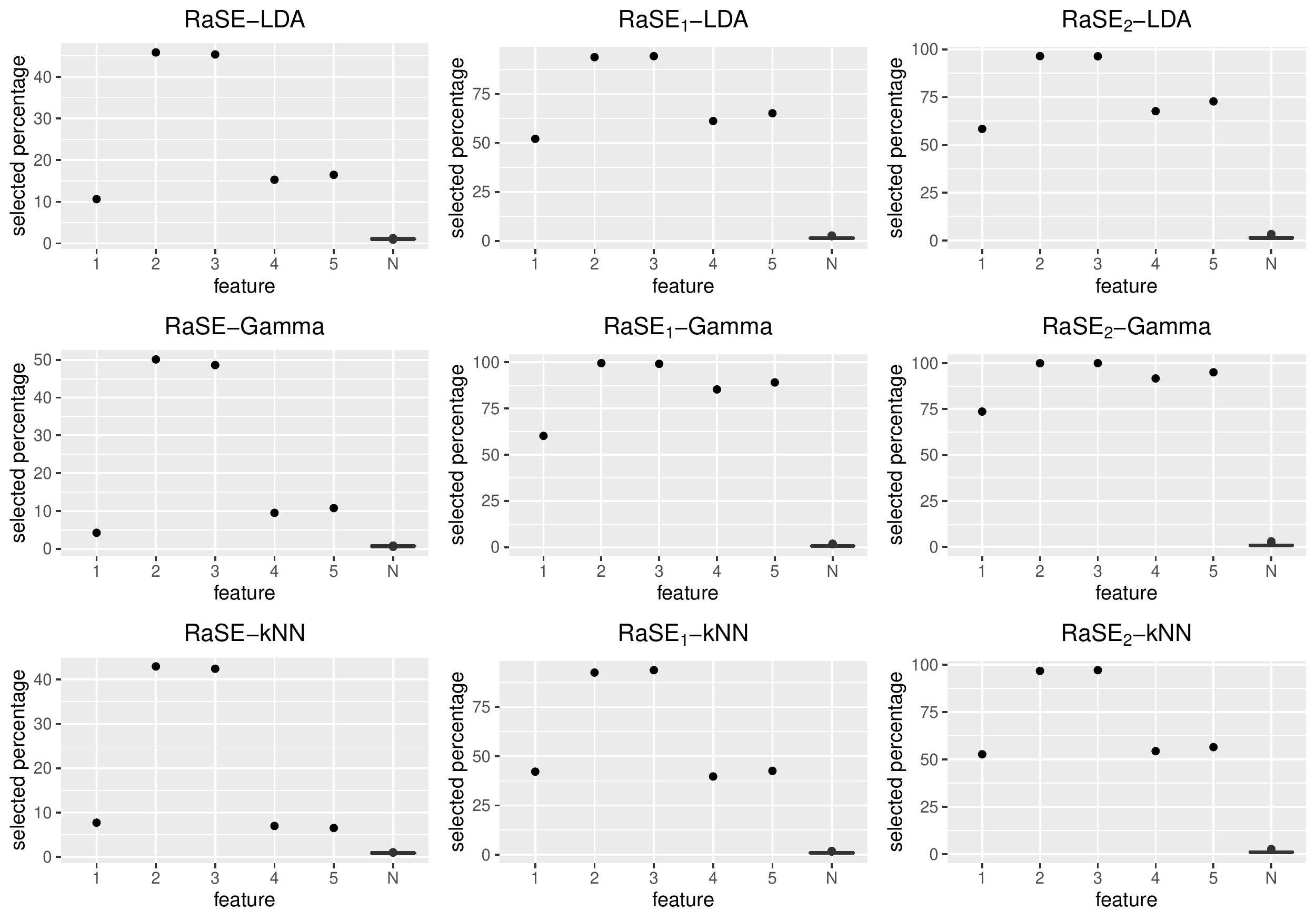}
  \caption{Average selected percentages of features for model 2 in 200 replicates when $n = 100$}
  \label{figure: model2_small}
\end{figure}

\begin{figure}[!htb]
  \includegraphics[width = \textwidth]{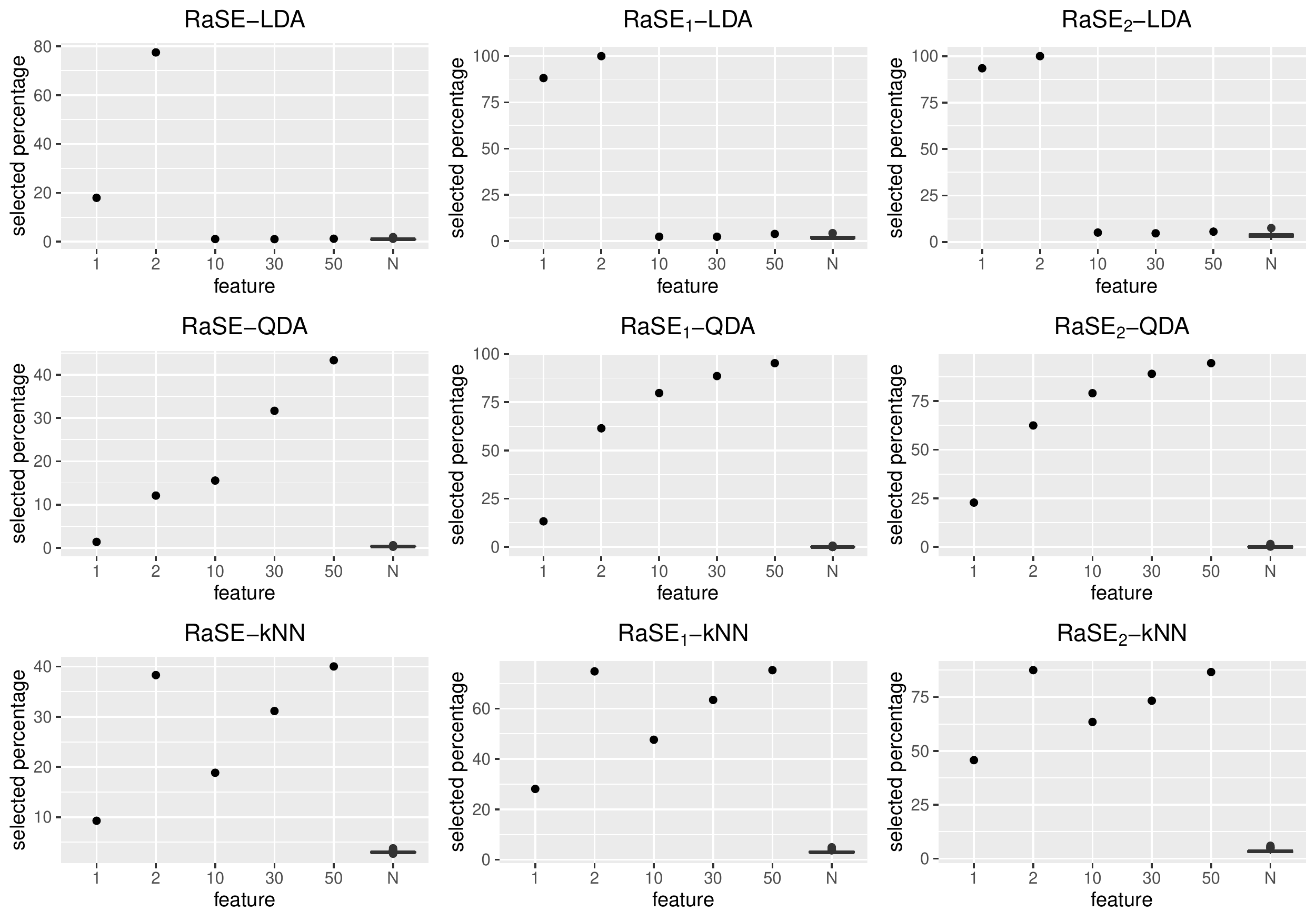}
  \caption{Average selected percentages of features for model 3 in 200 replicates when $n = 200$}
  \label{figure: model3_small}
\end{figure}
\FloatBarrier

\begin{figure}[!]
  \includegraphics[width = \textwidth]{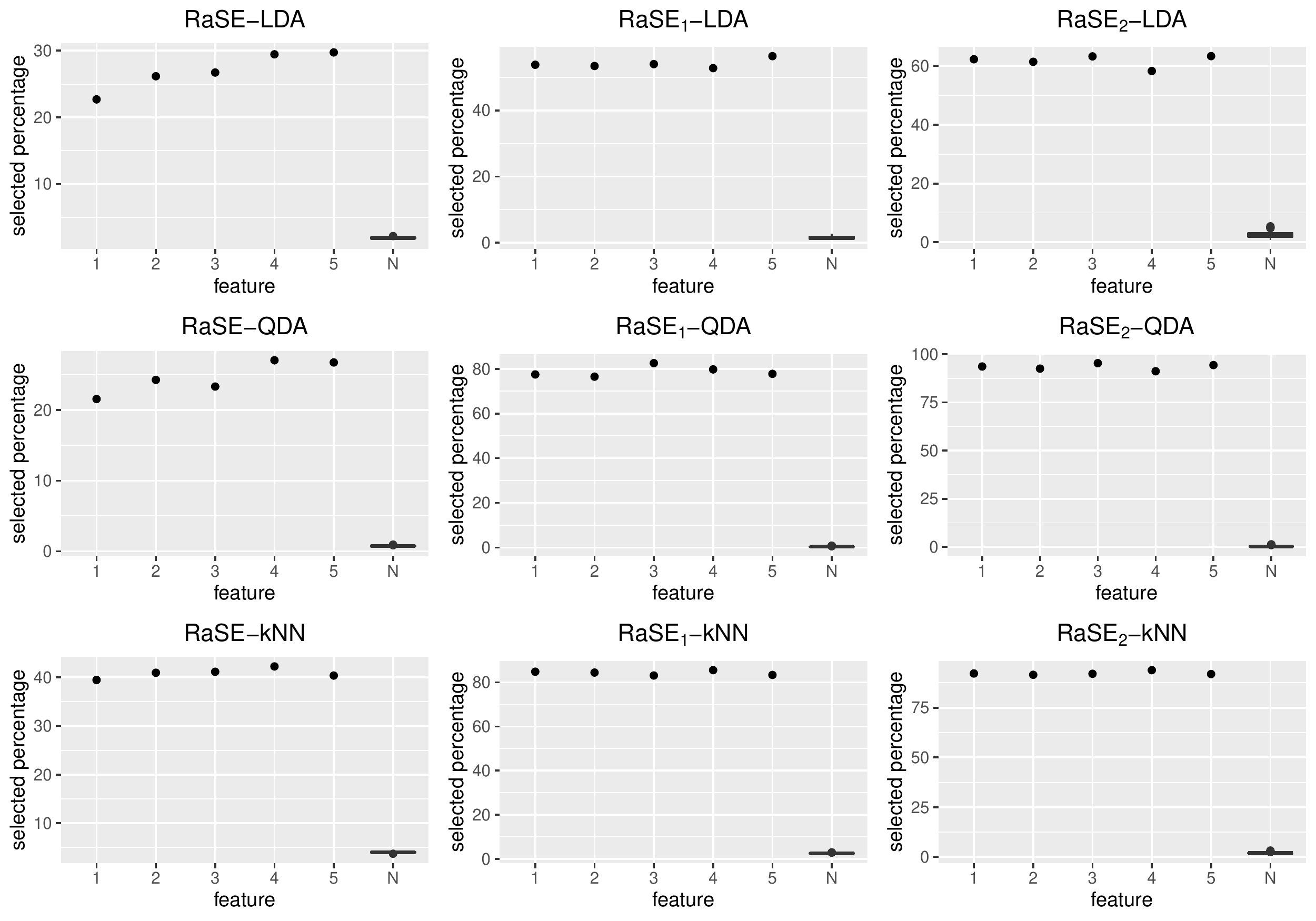}
  \caption{Average selected percentages of features for model 4 in 200 replicates when $n = 200$}
  \label{figure: model4_sparse_small}
\end{figure}
\FloatBarrier

\begin{figure}[!htb]
  \includegraphics[width = \textwidth]{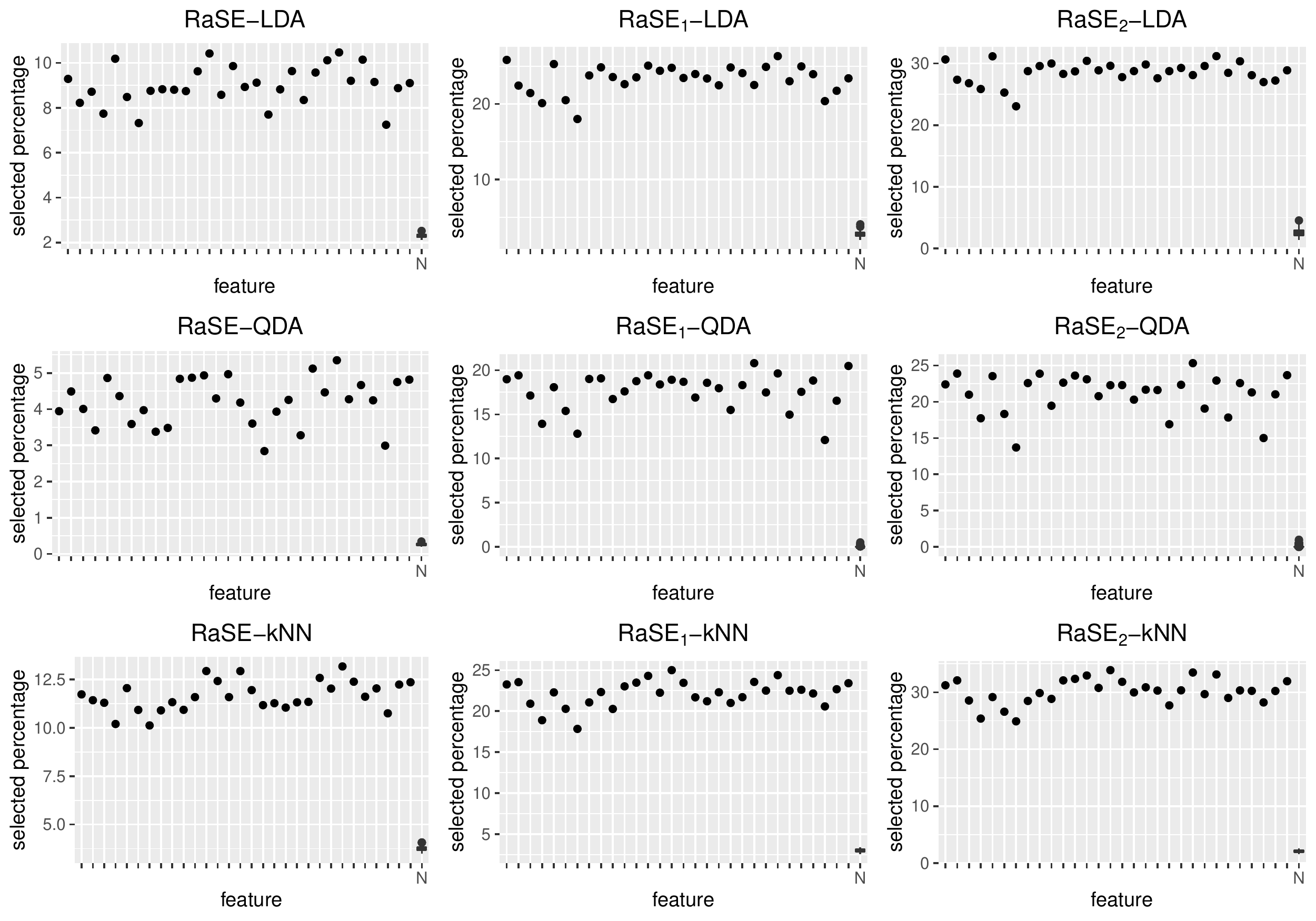}
  \caption{Average selected percentages of features for model 4' in 200 replicates when $n = 200$}
  \label{figure: model4_nonsparse_small}
\end{figure}
\FloatBarrier

\section{Main Proofs}\label{appendix a: proof}
In the following proofs, for convenience, we use $C$ to represent a positive constant, which could be different at different occurences.
\subsection{Proof of Proposition \ref{prop: discriminative set}}

(\rom{1}) The conditional probability is
	\begin{equation}\label{proof: discriminative regression func}
		\p(y = 1|\bm{x}) = \frac{\pi_1 f^{(1)}(\bm{x})}{\pi_1 f^{(1)}(\bm{x}) + \pi_0 f^{(0)}(\bm{x})}.
	\end{equation}
	
	The definition of discriminative set is equivalent to
	\begin{equation}\label{eq: prob}
		\p(y = 1|\bm{x}) = \p(y = 1|\bm{x}_{S}),
	\end{equation} or
\begin{equation}
\p(y = 0|\bx) = \p(y = 0|\bx_S)
\end{equation}
	almost surely, which is also equivalent to
	\begin{equation}\label{proof: discriminative}
		\frac{f^{(1)}(\bm{x})}{f^{(0)}(\bm{x})} = \frac{f^{(1)}_S(\bm{x}_S)}{f^{(0)}_S(\bm{x}_S)}.
	\end{equation}

(\rom{2}) First, we observe the condition is equivalent to $f^{(1)}(\bm{x}) = h(\bm{x}_S)f^{(0)}(\bm{x})$. Taking integration on both sides with respect to $\bm{x}_{S^c}$, we get $f^{(1)}_S(\bm{x}_S)=h(\bm{x}_S) f^{(0)}_S(\bm{x}_S)$. Due to the equivalence of \eqref{eq: prob} and \eqref{proof: discriminative}, $S$ is a discriminative set.

\subsection{Proof of Proposition \ref{prop: uni discriminative set for gaussian}}
To facilitate our analysis, we first state the following lemma.
\begin{lemma}\label{lem: gaussian uni}
	A discriminative set $S$ is unique, if it will not be a discriminative set anymore after removing any features from it.
\end{lemma}
\begin{proof}[Proof of Lemma \ref{lem: gaussian uni}] 
Suppose the conclusion is not correct, then there exist two different discriminative sets $S_1$ and $S_2$ satisfying such a property, that is, deleting any features from them leads to non-discriminative sets. Then according to Proposition \ref{prop: discriminative set} (\rom{1}), we have
\begin{equation}\label{eq: gaussian}
	\frac{f^{(1)}_{S_1}(\bm{x}_{S_1})}{f^{(0)}_{S_1}(\bm{x}_{S_1})} = \frac{f^{(1)}_{S_2}(\bm{x}_{S_2})}{f^{(0)}_{S_2}(\bm{x}_{S_2})},
\end{equation}
almost surely w.r.t $\p^{\bx}$, where $\p^{\bx} = \pi_0 \p^{(0)} + \pi_1 \p^{(1)}$. Since $f^{(0)}, f^{(1)}$ are supported on the whole $\mathbb{R}^p$, $\p^{\bx}$ dominates the Lebesgue measure. Combined with the explicit form of density functions of Gaussian distribution and denoting $\Omega_{S_1, S_1} = (\Sigma^{(1)}_{S_1, S_1})^{-1} - (\Sigma^{(0)}_{S_1, S_1})^{-1}, \Omega_{S_2, S_2} = (\Sigma^{(1)}_{S_2, S_2})^{-1} - (\Sigma^{(0)}_{S_2, S_2})^{-1}, \bdelta_{S_1} = (\Sigma^{(0)}_{S_1, S_1})^{-1}\bm{\mu}_{S_1}^{(0)} - (\Sigma^{(1)}_{S_1, S_1})^{-1}\bm{\mu}_{S_1}^{(1)}, \bdelta_{S_2} = (\Sigma^{(0)}_{S_2, S_2})^{-1}\bm{\mu}_{S_2}^{(0)} - (\Sigma^{(1)}_{S_2, S_2})^{-1}\bm{\mu}_{S_2}^{(1)}$, it can be obtained from \eqref{eq: qda} and \eqref{eq: gaussian} that
\begin{equation*}
	c + \frac{1}{2}\bm{x}_{S_1}^T\Omega_{S_1, S_1}\bm{x}_{S_1} + \bdelta_{S_1}^T\bm{x}_{S_1} = c' + \frac{1}{2}\bm{x}_{S_2}^T\Omega_{S_2, S_2}\bm{x}_{S_2} + \bdelta_{S_2}^T\bm{x}_{S_2},
\end{equation*}
where $c, c'$ are constants irrelative to $S_1, S_2$. Considering the combined vector $\bm{x}_{S_1 \cup S_2}$, there exists some matrix $\Omega$ and vector $\bm{\delta}$ such that the equation above can be simplified as 
\begin{equation}\label{eq: x}
	\bm{x}_{S_1 \cup S_2}^T\Omega\bm{x}_{S_1 \cup S_2} + \bm{\delta}^T\bm{x}_{S_1 \cup S_2} + c- c' = 0,
\end{equation}
for almost every $\bm{x}_{S_1 \cup S_2}$ in the Euclidean space. Since $S_1$ is a discriminative set, by Example \ref{exp: qda}, for every feature $j \in S_1$, the corresponding row of $\Omega_{S_1, S_1}$ is not zero vector or the corresponding component of $\bdelta_{S_1}$ is not zero. The same argument holds for $S_2$ as well. Since $S_1 \backslash S_2$ and $S_2 \backslash S_1$ are not empty, at least one of $\Omega$ and $\bm{\delta}$ contain non-zero components. However, it's obvious that \eqref{eq: x} cannot hold for almost every $\bm{x}_{S_1 \cup S_2}$ in Euclidean space, which leads to  contradiction. Thus $S^*$ is unique. 
\end{proof}

Now let's proceed on the proof of Proposition \ref{prop: uni discriminative set for gaussian}. 

By Definition \ref{def: discriminative set}, the minimal discriminative set $S^*$ satisfies the property described in Lemma \ref{lem: gaussian uni}, therefore by this lemma $S^*$ is obviously unique, which implies (\rom{1}). 

For (\rom{2}), if there exists a discriminative set $S \not\supseteq S^*$, we can remove elements from $S$ until we arrive at a discriminative set $S'$ that satisfies the property stated in Lemma \ref{lem: gaussian uni}. It's apparent $S' \neq S^*$, which contradicts with Lemma \ref{lem: gaussian uni}.

(\rom{3}) is trivial from Definition \ref{def: discriminative set}.

\subsection{Proof of Proposition \ref{prop: RIC for LDA case}}
By Definition \ref{def: ric} and the fact that $\deg(S) = |S| + 1$, it's easy to obtain that
\begin{align}
	\textrm{RIC}_n(S) &= \frac{2}{n}\sum_{i=1}^n \mathds{1}(y_i = 0)\left[((\hat{\bmu}_{S}^{(1)})^T \hat{\Sigma}_{S, S}^{-1} - (\hat{\bmu}_{S}^{(0)})^T \hat{\Sigma}_{S, S}^{-1})\bm{x}_{i, S} + \frac{1}{2} (\hat{\bmu}_{S}^{(0)})^T \hat{\Sigma}_{S,S}^{-1} \hat{\bmu}_{S}^{(0)}\right. \nonumber\\
	&\quad \left. - \frac{1}{2} (\hat{\bmu}_{S}^{(1)})^T \hat{\Sigma}_{S,S}^{-1} \hat{\bmu}_{S}^{(1)}\right] + \frac{2}{n}\sum_{i=1}^n \mathds{1}(y_i = 1)\bigg[((\hat{\bmu}_{S}^{(0)})^T \hat{\Sigma}_{S, S}^{-1} - (\hat{\bmu}_{S}^{(1)})^T \hat{\Sigma}_{S, S}^{-1})\bm{x}_{i, S} \nonumber\\
	&\quad + \frac{1}{2} (\hat{\bmu}_{S}^{(1)})^T \hat{\Sigma}_{S,S}^{-1} \hat{\bmu}_{S}^{(1)} - \frac{1}{2} (\hat{\bmu}_{S}^{(0)})^T \hat{\Sigma}_{S,S}^{-1} \hat{\bmu}_{S}^{(0)}\bigg] + c_n \cdot (|S| + 1) \nonumber\\
	&= 2\hat{\pi}_0 \left[((\hat{\bmu}_{S}^{(1)})^T \hat{\Sigma}_{S, S}^{-1} - (\hat{\bmu}_{S}^{(0)})^T \hat{\Sigma}_{S, S}^{-1})\hat{\bmu}_{S}^{(0)} + \frac{1}{2} (\hat{\bmu}_{S}^{(0)})^T \hat{\Sigma}_{S,S}^{-1} \hat{\bmu}_{S}^{(0)} - \frac{1}{2} (\hat{\bmu}_{S}^{(1)})^T \hat{\Sigma}_{S,S}^{-1} \hat{\bmu}_{S}^{(1)}\right] \nonumber\\
	&\quad + 2\hat{\pi}_1 \left[((\hat{\bmu}_{S}^{(0)})^T \hat{\Sigma}_{S, S}^{-1} - (\hat{\bmu}_{S}^{(1)})^T \hat{\Sigma}_{S, S}^{-1})\hat{\bmu}_{S}^{(1)} + \frac{1}{2} (\hat{\bmu}_{S}^{(1)})^T \hat{\Sigma}_{S,S}^{-1} \hat{\bmu}_{S}^{(1)} - \frac{1}{2} (\hat{\bmu}_{S}^{(0)})^T \hat{\Sigma}_{S,S}^{-1} \hat{\bmu}_{S}^{(0)}\right] \nonumber\\
	&\quad + c_n \cdot (|S| + 1) \nonumber\\
	&= -(\hat{\bmu}_{S}^{(1)} - \hat{\bmu}_{S}^{(0)})^T \hat{\Sigma}_{S, S}^{-1}(\hat{\bmu}_{S}^{(1)} - \hat{\bmu}_{S}^{(0)}) + c_n \cdot (|S| + 1),
\end{align}
which completes the proof.

\subsection{Proof of Proposition \ref{prop: RIC for QDA case}}
By Definition \ref{def: ric} and the fact $\deg(S) = |S|(|S| + 3)/2 + 1$, it's easy to obtain that
\begin{align}
	\textrm{RIC}_n(S) &= \frac{2}{n}\sum_{i=1}^n \mathds{1}(y_i = 0)\left[\frac{1}{2}\bm{x}_{i, S}^T((\hat{\Sigma}_{S, S}^{(0)})^{-1} - (\hat{\Sigma}_{S, S}^{(1)})^{-1})\bm{x}_{i, S} + ((\hat{\bmu}_{S}^{(1)})^T (\hat{\Sigma}_{S, S}^{(1)})^{-1}\right. \nonumber\\
	&\quad  - (\hat{\bmu}_{S}^{(0)})^T (\hat{\Sigma}_{S, S}^{(0)})^{-1})\bm{x}_{i, S} + \frac{1}{2} (\hat{\bmu}_{S}^{(0)})^T (\hat{\Sigma}_{S,S}^{(0)})^{-1} \hat{\bmu}_{S}^{(0)} - \frac{1}{2} (\hat{\bmu}_{S}^{(1)})^T (\hat{\Sigma}_{S,S}^{(1)})^{-1} \hat{\bmu}_{S}^{(1)}  \nonumber\\
	&\quad + \log(|\hat{\Sigma}_{S, S}^{(0)}|)- \log(|\hat{\Sigma}_{S, S}^{(1)}|) \bigg] + \frac{2}{n}\sum_{i=1}^n \mathds{1}(y_i = 1)\bigg[\frac{1}{2}\bm{x}_{i, S}^T((\hat{\Sigma}_{S, S}^{(1)})^{-1} - (\hat{\Sigma}_{S, S}^{(0)})^{-1})\bm{x}_{i, S}  \nonumber\\
	&\quad + ((\hat{\bmu}_{S}^{(0)})^T (\hat{\Sigma}_{S, S}^{(0)})^{-1} - (\hat{\bmu}_{S}^{(1)})^T (\hat{\Sigma}_{S, S}^{(1)})^{-1})\bm{x}_{i, S} + \frac{1}{2} (\hat{\bmu}_{S}^{(1)})^T (\hat{\Sigma}_{S,S}^{(1)})^{-1} \hat{\bmu}_{S}^{(1)} \nonumber\\
	&\quad - \frac{1}{2} (\hat{\bmu}_{S}^{(0)})^T (\hat{\Sigma}_{S,S}^{(0)})^{-1} \hat{\bmu}_{S}^{(0)} + \log(|\hat{\Sigma}_{S, S}^{(1)}|) -\log(|\hat{\Sigma}_{S, S}^{(0)}|)\bigg] + c_n \cdot (|S|(|S|+3)/2 + 1).
\end{align}
Denote the observations with class $r$ as $\{\bm{x}_{i}^{(r)}\}_{i=1}^{n_r}, r = 0,1$. And it holds that
\begin{align}
	&\frac{1}{n}\sum_{i=1}^n \mathds{1}(y_i = 0)\left[\bm{x}_{i, S}^T((\hat{\Sigma}_{S, S}^{(0)})^{-1} - (\hat{\Sigma}_{S, S}^{(1)})^{-1})\bm{x}_{i, S}\right] \nonumber\\
	&= \frac{n_0}{n}\cdot \frac{1}{n_0}\sum_{i=1}^{n_0}\left[(\bm{x}_{i, S}^{(0)})^T((\hat{\Sigma}_{S, S}^{(0)})^{-1} - (\hat{\Sigma}_{S, S}^{(1)})^{-1})\bm{x}_{i, S}^{(0)}\right] \nonumber\\
	&= \hat{\pi}_0 \cdot \frac{1}{n_0}\sum_{i=1}^{n_0}\tr\left[((\hat{\Sigma}_{S, S}^{(0)})^{-1} - (\hat{\Sigma}_{S, S}^{(1)})^{-1})\bm{x}_{i, S}^{(0)}(\bm{x}_{i, S}^{(0)})^T\right] \nonumber\\
	&= \hat{\pi}_0\tr\left[((\hat{\Sigma}_{S, S}^{(0)})^{-1} - (\hat{\Sigma}_{S, S}^{(1)})^{-1})\frac{1}{n_0}\sum_{i=1}^{n_0}\bm{x}_{i, S}^{(0)}(\bm{x}_{i, S}^{(0)})^T\right]\nonumber\\
	&= \hat{\pi}_0\tr\left[((\hat{\Sigma}_{S, S}^{(0)})^{-1} - (\hat{\Sigma}_{S, S}^{(1)})^{-1}) (\hat{\bmu}^{(0)}_{S}(\hat{\bmu}^{(0)}_{S})^T + \hat{\Sigma}^{(0)}_{S,S})\right]\nonumber\\
	&= \hat{\pi}_0(\hat{\bmu}^{(0)}_{S})^T((\hat{\Sigma}_{S, S}^{(0)})^{-1} - (\hat{\Sigma}_{S, S}^{(1)})^{-1})\hat{\bmu}^{(0)}_{S} + \hat{\pi}_0\tr\left[((\hat{\Sigma}_{S, S}^{(0)})^{-1} - (\hat{\Sigma}_{S, S}^{(1)})^{-1})\hat{\Sigma}^{(0)}_{S,S}\right].
\end{align}
Similarly we have
\begin{align}
	&\frac{1}{n}\sum_{i=1}^n \mathds{1}(y_i = 1)\left[\bm{x}_{i, S}^T((\hat{\Sigma}_{S, S}^{(1)})^{-1} - (\hat{\Sigma}_{S, S}^{(0)})^{-1})\bm{x}_{i, S}\right] \nonumber\\
	&= \hat{\pi}_1(\hat{\bmu}^{(1)}_{S})^T((\hat{\Sigma}_{S, S}^{(1)})^{-1} - (\hat{\Sigma}_{S, S}^{(0)})^{-1})\hat{\bmu}_{S}^{(1)} + \hat{\pi}_1\tr\left[((\hat{\Sigma}_{S, S}^{(1)})^{-1} - (\hat{\Sigma}_{S, S}^{(0)})^{-1})\hat{\Sigma}^{(1)}_{S,S}\right].
\end{align}
Combining with the fact \[
\frac{1}{n}\sum_{i=1}^n\mathds{1}(y_i = r)\bm{x}_{i, S} = \hat{\pi}_r \bmu^{(r)}_S, r = 0,1,
\]we obtain that
\begin{align}
	\textup{RIC}_n (S) &= - (\hat{\bm{\mu}}_{S}^{(1)} - \hat{\bm{\mu}}_{S}^{(0)})^T[\hat{\pi}_1(\hat{\Sigma}_{S, S}^{(0)})^{-1} + \hat{\pi}_0(\hat{\Sigma}_{S, S}^{(1)})^{-1}](\hat{\bm{\mu}}_{S}^{(1)} - \hat{\bm{\mu}}_{S}^{(0)}) \nonumber\\
	&\quad+ \textup{Tr}[((\hat{\Sigma}_{S, S}^{(1)})^{-1} - (\hat{\Sigma}_{S, S}^{(0)})^{-1})(\hat{\pi}_1\hat{\Sigma}_{S, S}^{(1)} - \hat{\pi}_0\hat{\Sigma}_{S, S}^{(0)})]+ (\hat{\pi}_1 - \hat{\pi}_0)(\log |\hat{\Sigma}_{S, S}^{(1)}| - \log|\hat{\Sigma}_{S, S}^{(0)}|) \nonumber\\ 
	&\quad+ c_n\cdot (|S|(|S| + 3)/2 + 1),
\end{align}
which completes the proof.

\subsection{Proof of Theorem \ref{thm: MC_error_bdd}}
Denote $ g_n(\alpha') = \pi_1 g_n^{(1)}(\alpha') - \pi_0 g_n^{(0)}(\alpha')$. By the definition of $\{\alpha_i\}_{i=1}^N$, it holds that 
	\begin{equation}
		g_n(\alpha) = 0, \textrm{when } \alpha \notin  \{\alpha_i\}_{i = 1}^N.
	\end{equation}
	Recall that $\bm{x}|y = 0 \sim \p^{(0)}, \bm{x}|y = 1 \sim \p^{(1)}$, where $\p^{(0)}, \p^{(1)}$ are the corresponding cumulative distribution functions. We have
	\begin{equation}
		\be[R(C^{RaSE}_n)] = \pi_0\int_{\mathcal{X}} \bp(\nu_n(\bm{x}) > \alpha) d\p^{(0)} + \pi_1\int_{\mathcal{X}} \bp(\nu_n(\bm{x}) \leq \alpha) d\p^{(1)}.
	\end{equation}
	For given $\bm{x}$ and the corresponding $\alpha' = \mu_n(\bm{x})$, we can construct a random variable $T = \sum_{j = 1}^{B_1} \mathds{1}\left\{C^{S_{j*}}_n(\bm{x}) = 1\right\} \sim Bin(B_1, \alpha')$. Then
	\begin{equation}
		\int_{\mathcal{X}} \bp(\nu_n(\bm{x}) \leq \alpha) d\p^{(1)} = \int_{[0, 1]} \p(T \leq B_1\alpha) dG_n^{(1)}(\alpha').
	\end{equation}
	Similarly,
	\begin{equation}
		\int_{\mathcal{X}} \bp(\nu_n(\bm{x}) > \alpha) d\p^{(0)} = 1 - \int_{[0, 1]} \p(T \leq B_1\alpha) dG_n^{(0)}(\alpha').
	\end{equation}
	This leads to
	\begin{equation}
		\be[R(C^{RaSE}_n)] = \pi_0 + \int_{[0, 1]} \p(T \leq B_1\alpha) dG_n(\alpha'),
	\end{equation}
	where $G_n(\alpha') = \pi_1 G_n^{(1)}(\alpha') - \pi_0 G_n^{(0)}(\alpha')$. And there also holds that
	\begin{align}
		R(C_n^{RaSE*}) &= \pi_0 \p^{(0)}(\mu_n(\bm{x}) > \alpha) + \pi_1 \p^{(1)}(\mu_n(\bm{x}) < \alpha) \nonumber\\ 
		&\quad + \frac{1}{2}[\pi_0 \p^{(0)}(\mu_n(\bm{x}) = \alpha)) + \pi_1 \p^{(1)}(\mu_n(\bm{x}) = \alpha)] \nonumber\\
		&= \pi_0 (1 - G_n^{(0)}(\alpha)) + \pi_1 G_n^{(1)}(\alpha) + \frac{1}{2}[\pi_0 g_n^{(0)}(\alpha) - \pi_1g_n^{(1)}(\alpha)] \nonumber\\
		&= \pi_0 + G_n(\alpha) - \frac{1}{2}g_n(\alpha),
	\end{align}
	where $g_n(\alpha') = \pi_1 g_n^{(1)}(\alpha') - \pi_0 g_n^{(0)}(\alpha')$. This implies
	\begin{equation}\label{bdd_MC}
		\be[R(C^{RaSE}_n)] - R(C_n^{RaSE*}) = \int_{[0, 1]} [\p(T \leq B_1\alpha) - \mathds{1}_{\{\alpha' \leq \alpha\}}] dG_n(\alpha') + \frac{1}{2}g_n(\alpha).
	\end{equation}
	
	\begin{enumerate}[label = (\roman*)]
		\item When $\alpha \notin  \{\alpha_i\}_{i = 1}^N$,
		$g_n(\alpha) = 0$ holds. For $\alpha' \in \{\alpha_i\}_{i = 1}^N$, by Hoeffding's inequality \citep{petrov2012sums}, we obtain that
			\begin{align}
			&|\p(T \leq B_1\alpha) - \mathds{1}_{\{\alpha' \leq \alpha\}}|  \nonumber\\
			&\leq \p(T - B_1\alpha' > B_1(\alpha - \alpha'))\mathds{1}_{\{\alpha' \leq \alpha\}} + \p(B_1\alpha' - T \geq B_1(\alpha' - \alpha))\mathds{1}_{\{\alpha' > \alpha\}} \nonumber\\
			&\leq \exp\{-2 B_1(\alpha' - \alpha)^2\} \nonumber\\
			&\leq \exp\left\{-C_{\alpha}B_1\right\},
		\end{align}
		where $C_{\alpha} = \min\limits_{1\leq i \leq N} |\alpha - \alpha_i|^2$. This leads to the final bound $|\be[R(C^{RaSE}_n)] - R(C_n^{RaSE*})| \leq \exp\left\{-C_{\alpha}B_1\right\}$.
				
		\item When $\alpha = \alpha_{i_0}, i_0 \in \{1,2, \ldots, N\}$, for $\alpha' \neq \alpha_{i_0}$, we  have
			\begin{equation}
				|\p(T \leq B_1\alpha) - \mathds{1}_{\{\alpha' \leq \alpha\}}| \leq \exp\{-2 B_1(\alpha' - \alpha)^2\}, 	
			\end{equation}
			leading to
			\begin{equation}\label{non_i0}
				\sum_{i \neq i_0} |\p(T \leq B_1\alpha) - \mathds{1}_{\{\alpha' \leq \alpha\}}|g_n(\alpha_i) = \exp\left\{-C_{\alpha}B_1\right\}\sum_{i \neq i_0}g_n(\alpha_i) \leq \exp\left\{-C_{\alpha}B_1\right\}, \nonumber
			\end{equation}
			where $C_{\alpha} = 2\min\limits_{1\leq i \leq N} |\alpha - \alpha_i|^2 = 2\min\limits_{i \neq i_0} |\alpha_{i_0} - \alpha_i|^2$. Therefore, again by \eqref{bdd_MC}, we have
			\begin{equation}
				|\be[R(C^{RaSE})] - R(C_n^{RaSE*})| \leq \left|[\p(T \leq B_1\alpha_{i_0}) - 1]\cdot g_n(\alpha_{i_0}) + \frac{1}{2}g_n(\alpha_{i_0})\right| + \exp\left\{-C_{\alpha}B_1\right\} . \nonumber\\
			\end{equation}
			By Berry-Esseen theorem \citep{esseen1956moment}, 
			\begin{equation}\label{BE}
				\p(T \leq B_1\alpha_{i_0}) = \p\left(\frac{T - B_1\alpha_{i_0}}{\sqrt{B_1(\alpha_{i_0}(1-\alpha_{i_0}))}} \leq 0\right) = \frac{1}{2} + O\left(\frac{1}{\sqrt{B_1}}\right),
			\end{equation}
			as $B_1 \rightarrow \infty$. Eventually there holds that
			\begin{equation}
				|\be[R(C^{RaSE})] - R(C_n^{RaSE*})| \leq  O\left(\frac{1}{\sqrt{B_1}}\right).
			\end{equation}
			This completes the proof. 
	\end{enumerate}
	
\subsection{Proof of Theorem \ref{thm: var}}
Referring to the proof of the bound on MC-variance in \cite{cannings2017random}, it can be obtained that
\begin{equation}
	\bv\left(\int_{\mathbb{R}^p}\mathds{1}_{\{\nu_n(\bm{x}) > \alpha\}} d\p^{(0)}\right) \leq 2\int_{[0, 1]}\int_{[0, \alpha'']}\p(T' \leq B_1\alpha)\p(T'' > B_1\alpha)dG_n^{(0)}(\alpha')dG_n^{(0)}(\alpha''),
\end{equation}
where given $\bm{x}$, $T' \sim Bin(B_1, \alpha'), T'' \sim Bin(B_1, \alpha'')$ and $T', T''$ are independent. 
\begin{enumerate}[label = (\roman*)]
	\item For $\alpha \notin \{\alpha^{(0)}_i\}_{i = 1}^{N_0} \cup \{\alpha^{(1)}_i\}_{i = 1}^{N_1}$: constant $C_{\alpha}^{(0)}  = 2\min\limits_{1 \leq i \leq N_0}(|\alpha - \alpha^{(0)}_i|^2) > 0, C_{\alpha}^{(1)} = 2\min\limits_{1 \leq i \leq N_1}(|\alpha - \alpha^{(1)}_i|^2) > 0$.
	If $\alpha \leq \alpha' \leq \alpha''$: Then by Hoeffding's inequality, we have
	\begin{equation}\label{eq_T1}
		\p(T' \leq B_1\alpha) = \p(T' - B_1\alpha' \leq B_1(\alpha - \alpha')) \leq \exp\{-B_1C_{\alpha}^{(0)}\}.
	\end{equation}
	If $\alpha' \leq \alpha'' \leq \alpha$: Similarly we have
	\begin{equation}\label{eq_T2}
		\p(T'' > B_1\alpha) \leq \exp\{-B_1C_{\alpha}^{(0)}\}.
	\end{equation}
	If $\alpha' \leq \alpha \leq \alpha''$: We have both \eqref{eq_T1} and \eqref{eq_T2}.
	
	Thus we always have
	\begin{equation}
		\p(T' \leq B_1\alpha)\p(T'' > B_1\alpha) \leq \exp\{-B_1C_{\alpha}^{(0)}\}.
	\end{equation}
	Since here the integration actually is the finite summation, it implies that
	\begin{equation}
		\bv\left(\int_{\mathcal{X}}\mathds{1}_{\{\nu_n(\bm{x}) > \alpha\}} d\p^{(0)}\right) \leq \exp\{-B_1C_{\alpha}^{(0)}\}.
	\end{equation}
	
	Since $\alpha \notin \{\alpha^{(1)}_i\}_{i = 1}^{N_1}$: We can obtain the similar conclusion that
	\begin{equation}
		\bv\left(\int_{\mathcal{X}}\mathds{1}_{\{\nu_n(\bm{x}) \leq \alpha\}} d\p^{(1)}\right) \leq \exp\{-B_1C_{\alpha}^{(1)}\}.
	\end{equation}
	Thus for any $\alpha \notin \{\alpha^{(0)}_i\}_{i = 1}^{N_0} \cup \{\alpha^{(1)}_i\}_{i = 1}^{N_1}$, setting $C_{\alpha} = 2\min\limits_{\substack{1 \leq i \leq N_0 \\ 1 \leq j \leq N_1}}(|\alpha - \alpha^{(0)}_i|^2, |\alpha - \alpha^{(1)}_j|^2) = \min(C_{\alpha}^{(0)}, C_{\alpha}^{(1)}) > 0$, then by the convexity, we have
\begin{align}
	\bv\left(R(C_n^{RaSE})\right) &= \bv\left(\pi_0\int_{\mathcal{X}}\mathds{1}_{\{\nu_n(\bm{x}) > \alpha\}} d\p^{(0)} + \pi_1\int_{\mathcal{X}}\mathds{1}_{\{\nu_n(\bm{x}) \leq \alpha\}} d\p^{(1)}\right) \\
	&\leq \pi_0 \bv\left(\int_{\mathcal{X}}\mathds{1}_{\{\nu_n(\bm{x}) > \alpha\}} d\p^{(0)}\right) + \pi_1 \bv\left(\int_{\mathcal{X}}\mathds{1}_{\{\nu_n(\bm{x}) \leq \alpha\}} d\p^{(1)}\right) \\
	&\leq \exp\{-B_1C_{\alpha}\}.
\end{align}
	
	\item For $\alpha = \alpha_{i_0}^{(0)}$ or $\alpha_{i_1}^{(1)}$: Without loss of generality, suppose $\alpha = \alpha_{i_0}^{(0)}$. When $\alpha' < \alpha''$, similar to (i), there exists positive number $C_{\alpha}'$ such that
	\begin{equation}
		\p(T' \leq B_1\alpha)\p(T'' > B_1\alpha) \leq \exp\{-B_1C_{\alpha}' \}.
	\end{equation}
	When $\alpha' = \alpha'' = \alpha_{i_0}^{(0)}$, similar to \eqref{BE}, there holds
	\begin{equation}
		\p(T' \leq B_1\alpha)\p(T'' > B_1\alpha) = \frac{1}{4} + O\left(\frac{1}{\sqrt{B_1}}\right).
	\end{equation}
	Thus it holds
	\begin{equation}
		\bv\left(\int_{\mathcal{X}}\mathds{1}_{\{\nu_n(\bm{x}) > \alpha\}} dP^{(0)}\right) \leq \frac{1}{2}(g^{(0)}_n(\alpha_{i_0}^{(0)}))^2 + O\left(\frac{1}{\sqrt{B_1}}\right).
	\end{equation}
	And similar results hold for $\bv\left(\int_{\mathcal{X}}\mathds{1}_{\{\nu_n(\bm{x}) \leq \alpha\}} dP^{(1)}\right)$. Eventually we would have
	\begin{align}
		\bv\left(R(C_n^{RaSE})\right) &\leq \pi_0 \bv\left(\int_{\mathcal{X}}\mathds{1}_{\{\nu_n(\bm{x}) > \alpha\}} d\p^{(0)}\right) + \pi_1 \bv\left(\int_{\mathcal{X}}\mathds{1}_{\{\nu_n(\bm{x}) \leq \alpha\}} d\p^{(1)}\right)\\
		 &\leq \frac{1}{2}\left[\pi_0(g^{(0)}_n(\alpha))^2 + \pi_1(g^{(1)}_n(\alpha))^2\right] + O\left(\frac{1}{\sqrt{B_1}}\right),
	\end{align}
	which completes the proof.
	\end{enumerate}

\subsection{Proof of Proposition \ref{prop: KL}}
\begin{enumerate}[label=(\roman*)]
	\item Due to Proposition \ref{prop: discriminative set}, this is trivial to be seen to hold here.
	\item First let's consider subspace $S = S' \cup \{j\}$, where $|S'| \geq 1, j \notin S'$. The conditional densities of $j|S'$ are denoted as $f^{(0)}_{j|S'}, f^{(1)}_{j|S'}$ and the components of $\bm{x}_S$ corresponding to $S'$ and $\{j\}$ are $\bm{x}_{S'}, \bm{x}_j$, respectively. The definition of KL divergence and Fubini theorem incur
	\begin{align}
		\textup{KL}(f_S^{(0)}||f_S^{(1)}) &= \e_{\bm{x}_S\sim f^{(0)}_S} \left[\log \left(\frac{f^{(0)}_{S'}(\bm{x}_{S'})}{f^{(1)}_{S'}(\bm{x}_{S'})}\right) + \log \left(\frac{f^{(0)}_{j|S'}(\bm{x}_j)}{f^{(1)}_{j|S'}(\bm{x}_j)}\right)\right] \nonumber\\
		&= \e_{\bm{x}_{S'}\sim f^{(0)}_{S'}}\left[\log \left(\frac{f^{(0)}_{S'}(\bm{x}_{S'})}{f^{(1)}_{S'}(\bm{x}_{S'})}\right)\right] + \e_{\bm{x}_{S'}\sim f^{(0)}_{S'}}\left[\e_{\bm{x}_j\sim f^{(0)}_{j|S'}}\log \left(\frac{f^{(0)}_{j|S'}(\bm{x}_j)}{f^{(1)}_{j|S'}(\bm{x}_j)}\right)\right] \nonumber\\
		&= \textup{KL}(f_{S'}^{(0)}||f_{S'}^{(1)}) + \e_{\bm{x}_{S'}\sim f^{(0)}_{S'}}\left[\textup{KL}(f_{j|S'}^{(0)}||f_{j|S'}^{(1)})\right] \label{eq: proof of KL lemma}\\
		&\geq \textup{KL}(f_{S'}^{(0)}||f_{S'}^{(1)}).
	\end{align}
	Here ``=" holds if and only if $f_{j|S'}^{(0)}(\bm{x}_j|\bm{x}_{S'}) = f_{j|S'}^{(1)}(\bm{x}_j|\bm{x}_{S'}) \,\,a.s.$ with respect to $\p^{(0)}$. By induction, this indicates that for any $S \supseteq S'$ and $S' \neq \emptyset$, there holds
	\begin{equation}\label{eq: monotone KL}
		\textup{KL}(f_S^{(0)}||f_S^{(1)}) \geq \textup{KL}(f_{S'}^{(0)}||f_{S'}^{(1)}).
	\end{equation}
	
	Note that in \eqref{eq: monotone KL}, by Proposition \ref{prop: discriminative set}, if $f_{j|S'}^{(0)}(\bm{x}_j|\bm{x}_{S'}) = f_{j|S'}^{(1)}(\bm{x}_j|\bm{x}_{S'}) \,\,a.s.$ with respect to $\p^{(0)}$, we have
	\begin{equation}\label{eq: ratio}
		\frac{f^{(0)}_{S}(\bm{x}_{S})}{f^{(1)}_{S}(\bm{x}_{S})} = \frac{f^{(0)}_{S'}(\bm{x}_{S'})}{f^{(1)}_{S'}(\bm{x}_{S'})}, \,a.s.\,\,, w.r.t. \,\, \p^{(0)}.
	\end{equation}
	Similarly, we have
	\begin{equation}\label{eq: monotone KL 2}
		\textup{KL}(f_S^{(1)}||f_S^{(0)}) \geq \textup{KL}(f_{S'}^{(1)}||f_{S'}^{(0)}),
	\end{equation}
	where the $``="$ holds if and only if
	\begin{equation}\label{eq: ratio 2}
		\frac{f^{(1)}_{S}(\bm{x}_{S})}{f^{(0)}_{S}(\bm{x}_{S})} = \frac{f^{(1)}_{S'}(\bm{x}_{S'})}{f^{(0)}_{S'}(\bm{x}_{S'})}, \,a.s.\,\,, w.r.t. \,\, \p^{(1)}.
	\end{equation}
	If $S \not\supseteq S^*$, consider $\bar{S} = S \cup (S^* \backslash S) \supseteq S^*$, then by (\rom{1}) and \eqref{eq: monotone KL}, there holds
	\begin{align}
			\pi_0\textup{KL}(f_S^{(0)}||f_S^{(1)}) +  \pi_1\textup{KL}(f_S^{(1)}||f_S^{(0)}) &\leq \pi_0\textup{KL}(f_{\bar{S}}^{(0)}||f_{\bar{S}}^{(1)}) + \pi_1\textup{KL}(f_{\bar{S}}^{(1)}||f_{\bar{S}}^{(0)}) \label{eq: sbar}\\
			&= \pi_0\textup{KL}(f_{S^*}^{(0)}||f_{S^*}^{(1)}) + \pi_1\textup{KL}(f_{S^*}^{(1)}||f_{S^*}^{(0)}).
	\end{align}
	Then by Proposition \ref{prop: discriminative set}, if the ``=" holds in \eqref{eq: sbar}, due to \eqref{eq: ratio} and \eqref{eq: ratio 2}, we have
	\begin{align}
		\frac{f^{(0)}_{S}(\bm{x}_{S})}{f^{(1)}_{S}(\bm{x}_{S})} = \frac{f^{(0)}_{\bar{S}}(\bm{x}_{\bar{S}})}{f^{(1)}_{\bar{S}}(\bm{x}_{\bar{S}})} = \frac{f^{(0)}_{S^*}(\bm{x}_{S^*})}{f^{(1)}_{S^*}(\bm{x}_{S^*})}, \,a.s.\,\,, w.r.t. \,\, \p^{\bx} = \pi_0\p^{(0)} + \pi_1\p^{(1)},
	\end{align}
	implying that $S$ is a discriminative set but $S \not\supseteq S^*$, which yields a contradiction. Therefore (\rom{2})  holds here.
\item holds since the full model $S_{\textup{Full}} \supseteq S^*$ and the KL divergence is monotone in the sense of \eqref{eq: monotone KL} and \eqref{eq: monotone KL 2}.
\end{enumerate}

\subsection{Proof of Proposition \ref{prop: lda check}}
Denote $R_S^{0|1}(\bmusr{0}, \bmusr{1}, \Sigma_{S,S}, \bx_S) = \log\left(\frac{f_S^{(0)}(\bx_S)}{f_S^{(1)}(\bx_S)} \right) = \left(\bmusr{0} - \bmusr{1}\right)^T\invsig_{S, S}\left[\bx_S - \frac{1}{2}(\bmusr{0} + \bmusr{1})\right]$. To eliminate any confusion and better illustrate the meaning of the gradient and second-order derivative matrix, we will use $\frac{\partial R_S^{0|1}}{\partial \bmusr{0}}, \frac{\partial R_S^{0|1}}{\partial \bmusr{1}}, \frac{\partial R_S^{0|1}}{\partial \Sigma_{S, S}}$ to represent the gradient and use $\frac{\partial^2 R_S^{0|1}}{\partial \bmusr{0} \partial \Sigma_{S, S}}, \frac{\partial^2 R_S^{0|1}}{\partial \bmusr{1} \partial \Sigma_{S, S}}, \frac{\partial^2 R_S^{0|1}}{\partial \Sigma_{S, S} \partial \Sigma_{S, S}}$ to represent the second-order derivative matrix.

According to \cite{brookes2005matrix} and \cite{petersen2012matrix}, with some calculation we can obtain that
\begin{align}
	\frac{\partial R_S^{0|1}}{\partial \bmusr{0}} &= \invsig_{S, S}\left(\bxs - \bmusr{0}\right),\\
	\frac{\partial R_S^{0|1}}{\partial \bmusr{1}} &= \invsig_{S, S}\left(\bxs - \bmusr{1}\right), \\
	\frac{\partial R_S^{0|1}}{\partial \Sigma_{S, S}} &= \invsig_{S, S}\left(\bmusr{0} - \bmusr{1}\right)\left(\bx_S - \frac{1}{2}(\bmusr{0} + \bmusr{1})\right)^T\invsig_{S, S}, \\
	\frac{\partial^2 R_S^{0|1}}{\partial \bmusr{0}\partial \bmusr{0}}&= \frac{\partial^2 R_S^{0|1}}{\partial \bmusr{1}\partial \bmusr{1}} = -\invsig_{S, S},\\
	\frac{\partial^2 R_S^{0|1}}{\partial \Sigma_{S, S} \partial \bmusr{0}} &= -\left(\invsig_{S, S} \otimes \invsig_{S, S}\right) \left(I_{|S|} \otimes \left(\bxs - \bmusr{0}\right) \right) ,\\
	\frac{\partial^2 R_S^{0|1}}{\partial \Sigma_{S, S} \partial \bmusr{1}} &= -\left(\invsig_{S, S} \otimes \invsig_{S, S}\right) \left(I_{|S|} \otimes \left(\bxs - \bmusr{1}\right) \right),\\
	\frac{\partial^2 R_S^{0|1}}{\partial \Sigma_{S, S} \partial \Sigma_{S, S}} &= -\left[I_{|S|} \otimes \Sigma^{-1}_{S, S}\left(\bmusr{0} - \bmusr{1}\right)\left(\bx_S - \frac{1}{2}(\bmusr{0} + \bmusr{1})\right)^T\right.\\
	&\quad \left. + \left(\bx_S - \frac{1}{2}(\bmusr{0} + \bmusr{1})\right)\left(\bmusr{0} - \bmusr{1}\right)^T \Sigma^{-1}_{S, S} \otimes I_{|S|}\right] \left(\invsig_{S, S} \otimes \invsig_{S, S}\right),
\end{align}
where $\otimes$ is the Kronecker product.

Then let's check conditions in Assumption \ref{asmp: consistency} one by one. Without loss of generality, for $(\rom{3})$, we only check the case that $\bx_{i, S} \overset{i.i.d.}{\sim} f^{(0)}_S$. And for $(\rom{4}), (\rom{5})$, we only check the case that $\bxs \sim f^{(0)}_S$.
\begin{enumerate}[label=(\roman*)]
	\item It's easy to see that the number of parameters in LDA model in $p$-dimensional space is $2 + 2p + \frac{p(p+1)}{2}$, therefore $\kappa_1 = 2$.
	\item According to Assumption \ref{asmp: lda}, we have
		\begin{align}
			\kl(f^{(0)}||f^{(1)}) &= \left(\bmur{1} - \bmur{0}\right)^T\invsig\left(\bmur{1} - \bmur{0}\right) \\
			&\leq \twonorm{\bmur{1}_{S^*} - \bmur{0}_{S^*}}^2\cdot \twonorm{\invsig_{S^*, S^*}} \\
			&\leq p^*(M')^2m^{-1}\\
			&\leq D(M')^2m^{-1}.	
		\end{align}
		The similar conclusion holds for $\kl(f^{(1)}||f^{(0)})$ as well.
	\item For any $(\tilde{\bmu}_S^{(0)}, \tilde{\bmu}_S^{(1)}, \tilde{\Sigma}_{S,S})$:
		\begin{align}
			&\maxnorma{\frac{1}{n}\sum_{i=1}^n \frac{\partial^2 R_S^{0|1}}{\partial \Sigma_{S, S} \partial \Sigma_{S, S}}(\tilde{\bmu}_S^{(0)}, \tilde{\bmu}_S^{(1)}, \tilde{\Sigma}_{S,S}, \bx_i)} \\
			&\leq \twonorma{\frac{1}{n}\sum_{i=1}^n \frac{\partial^2 R_S^{0|1}}{\partial \Sigma_{S, S} \partial \Sigma_{S, S}}(\tilde{\bmu}_S^{(0)}, \tilde{\bmu}_S^{(1)}, \tilde{\Sigma}_{S,S}, \bx_i)} \\
			&\leq 2\twonorma{\tilde{\Sigma}_{S,S}^{-1}(\tilde{\bmu}_S^{(0)} - \tilde{\bmu}_S^{(1)})}\cdot \twonorma{\frac{1}{n}\sum_{i=1}^n \left[\bx_{i, S} - \frac{1}{2}(\tilde{\bmu}_S^{(0)} + \tilde{\bmu}_S^{(1)})\right]}\cdot \twonorma{\tilde{\Sigma}_{S,S}^{-1}}^2 \\
			&\leq 2\twonorma{\tilde{\Sigma}_{S,S}^{-1}}^3\cdot \twonorma{\tilde{\bmu}_S^{(0)} - \tilde{\bmu}_S^{(1)}}\cdot \twonorma{\frac{1}{n}\sum_{i=1}^n \left[\bx_{i, S} - \frac{1}{2}(\tilde{\bmu}_S^{(0)} + \tilde{\bmu}_S^{(1)})\right]}.
		\end{align}
		When $\twonorm{\tilde{\bmu}_S^{(0)} - \bmu_S^{(0)}}, \twonorm{\tilde{\bmu}_S^{(1)} - \bmu_S^{(1)}}, \fnorm{\tilde{\Sigma}_{S,S} - \Sigma_{S, S}} \leq \zeta < m$, due to \eqref{eq: inv ineq}, it follows that
		\begin{equation}
			\twonorm{\tilde{\Sigma}^{-1}_{S, S} - \Sigma^{-1}_{S, S}} \leq \frac{\frac{1}{m^2}\twonorm{\tilde{\Sigma}_{S, S} - \Sigma_{S, S}}}{1 - \frac{1}{m}\twonorm{\tilde{\Sigma}_{S, S} - \Sigma_{S, S}}} \leq \frac{\frac{1}{m^2}\fnorm{\tilde{\Sigma}_{S, S} - \Sigma_{S, S}}}{1 - \frac{1}{m}\fnorm{\tilde{\Sigma}_{S, S} - \Sigma_{S, S}}} \leq \frac{\zeta}{m^2-m\zeta},
		\end{equation}
		which leads to
		\begin{equation}
			\twonorm{\tilde{\Sigma}^{-1}_{S, S}} \leq \twonorm{\Sigma^{-1}_{S, S}} + \twonorm{\tilde{\Sigma}^{-1}_{S, S} - \Sigma^{-1}_{S, S}} \leq \frac{1}{m - \zeta}.
		\end{equation}
		In addition, we have
		\begin{equation}
			\twonorm{\tilde{\bmu}_S^{(0)} - \tilde{\bmu}_S^{(1)}} \leq \twonorm{\tilde{\bmu}_S^{(0)} - \bmu_S^{(0)}} + \twonorm{\tilde{\bmu}_S^{(1)} - \bmu_S^{(1)}} + \twonorm{\bmu_S^{(0)} - \bmu_S^{(1)}} \lesssim D^{\frac{1}{2}}M'.
		\end{equation}
		Without loss of generality, consider $\bx_{i, S} \overset{i.i.d.}{\sim} f^{(0)}_S$, it holds that
		\begin{align}
			\twonorma{\frac{1}{n}\sum_{i=1}^n \left[\bx_{i, S} - \frac{1}{2}(\tilde{\bmu}_S^{(0)} + \tilde{\bmu}_S^{(1)})\right]} &\leq \twonorma{\frac{1}{n}\sum_{i=1}^n\left(\bx_{i, S} - \bmu_S^{(0)}\right)} + \frac{1}{2}\twonorm{\tilde{\bmu}_S^{(0)} - \bmu_S^{(0)}} \\
			&\quad + \frac{1}{2}\twonorm{\tilde{\bmu}_S^{(1)} - \bmu_S^{(1)}} + \frac{1}{2}\twonorm{\bmu_S^{(0)} - \bmu_S^{(1)}}\\
			&\lesssim \twonorma{\frac{1}{n}\sum_{i=1}^n\left(\bx_{i, S} - \bmu_S^{(0)}\right)} + D^{\frac{1}{2}}M'.
		\end{align}
		By Proposition 1 in \cite{hsu2012tail}, since $\frac{1}{n}\sum_{i=1}^n(\bx_{i, S} - \bmu_S^{(0)}) \sim N(\bm{0}, \frac{1}{n}\Sigma_{S, S})$, it follows that
		\begin{equation}
			\tp\left(\twonorma{\frac{1}{n}\sum_{i=1}^n\left(\bx_{i, S} - \bmu_S^{(0)}\right)} > \sqrt{\frac{1}{n}\tr(\Sigma_{S, S}) + 2\sqrt{\tr(\Sigma_{S, S}^2)\cdot \frac{\epsilon}{n}} + 2\twonorm{\Sigma_{S, S}}\epsilon}\right) \leq \exp\{-n\epsilon\}.
		\end{equation}
		And due to Assumption \ref{asmp: lda}, we have
		\begin{align}
			\tr(\Sigma_{S, S}) &\leq D\maxnorm{\Sigma_{S, S}} \leq DM, \\
			\tr(\Sigma_{S, S}^2) &= \fnorm{\Sigma_{S, S}}^2 \leq D^2\maxnorm{\Sigma_{S, S}}^2 \leq D^2 M^2, \\
			\twonorm{\Sigma_{S, S}} &\leq D\maxnorm{\Sigma_{S, S}} \leq DM,
		\end{align}
		yielding
		\begin{equation}
			\tp\left(\twonorma{\frac{1}{n}\sum_{i=1}^n\left(\bx_{i, S} - \bmu_S^{(0)}\right)} > \sqrt{\frac{1}{n}DM + 2DM\sqrt{\frac{\epsilon}{n}} + 2DM\epsilon}\right) \leq \exp\{-n\epsilon\}.
		\end{equation}
		Therefore when $\twonorm{\tilde{\bmu}_S^{(0)} - \bmu_S^{(0)}}, \twonorm{\tilde{\bmu}_S^{(1)} - \bmu_S^{(1)}}, \twonorm{\tilde{\Sigma}_{S,S} - \Sigma_{S, S}} \leq \zeta$, we have
		\begin{align}
			\maxnorma{\frac{1}{n}\sum_{i=1}^n \frac{\partial^2 R_S^{0|1}}{\partial \Sigma_{S, S} \partial \Sigma_{S, S}}(\tilde{\bmu}_S^{(0)}, \tilde{\bmu}_S^{(1)}, \tilde{\Sigma}_{S,S}, \bx_i)} &\lesssim D^{\frac{1}{2}}M\left(\twonorma{\frac{1}{n}\sum_{i=1}^n\left(\bx_{i, S} - \bmu_S^{(0)}\right)} + D^{\frac{1}{2}}M\right) \\
			&\coloneqq V_S(\{\bx_{i, S}\}_{i=1}^n),
		\end{align}
		and
		\begin{equation}
			\tp\left(V_S(\{\bx_{i, S}\}_{i=1}^n) > CD\right) \lesssim \exp\{-Cn\}.
		\end{equation}
		Thus we can set $\kappa_2 = 1$.
		\item $\invsig_{S, S}\left(\bxs - \bmusr{0}\right) \sim N(\bm{0}, \Sigma_{S, S}^{-1}), \invsig_{S, S}\left(\bxs - \bmusr{1}\right) \sim N(\Sigma_{S, S}^{-1}(\bmusr{0}-\bmusr{1}), \Sigma_{S, S}^{-1})$ when $\bxs \sim f^{(0)}_S$. And also we have \[
			\maxnorma{\Sigma_{S, S}^{-1}} \leq \twonorma{\Sigma_{S, S}^{-1}} \leq \twonorma{\Sigma^{-1}} \leq m^{-1}.
		\]Then each component of $\frac{\partial R_S^{0|1}}{\partial \bmusr{0}}$ and $\frac{\partial R_S^{0|1}}{\partial \bmusr{1}}$ is $\sqrt{m^{-1}}$-subGaussian. On the other hand, since $\invsig_{S, S}\left(\bx_S - \frac{1}{2}(\bmusr{0} + \bmusr{1})\right) \sim N\left(\frac{1}{2}\Sigma_{S, S}^{-1}(\bmusr{0}-\bmusr{1}), \Sigma_{S, S}^{-1}\right)$ and
		\begin{equation}\label{eq: inv dif inf}
			\infnorma{\Sigma_{S, S}^{-1}(\bmusr{0}-\bmusr{1})} \leq D^{\frac{1}{2}}\twonorma{\Sigma_{S, S}^{-1}}\cdot \infnorma{\bmusr{0}-\bmusr{1}} \leq m^{-1}M'D^{\frac{1}{2}},
		\end{equation}
		it follows that each component of $\frac{\partial R_S^{0|1}}{\partial \Sigma_{S, S}}$ is $m^{-1}\sqrt{M'D}$-subGaussian. Therefore $\kappa_3 = \frac{1}{2}$.
		\item Notice that because of \eqref{eq: inv dif inf}, we have
		\begin{equation}
			\infnorma{\te_{\bxs \sim f^{(0)}_S}\left[\Sigma_{S, S}^{-1}(\bxs - \bmusr{1})\right]} = \infnorma{\Sigma_{S, S}^{-1}(\bmusr{0}-\bmusr{1})} \leq m^{-1}M'D^{\frac{1}{2}},
		\end{equation}
		\begin{align}
			&\maxnorma{\te_{\bxs \sim f^{(0)}_S}\left[\invsig_{S, S}\left(\bmusr{0} - \bmusr{1}\right)\left(\bx_S - \frac{1}{2}(\bmusr{0} + \bmusr{1})\right)^T\invsig_{S, S}\right]} \\
			&= \frac{1}{2}\maxnorma{\invsig_{S, S}\left(\bmusr{0} - \bmusr{1}\right)\left(\bmusr{0} -  \bmusr{1}\right)^T\invsig_{S, S}} \\
			&\leq \frac{1}{2}\infnorma{\Sigma_{S, S}^{-1}(\bmusr{0}-\bmusr{1})}^2 \\
			&\leq \frac{1}{2}(m^{-1}M')^2D.
		\end{align}
		Therefore $\kappa_4 = 1$.
		\item It's easy to see that $\left(\bmusr{1} - \bmusr{0}\right)^T\invsig_{S, S}\left[\bx_S - \frac{1}{2}(\bmusr{0} + \bmusr{1})\right] \sim N((\bmusr{1} - \bmusr{0})^T\invsig_{S, S}(\bmusr{1} - \bmusr{0}), (\bmusr{1} - \bmusr{0})^T\invsig_{S, S}(\bmusr{1} - \bmusr{0}))$. And there holds
		\begin{equation}
			(\bmusr{1} - \bmusr{0})^T\invsig_{S, S}(\bmusr{1} - \bmusr{0}) \leq \twonorma{\bmusr{1} - \bmusr{0}}^2 \cdot \twonorma{\invsig_{S, S}} \leq m^{-1}(M')^2D,
		\end{equation}
		which yields that $\log\left(\frac{f_S^{(0)}(\bx_S)}{f_S^{(1)}(\bx_S)} \right)$ is $M'\sqrt{m^{-1}D}$-subGaussian. So $\kappa_5 = \frac{1}{2}$.
		\item This can be easily derived from Lemma \ref{lem: lda consis 1} and its proof.
		\item This is obvious because of Lemma \ref{lem: lda consis 4}.(\rom{3}) and Assumption \ref{asmp: lda}.(\rom{3}).
\end{enumerate}

\subsection{Proof of Theorem \ref{thm: consistency}}
First suppose that we have $n_0$ observations of class 0 $\{\bm{x}_i^{(0)}\}_{i=1}^{n_0}$ and $n_1$ observations of class 1 $\{\bm{x}_i^{(1)}\}_{i=1}^{n_1}$, where $n_0 + n_1 = n$. Define \[
G_{n_0, S}^{(0)}(\bthetas') = \frac{1}{n_0}\sum_{i = 1}^{n_0}\log \left[\frac{f_S^{(0)}(\bm{x}_{i,S}^{(0)}|\bthetas')}{f_S^{(1)}(\bm{x}_{i,S}^{(0)}|\bthetas')}\right], G_{n_1, S}^{(1)}(\bthetas') = \frac{1}{n_1}\sum_{i = 1}^{n_1}\log \left[\frac{f_S^{(1)}(\bm{x}_{i,S}^{(1)}|\bthetas')}{f_S^{(0)}(\bm{x}_{i,S}^{(1)}|\bthetas')}\right],
\] for any $\bthetas'$. Denote $\widehat{\ric}(S) = -2[\hat{\pi}_0G_{n_0, S}^{(0)}(\hthetas) + \hat{\pi}_1G_{n_1, S}^{(1)}(\hthetas)]$, then $\ric_n(S) = \widehat{\ric}(S) + c_n\deg(S)$.
\begin{lemma}\label{lem: cons}
	If Assumptions 1-3 holds, then we have
	\begin{align}
		\tp\left(\sup_{S: |S| \leq D}|\widehat{\ric}(S) - \ric(S)| > \epsilon \right) &\lesssim p^D D^{\kappa_1}\exp\left\{-Cn\left(\frac{\epsilon}{D^{\kappa_1}(D^{\kappa_3} + D^{\kappa_4})}\right)^2\right\} \\
	&\quad +p^D D^{\kappa_1}\exp\left\{-Cn\cdot \frac{\epsilon}{D^{2\kappa_1 + \kappa_2}}\right\} \\
	&\quad + p^D\exp\left\{-Cn\left(\frac{\epsilon}{D^{\kappa_5}}\right)^2\right\}.
	\end{align}
\end{lemma}

\begin{proof}[Proof of Lemma \ref{lem: cons}]
By Taylor expansion and mean-value theorem, there exists $\lambda \in (0,1)$ and $\tilde{\btheta}_S = \lambda\bthetas + (1-\lambda)\hthetas$ satisfying that
\begin{equation}
	G_{n_0, S}^{(0)}(\hthetas) = G_{n_0, S}^{(0)}(\bthetas) + \nabla G_{n_0, S}^{(0)}(\bthetas)^T (\hthetas - \bthetas) + \frac{1}{2}(\hthetas - \bthetas)^T \nabla^2 G_{n_0, S}^{(0)}(\tilde{\btheta}) (\hthetas - \bthetas). \label{eq: taylor 0}
\end{equation}
Notice that
\begin{equation}
	\twonorma{\nabla G_{n_0, S}^{(0)}(\bthetas)^T (\hthetas - \bthetas)} \leq D^{\kappa_1}\infnorma{\nabla G_{n_0, S}^{(0)}(\bthetas)}\infnorm{\hthetas - \bthetas},
\end{equation}
and when $\twonorm{\hthetas - \bthetas} \leq \zeta$, we also have
\begin{align}
	\norma{(\hthetas - \bthetas)^T \nabla^2 G_{n_0, S}^{(0)}(\tilde{\btheta}) (\hthetas - \bthetas)} &\lesssim D^{\kappa_1}\infnorm{\hthetas - \bthetas}^2 \maxnorm{\nabla^2 G_{n_0, S}^{(0)}(\tilde{\btheta})} \\
	&\leq D^{2\kappa_1}\infnorm{\hthetas - \bthetas}^2\cdot \norma{V_S(\{\bx^{(0)}_{i, S}\}_{i=1}^{n_0})}.
\end{align}
Since $n_r \sim Bin(n, \pi_r), r = 0,1$, by Hoeffding's inequality, we have
\begin{equation}
	\tp(|n_r - n\pi_r| > 0.5n\pi_r) \lesssim \exp\{-Cn\}, r = 0, 1.
\end{equation}
Because of Assumption \ref{asmp: consistency}.(\rom{4}), given $n_0$, each component of $n_0\nabla G_{n_0, S}^{(0)}(\bthetas)$ is $n_0\sqrt{2C_3}D^{\kappa_3}$-subGaussian, then the tail bound in the below holds:
\begin{equation}
	\tp\left(\infnorma{\nabla G_{n_0, S}^{(0)}(\bthetas) - \te_{\bm{x}_S \sim f^{(0)}_S} \nabla_{\bthetas} L_S(\bx; \btheta)} > \epsilon \Bigg| n_0\right) \lesssim D^{\kappa_1}\exp\left\{-Cn_0\left(\frac{\epsilon}{D^{\kappa_3}}\right)^2\right\}.
\end{equation}

Then according to all conclusions above, we have
\begin{align}
	&\tp(|G_{n_0, S}^{(0)}(\hthetas) - \textup{KL}(f_S^{(0)}||f_S^{(1)})| > \epsilon) \\
	&\leq \e_{n_0}[\tp(|G_{n_0, S}^{(0)}(\hthetas) - \textup{KL}(f_S^{(0)}||f_S^{(1)})| > \epsilon||n_0 - n\pi_0| \leq 0.5n\pi_0)] + \tp(|n_0 - n\pi_0| > 0.5n\pi_0)\\
	&\leq \e_{n_0}\tp\left(\norma{G_{n_0, S}^{(0)}(\bthetas) - \textup{KL}(f_S^{(0)}||f_S^{(1)})}  > \epsilon/3 \Big||n_0 - n\pi_0| \leq 0.5n\pi_0\right) +  \\
	&\quad + \e_{n_0}\tp\left(\infnorma{\nabla G_{n_0, S}^{(0)}(\bthetas) - \te_{\bm{x}_S \sim f^{(0)}_S} \nabla G_{n_0, S}^{(0)}(\bthetas)} > CD^{\kappa_3} \Bigg| |n_0 - n\pi_0| \leq 0.5n\pi_0\right) \\
	&\quad + \tp\left(CD^{\kappa_1}(D^{\kappa_3} + D^{\kappa_4}) \infnorm{\hthetas - \bthetas} > \epsilon/3\right) + \tp\left(0.5CD^{2\kappa_1 + \kappa_2} \infnorm{\hthetas - \bthetas}^2 > \epsilon/3\right)\\
	&\quad +  \e_{n_0}\tp\left(V_S(\{\bx_{i, S}^{(0)}\}_{i=1}^{n_0}) > CD^{\kappa_2} \Bigg| |n_0 - n\pi_0| \leq 0.5n\pi_0\right)\\
	&\quad + \tp(|n_0 - n\pi_0| > 0.5n\pi_0) + \tp(D^{\frac{1}{2}\kappa_1}\infnorm{\hthetas - \bthetas} > \zeta)\\
	&\lesssim D^{\kappa_1}\exp\left\{-Cn\left(\frac{\epsilon}{D^{\kappa_1}(D^{\kappa_3} + D^{\kappa_4})}\right)^2\right\} + D^{\kappa_1}\exp\left\{-Cn\cdot \frac{\epsilon}{D^{2\kappa_1 + \kappa_2}}\right\} + \exp\left\{-Cn\left(\frac{\epsilon}{D^{\kappa_5}}\right)^2\right\},
\end{align}
which yields
\begin{align}
	\tp(|\hat{\pi}_0G_{n_0, S}^{(0)}(\hthetas) - \pi_0\textup{KL}(f_S^{(0)}||f_S^{(1)})| > \epsilon)&\lesssim  D^{\kappa_1}\exp\left\{-Cn\left(\frac{\epsilon}{D^{\kappa_1}(D^{\kappa_3} + D^{\kappa_4})}\right)^2\right\} \\
	&\quad + D^{\kappa_1}\exp\left\{-Cn\cdot \frac{\epsilon}{D^{2\kappa_1 + \kappa_2}}\right\} + \exp\left\{-Cn\left(\frac{\epsilon}{D^{\kappa_5}}\right)^2\right\}.
\end{align}
Similarly, there holds
\begin{align}
	\tp(|\hat{\pi}_1G_{n_1, S}^{(1)}(\hthetas) - \pi_1\textup{KL}(f_S^{(1)}||f_S^{(0)})| > \epsilon)&\lesssim D^{\kappa_1}\exp\left\{-Cn\left(\frac{\epsilon}{D^{\kappa_1}(D^{\kappa_3} + D^{\kappa_4})}\right)^2\right\} \\
	&\quad + D^{\kappa_1}\exp\left\{-Cn\cdot \frac{\epsilon}{D^{2\kappa_1 + \kappa_2}}\right\} + \exp\left\{-Cn\left(\frac{\epsilon}{D^{\kappa_5}}\right)^2\right\}.
\end{align}
Since $\widehat{\ric}(S) = -2[\hat{\pi}_0G_{n_0, S}^{(0)}(\hthetas) + \hat{\pi}_1G_{n_1, S}^{(1)}(\hthetas)]$, we obtain that
\begin{align}
		\tp\left(|\widehat{\ric}(S) - \ric(S)| > \epsilon \right) &\lesssim D^{\kappa_1}\exp\left\{-Cn\left(\frac{\epsilon}{D^{\kappa_1}(D^{\kappa_3} + D^{\kappa_4})}\right)^2\right\} \\
	&\quad + D^{\kappa_1}\exp\left\{-Cn\cdot \frac{\epsilon}{D^{2\kappa_1 + \kappa_2}}\right\} + \exp\left\{-Cn\left(\frac{\epsilon}{D^{\kappa_5}}\right)^2\right\}.
\end{align}
Due to the union bound over all $\binom{p}{D} = O(p^D)$ possible subsets, we obtain the conclusion.

\end{proof}

Let's now prove Theorem \ref{thm: consistency}. 
\begin{enumerate}[label=(\roman*)]
	\item It's easy to see that
	\begin{align}
		&\tp\left(\sup_{\substack{S: S \supseteq S^* \\ |S| \leq D}}\textup{RIC}_n(S) \geq \inf_{\substack{S: S \not\supseteq S^* \\ |S| \leq D}}\textup{RIC}_n(S)\right) \\
		&\leq \tp\left(\ric(S^*) + c_n\supp \deg(S) + \Delta/3 \geq \inf_{\substack{S: S \not\supseteq S^* \\ |S| \leq D}}\ric(S) - \Delta/3\right) \\
		&\quad + \tp\left(\sup_{S: |S| \leq D}|\widehat{\ric}(S) - \ric(S)| > \Delta/3 \right)\\
		&\leq \tp(c_n\supp \deg(S)  \geq \Delta/3) + \tp\left(\sup_{S: |S| \leq D}|\widehat{\ric}(S) - \ric(S)| > \Delta/3 \right)\\
		&\lesssim   p^{D}D^{\kappa_1}\exp\left\{-Cn\left(\frac{\Delta}{D^{\kappa_1}(D^{\kappa_3} + D^{\kappa_4})}\right)^2\right\}+ p^{D}D^{\kappa_1}\exp\left\{-Cn\left(\frac{\Delta}{D^{2\kappa_1 + \kappa_2}}\right)\right\}\\
				&\quad + p^{D}\exp\left\{-Cn\left(\frac{\Delta}{D^{\kappa_5}}\right)^2\right\} \\
				&\rightarrow 0.
	\end{align}
	\item Similar to above, we have
	\begin{align}
		&\tp\left(\textup{RIC}_n(S^*) \geq \inf_{\substack{S: S \not\supseteq S^* \\ |S| \leq D}}\textup{RIC}_n(S)\right) \\
		&\leq \tp\left(\ric(S^*) + c_n\deg(S^*) + \Delta/3 \geq \inf_{\substack{S: S \not\supseteq S^* \\ |S| \leq D}}\ric(S) - \Delta/3\right) \\
		&\quad + \tp\left(\sup_{S: |S| \leq D}|\widehat{\ric}(S) - \ric(S)| > \Delta/3 \right)\\
		&\lesssim   p^{D}D^{\kappa_1}\exp\left\{-Cn\left(\frac{\Delta}{D^{\kappa_1}(D^{\kappa_3} + D^{\kappa_4})}\right)^2\right\}+ p^{D}D^{\kappa_1}\exp\left\{-Cn\left(\frac{\Delta}{D^{2\kappa_1 + \kappa_2}}\right)\right\}\\
				&\quad + p^{D}\exp\left\{-Cn\left(\frac{\Delta}{D^{\kappa_5}}\right)^2\right\} \\
				&\rightarrow 0.
	\end{align}
	Besides, it holds
	\begin{align}
		&\tp\left(\textup{RIC}_n(S^*) \geq \inf_{\substack{S: S \supseteq S^* \\ |S| \leq D}}\textup{RIC}_n(S)\right) \\
		&\leq \tp\left(\ric(S^*) + c_n\deg(S^*) + c_n/3 \geq \ric(S^*) + c_n(\deg(S^*)+1) - c_n/3\right) \\
		&\quad + \tp\left(\sup_{S: |S| \leq D}|\widehat{\ric}(S) - \ric(S)| > c_n/3 \right)\\
		&\lesssim   p^{D}D^{\kappa_1}\exp\left\{-Cn\left(\frac{c_n}{D^{\kappa_1}(D^{\kappa_3} + D^{\kappa_4})}\right)^2\right\}+ p^{D}D^{\kappa_1}\exp\left\{-Cn\left(\frac{c_n}{D^{2\kappa_1 + \kappa_2}}\right)\right\}\\
		&\quad + p^{D}\exp\left\{-Cn\left(\frac{c_n}{D^{\kappa_5}}\right)^2\right\} \\
		&\rightarrow 0.
	\end{align}
	Therefore we have
	\begin{align}
		&\tp\left(\textup{RIC}_n(S^*) \neq \inf_{S: |S| \leq D}\textup{RIC}_n(S)\right) \\
		&\leq \tp\left(\textup{RIC}_n(S^*) \geq \inf_{\substack{S: S \not\supseteq S^* \\ |S| \leq D}}\textup{RIC}_n(S)\right) \\
		&\quad +\tp\left(\textup{RIC}_n(S^*) \geq \inf_{\substack{S: S \supseteq S^* \\ |S| \leq D}}\textup{RIC}_n(S)\right) \\
		&\lesssim   p^{D}D^{\kappa_1}\exp\left\{-Cn\left(\frac{c_n}{D^{\kappa_1}(D^{\kappa_3} + D^{\kappa_4})}\right)^2\right\}+ p^{D}D^{\kappa_1}\exp\left\{-Cn\left(\frac{c_n}{D^{2\kappa_1 + \kappa_2}}\right)\right\}\\
				&\quad + p^{D}\exp\left\{-Cn\left(\frac{c_n}{D^{\kappa_5}}\right)^2\right\} \\
				&\rightarrow 0,
	\end{align}
	which is because $c_n \ll \Delta$. This completes the proof.
\end{enumerate}

\subsection{Proof of Theorem \ref{thm: lda consistency}}
\begin{lemma}\label{lem: lda consis 1}
	For arbitrary $\epsilon \in (0, m^{-1})$, we have the following conclusions:
	\begin{enumerate}[label=(\roman*)]
		\item $\tp(\infnorm{(\hmur{1} - \hmur{0}) - (\bmur{1} - \bmur{0})} > \epsilon) \lesssim p\exp\{-Cn\epsilon^2\}, r = 0, 1$;
		\item $\tp\left(\sup\limits_{\substack{S_1: |S_1| \leq D \\ S_2: |S_2| \leq D}}\infnorm{\hsig_{S_1, S_2} - \Sigma_{S_1, S_2}} > \epsilon\right) \lesssim p^2\exp\left\{-Cn\cdot \left(\frac{\epsilon}{D}\right)^2\right\} + p\exp\left\{-Cn\cdot \frac{\epsilon}{D}\right\}$;
		\item $\tp\left(\sup\limits_{S:|S| \leq D}\twonorm{\hsig_{S, S} - \Sigma_{S, S}} > \epsilon\right) \lesssim p^2\exp\left\{-Cn\cdot \left(\frac{\epsilon}{D}\right)^2\right\} + p\exp\left\{-Cn\cdot \frac{\epsilon}{D}\right\}$;
		\item $\tp\left(\sup\limits_{S:|S| \leq D}\twonorm{\hinvsig_{S, S} - \invsig_{S, S}} > \epsilon\right) \lesssim p^2\exp\left\{-Cn\cdot \left(\frac{\epsilon}{D}\right)^2\right\} + p\exp\left\{-Cn\cdot \frac{\epsilon}{D}\right\}$.
	\end{enumerate}
\end{lemma}

\begin{proof}[Proof of Lemma \ref{lem: lda consis 1}]
For (\rom{1}), because $\maxnorm{\Sigma} \leq M$, for any $j = 1, \ldots, p$ and $r = 0, 1$, $\hat{\mu}^{(r)}_j - \mu^{(r)}_j$ is a $\sqrt{M}$-subGaussian variable. By the tail bound and union bound, we have
\begin{equation}
	\tp(\infnorm{\hmur{r}  - \bmur{r}} > \epsilon) \leq \sum_{j=1}^p\tp(\norm{\hat{\mu}^{(r)}_j - \mu^{(r)}_j} > \epsilon) \lesssim p\exp\{-Cn\epsilon^2\}, r = 0, 1,
\end{equation}
which leads to (\rom{1}). Denote $\Sigma = (\sigma_{ij})_{p \times p}, \hsig = (\hat{\sigma}_{ij})_{p \times p}$. To show (\rom{2}), similar to \cite{bickel2008covariance}, we have
	\begin{equation}
		\tp\left(\max_{i, j} |\hat{\sigma}_{ij} - \sigma_{ij}| > \epsilon\right) \lesssim p^2\exp\{-Cn\epsilon^2\} + p\exp\{-Cn\epsilon\}.
	\end{equation}
	And it yields
	\begin{align}
		\tp\left(\sup\limits_{\substack{S_1: |S_1| \leq D \\ S_2: |S_2| \leq D}} \infnorm{\hat{\Sigma}_{S_1, S_2} - \Sigma_{S_1, S_2}} > \epsilon\right) &= \tp\left(\sup\limits_{\substack{S_1: |S_1| \leq D \\ S_2: |S_2| \leq D}} \sup_{i\in S_1}\sum_{j \in S_2}|\hat{\sigma}_{ij} - \sigma_{ij}| > \epsilon \right) \\
		&\leq \tp\left(D\cdot \max_{i, j} |\hat{\sigma}_{ij} - \sigma_{ij}| > \epsilon \right) \\
		&\lesssim p^2\exp\left\{-Cn\cdot \left(\frac{\epsilon}{D}\right)^2\right\} + p\exp\left\{-Cn\cdot \frac{\epsilon}{D}\right\}.
	\end{align}
	Since $\hat{\Sigma}_{S, S} - \Sigma_{S, S}$ is symmetric, we have $\twonorm{\hat{\Sigma}_{S, S} - \Sigma_{S, S}} \leq \infnorm{\hat{\Sigma}_{S, S} - \Sigma_{S, S}}$ (\cite{bickel2008covariance}). For (\rom{4}), firstly because the operator norm is sub-multiplicative, we have
	\begin{align}
		\twonorm{\hat{\Sigma}^{-1}_{S, S} - \Sigma^{-1}_{S, S}} &= \twonorm{\hat{\Sigma}^{-1}_{S, S}(\hat{\Sigma}_{S, S} - \Sigma_{S, S})\Sigma_{S, S}^{-1}} \\
		&\leq \twonorm{\hat{\Sigma}^{-1}_{S, S}} \cdot \twonorm{\hat{\Sigma}_{S, S} - \Sigma_{S, S}} \cdot \twonorm{\Sigma_{S, S}^{-1}} \\
		&\leq (\twonorm{\hat{\Sigma}^{-1}_{S, S} - \Sigma^{-1}_{S, S}} + \twonorm{\Sigma_{S, S}^{-1}})\cdot \twonorm{\hat{\Sigma}_{S, S} - \Sigma_{S, S}} \cdot \twonorm{\Sigma_{S, S}^{-1}},
	\end{align}
	leading to
	\begin{equation}\label{eq: inv ineq}
		\twonorm{\hat{\Sigma}^{-1}_{S, S} - \Sigma^{-1}_{S, S}} \leq \frac{\twonorm{\Sigma_{S, S}^{-1}}^2 \cdot \twonorm{\hat{\Sigma}_{S, S} - \Sigma_{S, S}}}{1 - \twonorm{\hat{\Sigma}_{S, S} - \Sigma_{S, S}} \cdot \twonorm{\Sigma_{S, S}^{-1}}} \leq \frac{\frac{1}{m^2}\twonorm{\hat{\Sigma}_{S, S} - \Sigma_{S, S}}}{1 - \frac{1}{m}\twonorm{\hat{\Sigma}_{S, S} - \Sigma_{S, S}}}.
	\end{equation}
	Then we obtain that
	\begin{align}
		\tp\left(\sup_{S: |S| \leq D}\twonorm{\hat{\Sigma}^{-1}_{S, S} - \Sigma^{-1}_{S, S}} > \epsilon\right) &\leq \tp\left(\twonorm{\hat{\Sigma}_{S, S} - \Sigma_{S, S}} > \frac{1}{2}m\right)\\
		&\quad + \tp\left(\frac{2}{m^2}\twonorm{\hat{\Sigma}_{S, S} - \Sigma_{S, S}} > \epsilon\right) \\
		&\leq 2\tp\left(\supp\twonorm{\hat{\Sigma}_{S, S} - \Sigma_{S, S}} > \frac{m^2}{2}\epsilon\right).
	\end{align}
	Then by applying (\rom{3}), we get (\rom{4}) immediately.
\end{proof}

\begin{lemma}\label{lem: lda consis 2}
	Define $\bdelta_{S} = \Sigma_{S, S}^{-1}(\bmur{1}_{S} - \bmur{0}_{S})$ for any subset $S$. For $\forall \tilde{S} = S \cup S^*$, where $S \cap S^* = \emptyset$, we have 	
	\begin{equation}
	\bdelta_{\tilde{S}} = \Sigma_{\tilde{S}, \tilde{S}}^{-1}(\bmur{1}_{\tilde{S}} - \bmur{0}_{\tilde{S}}) = 
	\Sigma_{\tilde{S}, \tilde{S}}^{-1}\begin{pmatrix}
		\bmusr{1} - \bmusr{0} \\ \bmur{1}_{S^*} - \bmur{0}_{S^*}
	\end{pmatrix} = \begin{pmatrix}
		\bm{0} \\ \bm{\delta}_{S^*}
	\end{pmatrix}.
\end{equation}
\end{lemma}

\begin{proof}[Proof of Lemma \ref{lem: lda consis 2}]
	First, when $\tilde{S} = S_{\textup{Full}}$, we know that
	\begin{equation}
		\bdelta_{\tilde{S}} = \Sigma_{\tilde{S}, \tilde{S}}^{-1}(\bmur{1}_{\tilde{S}} - \bmur{0}_{\tilde{S}}) = 
	\Sigma_{\tilde{S}, \tilde{S}}^{-1}\begin{pmatrix}
		\bmusr{1} - \bmusr{0} \\ \bmur{1}_{S^*} - \bmur{0}_{S^*}
	\end{pmatrix} = \begin{pmatrix}
		\bm{0} \\ *
	\end{pmatrix}.
	\end{equation}
	Then combined with the matrix decomposition
	\[\Sigma_{\tilde{S}, \tilde{S}}^{-1} = \begin{pmatrix}
	 (\Sigma_{S, S} - \Sigma_{S, S^*}\Sigma_{S^*, S^*}^{-1}\Sigma_{S^*, S})^{-1} & -\Sigma_{S, S}^{-1}\Sigma_{S, S^*}(\Sigma_{S^*, S^*} - \Sigma_{S^*, S}\Sigma_{S, S}^{-1}\Sigma_{S, S^*})^{-1} \\ -\Sigma_{S^*, S^*}^{-1}\Sigma_{S^*, S}(\Sigma_{S, S} - \Sigma_{S, S^*}\Sigma_{S^*, S^*}^{-1}\Sigma_{S^*, S})^{-1} & (\Sigma_{S^*, S^*} - \Sigma_{S^*, S}\Sigma_{S, S}^{-1}\Sigma_{S, S^*})^{-1}\\
\end{pmatrix},\]
it can be noticed that
\[
	(\Sigma_{S, S} - \Sigma_{S, S^*}\Sigma_{S^*, S^*}^{-1}\Sigma_{S^*, S})^{-1}(\bmusr{1} - \bmusr{0}) = \Sigma_{S, S}^{-1}\Sigma_{S, S^*}(\Sigma_{S^*, S^*} - \Sigma_{S^*, S}\Sigma_{S, S}^{-1}\Sigma_{S, S^*})^{-1}(\bmur{1}_{S^*} - \bmur{0}_{S^*}).
\]
Also since there holds that
\begin{equation}
	(\Sigma_{S^*, S^*} - \Sigma_{S^*, S}\Sigma_{S, S}^{-1}\Sigma_{S, S^*})^{-1} = \Sigma_{S^*, S^*}^{-1} - \Sigma_{S^*, S^*}^{-1}\Sigma_{S^*, S}(-\Sigma_{S, S} + \Sigma_{S, S^*}\Sigma_{S^*, S^*}^{-1}\Sigma_{S^*, S})^{-1}\Sigma_{S, S^*}\Sigma_{S^*, S^*}^{-1}\nonumber,
\end{equation}
it can be easily verified that the ``*" part is actually $\bdelta_{S^*}$. Therefore the conclusion holds with $\tilde{S} = S_{\textup{Full}}$. 

For general $\tilde{S} \supseteq S^*$, notice that $S_{\textup{Full}} = (S_{\textup{Full}}\backslash\tilde{S}) \cup \tilde{S}$, with the same procedure, we can obtain that
\[
\bdelta_{S_{\textup{Full}}} = \Sigma^{-1}(\bmur{1} - \bmur{0}) = 
 \begin{pmatrix}
		\bm{0} \\ \bdelta_{\tilde{S}}
	\end{pmatrix}.
\]
And since\[
\bdelta_{S_{\textup{Full}}} = \begin{pmatrix}
		\bm{0} \\ \bdelta_{S^*}
	\end{pmatrix},
\]we reach the conclusion.
\end{proof}

\begin{lemma}\label{lem: lda consis 3}
	For $\forall S$ which satisfies $S \not\supseteq S^*$, let $\tilde{S}^* = S^* \backslash S$, then
	\begin{equation}
		\twonorma{\Sigma_{\tilde{S}^*, S}\invsig_{S, S}\left(\bmusr{1} - \bmusr{0}\right) - \left(\bmur{1}_{S^*} - \bmur{0}_{S^*}\right)} \geq\gamma m.
	\end{equation} 
\end{lemma}

\begin{proof}[Proof of Lemma \ref{lem: lda consis 3}]
Let $\tilde{S} = S \cup \tilde{S}^*$, due to Lemma \ref{lem: lda consis 2}, it's easy to see that there holds
\begin{equation}\label{eq129}
	\Sigma_{\tilde{S}, \tilde{S}}^{-1}\begin{pmatrix}
		\bmusr{1} - \bmusr{0} \\ \bmur{1}_{\tilde{S}^*} - \bmur{0}_{\tilde{S}^*}
	\end{pmatrix} = \begin{pmatrix}
		* \\ \bdelta'_{\tilde{S}^*}
	\end{pmatrix},
\end{equation}
where 
\begin{equation}\label{eq: inv}
	\Sigma_{\tilde{S}, \tilde{S}}^{-1} = \begin{pmatrix}
	 (\Sigma_{S, S} - \Sigma_{S, \tilde{S}^*}\Sigma_{\tilde{S}^*, \tilde{S}^*}^{-1}\Sigma_{\tilde{S}^*, S})^{-1} & -\Sigma_{S, S}^{-1}\Sigma_{S, \tilde{S}^*}(\Sigma_{\tilde{S}^*, \tilde{S}^*} - \Sigma_{\tilde{S}^*, S}\Sigma_{S, S}^{-1}\Sigma_{S, \tilde{S}^*})^{-1} \\ -\Sigma_{\tilde{S}^*, \tilde{S}^*}^{-1}\Sigma_{\tilde{S}^*, S}(\Sigma_{S, S} - \Sigma_{S, \tilde{S}^*}\Sigma_{\tilde{S}^*, \tilde{S}^*}^{-1}\Sigma_{\tilde{S}^*, S})^{-1} & (\Sigma_{\tilde{S}^*, \tilde{S}^*} - \Sigma_{\tilde{S}^*, S}\Sigma_{S, S}^{-1}\Sigma_{S, \tilde{S}^*})^{-1}\\
\end{pmatrix},
\end{equation}
and $\bdelta'_{\tilde{S}^*}$ consists of several components of $\bdelta_{S^*}$. Denote $c = \Sigma_{\tilde{S}^*, S}\invsig_{S, S}\left(\bmusr{1} - \bmusr{0}\right) - \left(\bmur{1}_{\tilde{S}^*} - \bmur{0}_{\tilde{S}^*}\right)$.
Since $\Sigma_{\tilde{S}, \tilde{S}}^{-1}$ is symmetric, combining with \eqref{eq129} we have
\begin{equation}
	(\Sigma_{\tilde{S}^*, \tilde{S}^*} - \Sigma_{\tilde{S}^*, S}\Sigma_{S, S}^{-1}\Sigma_{S, \tilde{S}^*})^{-1}c = -\bm{\delta}'_{\tilde{S}^*},
\end{equation}
leading to
\begin{equation}
	\twonorm{(\Sigma_{\tilde{S}^*, \tilde{S}^*} - \Sigma_{\tilde{S}^*, S}\Sigma_{S, S}^{-1}\Sigma_{S, \tilde{S}^*})^{-1}c} = \twonorm{\bm{\delta}'_{\tilde{S}^*}} \geq \gamma, \nonumber
\end{equation}
by Assumption \ref{asmp: lda}.(\rom{3}). For the left-hand side:
\begin{align}
	&\twonorm{(\Sigma_{\tilde{S}^*, \tilde{S}^*} - \Sigma_{\tilde{S}^*, S}\Sigma_{S, S}^{-1}\Sigma_{S, \tilde{S}^*})^{-1}c} \\
	&\leq \twonorm{c}\cdot \twonorm{(\Sigma_{\tilde{S}^*, \tilde{S}^*} - \Sigma_{\tilde{S}^*, S}\Sigma_{S, S}^{-1}\Sigma_{S, \tilde{S}^*})^{-1}}  \\
	&\leq \twonorm{c}\cdot \twonorm{\Sigma_{\tilde{S}, \tilde{S}}^{-1}} \\
	&\leq \twonorm{c}\cdot \lambda^{-1}_{\min}(\Sigma_{\tilde{S}, \tilde{S}}) \\
	&\leq m^{-1}\twonorm{c},
\end{align}
which leads to $\twonorm{c} \geq \gamma m$.
\end{proof}

\begin{lemma}\label{lem: lda consis 4}
	For LDA, $\ric(S) = -\left(\bmusr{1} - \bmusr{0}\right)^T \invsig_{S, S}\left(\bmusr{1} - \bmusr{0}\right)$. It satisfies the following conclusions:
	\begin{enumerate}[label=(\roman*)]
		\item If $S \supseteq S^*$, then $\ric(S) = \ric(S^*)$;
		\item For $\tilde{S} = \tilde{S}^* \cup S$, we have
		\begin{align}
			\ric(\tilde{S}) &= \ric(S) - \left[\Sigma_{\tilde{S}^*, S}\invsig_{S, S}\left(\bmusr{1} - \bmusr{0}\right) - \left(\bmur{1}_{\tilde{S}^*} - \bmur{0}_{\tilde{S}^*}\right)\right]^T \\
			&\quad \left(\Sigma_{\tilde{S}^*, \tilde{S}^*} - \Sigma_{\tilde{S}^*, S}\invsig_{S, S}\Sigma_{S, \tilde{S}^*}\right)\left[\Sigma_{\tilde{S}^*, S}\invsig_{S, S}\left(\bmusr{1} - \bmusr{0}\right) - \left(\bmur{1}_{\tilde{S}^*} - \bmur{0}_{\tilde{S}^*}\right)\right]
		\end{align}
		\item It holds
		\begin{equation}
			\inf_{\substack{S: |S| \leq D \\ S \not\supseteq S^*}} \ric(S) - \ric(S^*) \geq m^3\gamma^2.
		\end{equation}
	\end{enumerate}
\end{lemma}

\begin{proof}[Proof of Lemma \ref{lem: lda consis 4}]
	\begin{enumerate}[label=(\roman*)]
		\item First let's suppose (\rom{2}) is correct. For $S \supseteq S^*$, consider $S_{\textup{Full}} =  \tilde{S}^* \cup S$ and $\tilde{S}^* \cap S = \emptyset$, then by sparsity condition, we have
		\begin{equation}
			\begin{pmatrix} * \\ \bm{0} \end{pmatrix} = \Sigma^{-1}(\bmu^{(1)} - \bmu^{(0)})
			 =  \Sigma^{-1}\begin{pmatrix} \bmu^{(1)}_S - \bmu^{(0)}_S \\ \bmu^{(1)}_{\tilde{S}^*} - \bmu^{(0)}_{\tilde{S}^*}\end{pmatrix}.
		\end{equation}
		Also by basic algebra, there holds
			\begin{equation}\label{eq: inv decomp}
				(\Sigma_{S, S} - \Sigma_{S, \tilde{S}^*}\Sigma_{\tilde{S}^*, \tilde{S}^*}^{-1}\Sigma_{\tilde{S}^*, S})^{-1} = \Sigma_{S, S}^{-1} - \Sigma_{S, S}^{-1}\Sigma_{S, \tilde{S}^*}(-\Sigma_{\tilde{S}^*, \tilde{S}^*} + \Sigma_{\tilde{S}^*, S}\Sigma_{S, S}^{-1}\Sigma_{S, \tilde{S}^*})^{-1}\Sigma_{\tilde{S}^*, S}\Sigma_{S, S}^{-1}.
				\end{equation}
		Combined with \eqref{eq: inv} and \eqref{eq: inv decomp}, it can be obtained that
		\begin{equation}
			\Sigma_{\tilde{S}^*, S}\invsig_{S, S}\left(\bmusr{1} - \bmusr{0}\right) - \left(\bmur{1}_{\tilde{S}^*} - \bmur{0}_{\tilde{S}^*}\right) = \bm{0},
		\end{equation}
		yielding that $\ric(S_{\textup{Full}}) = \ric(S)$. Since the same procedure can be conducted for arbitrary $S \supseteq S^*$, we complete the proof.
		\item This can be directly calculated and simplified by applying \eqref{eq: inv} and \eqref{eq: inv decomp}.
		\item It's easy to see that
		\begin{align}
			\lambda_{\min}(\Sigma_{\tilde{S}^*, \tilde{S}^*} - \Sigma_{\tilde{S}^*, S}\invsig_{S, S}\Sigma_{S, \tilde{S}^*}) &= \lambda_{\max}^{-1}((\Sigma_{\tilde{S}^*, \tilde{S}^*} - \Sigma_{\tilde{S}^*, S}\invsig_{S, S}\Sigma_{S, \tilde{S}^*})^{-1})\\
			&\geq \lambda_{\max}^{-1}(\invsig_{\tilde{S}, \tilde{S}}) \\
			&= \lambda_{\min}(\Sigma_{\tilde{S}, \tilde{S}})\\
			&\geq m.
		\end{align}
		Then by (\rom{2}) and Lemma \ref{lem: lda consis 3}, it holds that
		\begin{align}
			&\ric(\tilde{S}) - \ric(S) \\
			&\geq \twonorma{\Sigma_{\tilde{S}^*, S}\invsig_{S, S}\left(\bmusr{1} - \bmusr{0}\right) - \left(\bmur{1}_{\tilde{S}^*} - \bmur{0}_{\tilde{S}^*}\right)}^2\cdot \lambda_{\min}(\Sigma_{\tilde{S}^*, \tilde{S}^*} - \Sigma_{\tilde{S}^*, S}\invsig_{S, S}\Sigma_{S, \tilde{S}^*}) \\
			&\geq m^3\gamma^2,
		\end{align}
		which completes the proof.
	\end{enumerate}
\end{proof}

\begin{lemma}\label{lem: lda consis 5}
	If Assumption \ref{asmp: lda} holds, for $\epsilon$ smaller than some constant and $n, p$ larger than some constants, we have
	\begin{equation}
		\tp\left(\sup_{S: |S| \leq D}|\widehat{\ric}(S) - \ric(S)| > \epsilon \right) \lesssim p^2\exp\left\{-Cn\left(\frac{\epsilon}{D^2}\right)^2\right\}.
	\end{equation}
\end{lemma}

\begin{proof}[Proof of Lemma \ref{lem: lda consis 5}]
	By Lemma \ref{lem: lda consis 1}, when $\epsilon < m^{-1}$, there holds that
	\begin{align}
		&\tp\left(\sup_{S: |S| \leq D}|\widehat{\ric}(S) - \ric(S)| > \epsilon \right) \\
		&= \tp\left(\sup_{S: |S| \leq D}\norma{\left(\hmusr{1} - \hmusr{0}\right)^T \hinvsig_{S, S}\left(\hmusr{1} - \hmusr{0}\right) - \left(\bmusr{1} - \bmusr{0}\right)^T \invsig_{S, S}\left(\bmusr{1} - \bmusr{0}\right)} > \epsilon \right) \\
		&\leq \tp\left(\sup_{S: |S| \leq D}\norma{\left(\hmusr{1} - \hmusr{0}\right)^T \left(\hinvsig_{S, S} - \invsig_{S, S}\right)\left(\hmusr{1} - \hmusr{0}\right)}\right.\\
		&\quad + \norma{\left(\hmusr{1} - \hmusr{0} - \bmusr{1} + \bmusr{0}\right)^T \invsig_{S, S}\left(\hmusr{1} - \hmusr{0}\right)} \\
		&\quad+ \norma{\left(\hmusr{1} - \hmusr{0} - \bmusr{1} + \bmusr{0}\right)^T \invsig_{S, S}\left(\bmusr{1} - \bmusr{0}\right)} > \epsilon \Bigg)\\
		&\leq \tp\left(\supp \left[\twonorma{\hmusr{1} - \hmusr{0}}^2\twonorma{\hinvsig_{S, S} - \invsig_{S, S}} + \left(\twonorma{(\hmusr{1} - \hmusr{0}) -(\bmusr{1} - \bmusr{0})}\right)\right.\right.\\
		&\quad \cdot \twonorma{\invsig_{S, S}}\cdot \left(\twonorma{\hmusr{1} - \hmusr{0}} + \twonorma{\bmusr{1} - \bmusr{0}}\right)\bigg] > \epsilon \Bigg)\\
		&\leq \tp\left((3M')^2D\cdot \supp\twonorma{\hinvsig_{S, S} - \invsig_{S, S}} > \frac{\epsilon}{3}\right) \\
		&\quad + \tp\left(\supp\infnorma{(\hmusr{1} - \hmusr{0}) -(\bmusr{1} - \bmusr{0})} > \frac{M'}{2}\right) \\
		&\quad + \tp\left(\supp\twonorma{(\hmusr{1} - \hmusr{0}) -(\bmusr{1} - \bmusr{0})}\cdot m^{-1}\left(\frac{3}{2}M' + M'\right)\sqrt{D}> \frac{\epsilon}{3}\right)\\
		&\lesssim p^2\exp\left\{-Cn\left(\frac{\epsilon}{D^2}\right)^2\right\} +p\exp\left\{-Cn\cdot \frac{\epsilon}{D}\right\} + p\exp\left\{-Cn\right\} + p\exp\left\{-Cn\left(\frac{\epsilon}{D}\right)^2\right\} \\
		&\lesssim p^2\exp\left\{-Cn\left(\frac{\epsilon}{D^2}\right)^2\right\}.
	\end{align}
\end{proof}

Then the following steps to prove Theorem \ref{thm: lda consistency} are analogous to what we did to prove Theorem \ref{thm: consistency}.

\subsection{Proof of Theorem \ref{thm: qda consistency}}
Denote $\homegasr{r}{S, S} = (\hsigsr{r}{S, S})^{-1}, r = 0, 1$. Then we denote
\begin{align}
	T(S) &= \tr\left[(\omegasr{1}{S, S} - \omegasr{0}{S, S})(\pi_1\Sigma_{S, S}^{(1)} - \pi_0\Sigma_{S, S}^{(0)})\right] + (\pi_1 - \pi_0)(\log|\Sigma_{S, S}^{(1)}| - \log|\Sigma_{S, S}^{(0)}|), \\
	D(S) &= (\bm{\mu}_{S}^{(1)} - \bm{\mu}_{S}^{(0)})^T\left[\pi_1\omegasr{0}{S, S} + \pi_0\omegasr{1}{S, S}\right](\bm{\mu}_{S}^{(1)} - \bm{\mu}_{S}^{(0)}), \\
	\ric(S) &= 2[T(S) - D(S)]. 
\end{align}
And their sample versions by MLEs are 
\begin{align}
	\hat{T}(S) &= \tr\left[(\homegasr{1}{S, S} - \homegasr{0}{S, S})(\hat{\pi}_1\hat{\Sigma}_{S, S}^{(1)} - \hat{\pi}_0\hat{\Sigma}_{S, S}^{(0)})\right] + (\hat{\pi}_1 - \hat{\pi}_0)(\log|\hat{\Sigma}_{S, S}^{(1)}| - \log|\hat{\Sigma}_{S, S}^{(0)}|),\\
	\hat{D}(S) &= (\hat{\bm{\mu}}_{S}^{(1)} - \hat{\bm{\mu}}_{S}^{(0)})^T\left[\hat{\pi}_1\omegasr{0}{S, S} + \hpi_0\omegasr{1}{S, S}\right](\hat{\bm{\mu}}_{S}^{(1)} - \hat{\bm{\mu}}_{S}^{(0)}), \\
	 \widehat{\ric}(S) &= 2[\hat{T}(S) - \hat{D}(S)]. 
\end{align}

Similar to Lemma \ref{lem: lda consis 1},  we have the following lemma holds.
\begin{lemma}\label{lem: qda consis 1}
	For arbitrary $\epsilon \in (0, m^{-1})$, we have conclusions in the follows for $r = 0, 1$:
	\begin{enumerate}[label=(\roman*)]
		\item $\tp(\infnorm{(\hmur{1} - \hmur{0}) - (\bmur{1} - \bmur{0})} > \epsilon) \lesssim p\exp\{-Cn\epsilon^2\}$;
		\item $\tp\left(\sup\limits_{S:|S| \leq D}\twonorma{\hsigr{r}_{S, S} - \sigr{r}_{S, S}} > \epsilon\right) \lesssim p^2\exp\left\{-Cn\cdot \left(\frac{\epsilon}{D}\right)^2\right\} + p\exp\left\{-Cn\cdot \frac{\epsilon}{D}\right\}$;
		\item $\tp\left(\sup\limits_{S:|S| \leq D}\twonorma{\homegasr{r}{S, S} - \omegasr{r}{S, S}} > \epsilon\right) \lesssim p^2\exp\left\{-Cn\cdot \left(\frac{\epsilon}{D}\right)^2\right\} + p\exp\left\{-Cn\cdot \frac{\epsilon}{D}\right\}$.
	\end{enumerate}
\end{lemma}

Also, the lemmas in the follows are useful to prove Theorem \ref{thm: qda consistency} as well.
\begin{lemma}\label{lem: qda consis 2}
	There hold the following conditions:
	\begin{enumerate}[label=(\roman*)]
		\item For $\tilde{S} = S \cup \{j\}, k_{j, S}^{(r)} = \Sigma_{j, j}^{(r)} - \Sigma_{j, S}^{(r)}(\Sigma_{S, S}^{(r)})^{-1}\Sigma_{S, j}^{(r)}, r = 0,1$, where $j \notin S$, we have:
	\begin{align}
		T(\tilde{S}) - T(S) &= -\left[\omegasr{1}{S, S}\Sigma_{S, j}^{(1)} - \omegasr{0}{S, S}\Sigma_{S, j}^{(0)}\right]^T\left(\frac{\pi_0}{k_{j, S}^{(0)}}\Sigma_{S, S}^{(0)} + \frac{\pi_1}{k_{j, S}^{(1)}}\Sigma_{S, S}^{(1)}\right) \\
		&\quad \left[\omegasr{1}{S, S}\Sigma_{S, j}^{(1)} - \omegasr{0}{S, S}\Sigma_{S, j}^{(0)}\right] + 1 - \frac{\pi_0}{k_{j, S}^{(1)}}\cdot k_{j, S}^{(0)} - \frac{\pi_1}{k_{j, S}^{(0)}}\cdot k_{j, S}^{(1)} \\
		&\quad + (\pi_1 - \pi_0)\log\left(\frac{k_{j, S}^{(1)}}{k_{j, S}^{(0)}}\right).\label{T}
	\end{align}
	And for $\tilde{S} = S \cup \tilde{S}^*$ where $S \cap \tilde{S}^* = \emptyset$, there holds
	\begin{align}
		D(\tilde{S}) - D(S) &= \pi_1\left[\Sigma_{\tilde{S}^*, S}^{(0)}\omegasr{0}{S, S}(\bmu_{S}^{(1)} - \bmu_{S}^{(0)}) - (\bmu_{\tilde{S}^*}^{(1)} - \bmu_{\tilde{S}^*}^{(0)})\right]^T \\
		&\quad \left(\sigr{0}_{\tilde{S}^*, \tilde{S}^*} - \sigr{0}_{\tilde{S}^*, S}\omegasr{0}{S, S}\sigr{0}_{S, \tilde{S}^*}\right)\left[\Sigma_{\tilde{S}^*, S}^{(0)}\omegasr{0}{S, S}(\bmu_{S}^{(1)} - \bmu_{S}^{(0)}) - (\bmu_{\tilde{S}^*}^{(1)} - \bmu_{\tilde{S}^*}^{(0)})\right] \\
		&\quad + \pi_0\left[\Sigma_{\tilde{S}^*, S}^{(1)}\omegasr{1}{S, S}(\bmu_{S}^{(1)} - \bmu_{S}^{(0)}) - (\bmu_{\tilde{S}^*}^{(1)} - \bmu_{\tilde{S}^*}^{(0)})\right]^T \\
		&\quad \left(\sigr{1}_{\tilde{S}^*, \tilde{S}^*} - \sigr{1}_{\tilde{S}^*, S}\omegasr{1}{S, S}\sigr{1}_{S, \tilde{S}^*}\right)\left[\Sigma_{\tilde{S}^*, S}^{(1)}\omegasr{1}{S, S}(\bmu_{S}^{(1)} - \bmu_{S}^{(0)}) - (\bmu_{\tilde{S}^*}^{(1)} - \bmu_{\tilde{S}^*}^{(0)})\right]. \label{D}
	\end{align}
	\item Further, (\rom{1}) implies the monotonicity of $T, D, \ric$ in the following sense: If $S_1 \supseteq S_2$, then $T(S_1) \leq T(S_2), D(S_1) \geq D(S_2)$, which leads to $\ric(S_1) \leq \ric(S_2)$.
	\item Define $S^*_l = \{j: [(\Sigma^{(0)})^{-1}\bm{\mu}^{(0)} - (\Sigma^{(1)})^{-1}\bm{\mu}^{(1)}]_j \neq 0\}$, $S^*_q = \{j: [(\Sigma^{(1)})^{-1} - (\Sigma^{(0)})^{-1}]_{ij} \neq 0, \exists i\}$, then:
	\begin{enumerate}
	\item If $S \supseteq S_q^*$, then $T(S) = T(S_q^*)$;
	\item If $S \supseteq S^*$, then $D(S) = D(S^*) = D(S_l^*)$.
	\end{enumerate}
	\end{enumerate}
		
\end{lemma}

\begin{proof}[Proof of Lemma \ref{lem: qda consis 2}]
\begin{enumerate}[label=(\roman*)]
	\item \eqref{D} is obvious due to Lemma \ref{lem: lda consis 4}. Now let's prove \eqref{T}. It's easy to see that
	\begin{align}
		T(\tilde{S}) &= \tr\left[(\omegasr{1}{\tilde{S}, \tilde{S}} - \omegasr{0}{\tilde{S}, \tilde{S}})(\pi_1\Sigma_{\tilde{S}, \tilde{S}}^{(1)} - \pi_0\Sigma_{\tilde{S}, \tilde{S}}^{(0)})\right] + (\pi_1 - \pi_0)(\log|\Sigma_{\tilde{S}, \tilde{S}}^{(1)}| - \log|\Sigma_{\tilde{S}, \tilde{S}}^{(0)}|) \\
		&= |S| - \pi_1 \tr\left[\omegasr{0}{\tilde{S}, \tilde{S}}\Sigma_{\tilde{S}, \tilde{S}}^{(1)}\right] - \pi_0 \tr\left[\omegasr{1}{\tilde{S}, \tilde{S}}\Sigma_{\tilde{S}, \tilde{S}}^{(0)}\right] + (\pi_1 - \pi_0)(\log|\Sigma_{\tilde{S}, \tilde{S}}^{(1)}| - \log|\Sigma_{\tilde{S}, \tilde{S}}^{(0)}|).\label{eq1}
	\end{align}
	Because
	\begin{equation}\label{eq: sigma inv}
		\omegasr{r}{\tilde{S}, \tilde{S}} = \begin{pmatrix}\omegasr{r}{S, S} + \frac{1}{k_{j, S}^{(r)}} \omegasr{r}{S, S}\Sigma_{S, j}^{(r)}\Sigma_{j, S}^{(r)}\omegasr{r}{S, S} & -\frac{1}{k_{j, S}^{(r)}} \omegasr{r}{S, S}\Sigma_{S, j}^{(r)} \\ -\frac{1}{k_{j, S}^{(r)}}\Sigma_{j, S}^{(r)}\omegasr{r}{S, S} & \frac{1}{k_{j, S}^{(r)}} \end{pmatrix},
	\end{equation}
	for $r = 0,1$, we have
	\begin{align}
		&\tr\left(\omegasr{1}{\tilde{S}, \tilde{S}}\Sigma_{\tilde{S}, \tilde{S}}^{(0)}\right) \\
		&= \tr\left[\omegasr{1}{S, S}\Sigma_{S, S}^{(0)} + \frac{1}{k_{j, S}^{(1)}}\omegasr{1}{S, S}\Sigma_{S, j}^{(1)}\Sigma_{j, S}^{(1)}\omegasr{1}{S, S}\Sigma_{S, S}^{(0)} - \frac{1}{k_{j, S}^{(1)}}\omegasr{1}{S, S}\Sigma_{S, j}^{(1)} \Sigma_{j, S}^{(0)}\right] \\
		&\quad -\frac{1}{k_{j, S}^{(1)}}\Sigma_{j, S}^{(1)}\omegasr{1}{S, S} \Sigma_{S, j}^{(0)} + \frac{\Sigma^{(0)}_{j,j}}{k^{(1)}_{j, S}} \\
		&= \tr[\omegasr{1}{S, S}\Sigma_{S, S}^{(0)}] + \frac{1}{k_{j, S}^{(1)}}\Sigma_{j, S}^{(1)}\omegasr{1}{S, S}\Sigma_{S, S}^{(0)}\omegasr{1}{S, S}\Sigma_{S, j}^{(1)} - \frac{2}{k_{j, S}^{(1)}}\Sigma_{j, S}^{(1)}\omegasr{1}{S, S} \Sigma_{S, j}^{(0)} + \frac{\Sigma^{(0)}_{j,j}}{k^{(1)}_{j, S}},\label{eq2}
	\end{align}
	where the last equality follows from the fact that $\tr(AB) = \tr(BA)$ if $A$ and $B$ are square matrices with the same dimension.
	
	Similarly we have
	\begin{equation}
		\tr\left(\omegasr{0}{\tilde{S}, \tilde{S}}\Sigma_{\tilde{S}, \tilde{S}}^{(1)}\right) = \tr[\omegasr{0}{S, S}\Sigma_{S, S}^{(1)}] + \frac{1}{k_{j, S}^{(0)}}\Sigma_{j, S}^{(0)}\omegasr{0}{S, S}\Sigma_{S, S}^{(1)}\omegasr{0}{S, S}\Sigma_{S, j}^{(0)} - \frac{2}{k_{j, S}^{(0)}}\Sigma_{j, S}^{(0)}\omegasr{0}{S, S} \Sigma_{S, j}^{(1)} + \frac{\Sigma^{(1)}_{j,j}}{k^{(0)}_{j, S}}. \label{eq3}
	\end{equation}
	Combining \eqref{eq1}, \eqref{eq2}, \eqref{eq3}, the fact that $\Sigma^{(r)}_{j,j} = k_{j, S}^{(r)} + \Sigma^{(r)}_{j,S} \omegasr{r}{S, S}\Sigma^{(r)}_{S, j}, |\Sigma_{\tilde{S}, \tilde{S}}^{(r)}| = |\Sigma_{S, S}^{(r)}|\cdot|k_{j, S}^{(r)}|, r = 0,1$ and $T(S) = |S| - \pi_1\tr\left(\omegasr{0}{S, S}\Sigma_{S, S}^{(1)}\right) - \pi_0\tr\left(\omegasr{1}{S, S}\Sigma_{S, S}^{(0)}\right) + (\pi_1 - \pi_0)(\log|\Sigma_{S, S}^{(1)}| - \log|\Sigma_{S, S}^{(0)}|)$, \eqref{T} is obtained.
	
	\item For the monotonicity, since $-\pi_0\log\left(\frac{k_{j, S}^{(1)}}{k_{j, S}^{(0)}}\right) = \pi_0\log\left(\frac{k_{j, S}^{(0)}}{k_{j, S}^{(1)}}\right) \leq \pi_0\left(\frac{k_{j, S}^{(0)}}{k_{j, S}^{(1)}} - 1\right)$ and $ \pi_1\log\left(\frac{k_{j, S}^{(1)}}{k_{j, S}^{(0)}}\right) \leq \pi_1\left(\frac{k_{j, S}^{(1)}}{k_{j, S}^{(0)}} - 1\right)$, we have
			\begin{align}
				&1 + (\pi_1 - \pi_0)\log\left(\frac{k_{j, S}^{(1)}}{k_{j, S}^{(0)}}\right) - \frac{\pi_0}{k_{j, S}^{(1)}}\cdot k_{j, S}^{(0)} - \frac{\pi_1}{k_{j, S}^{(0)}}\cdot k_{j, S}^{(1)} \nonumber\\
				&\leq 1 + \pi_0\left(\frac{k_{j, S}^{(0)}}{k_{j, S}^{(1)}} - 1\right) + \pi_1\left(\frac{k_{j, S}^{(1)}}{k_{j, S}^{(0)}} - 1\right) - \frac{\pi_0}{k_{j, S}^{(1)}}\cdot k_{j, S}^{(0)} - \frac{\pi_1}{k_{j, S}^{(0)}}\cdot k_{j, S}^{(1)} \nonumber\\
				&\leq 0,
			\end{align}
			implying that 
			\begin{equation}
				T(\tilde{S}) \leq T(S).
			\end{equation}
	And it's easy to see that $\sigr{r}_{\tilde{S}^*, \tilde{S}^*} - \sigr{r}_{\tilde{S}^*, S}\omegasr{r}{S, S}\sigr{r}_{S, \tilde{S}^*}$ is positive-definite, thus \[
		D(\tilde{S}) \geq D(S).
	\]And we also have\[
	\ric(\tilde{S}) \leq \ric(S).
	\]
	By induction we will obtain the monotonicity.
	\item Denote $\omegar{r} = (\sigr{r})^{-1}, r = 0, 1$. Consider full feature space $S_{\textrm{Full}} \supseteq S^*$. Let's remove one feature which does not belong to $S^*_q$ from $S_{\textrm{Full}}$. Again, without loss of generality (in fact, we can always switch this feature with the last one), suppose we are removing the last feature $j$. That is, $S_{\textup{Full}} = S \cup \{j\}$.
	
		By sparsity:
		\begin{equation}\label{eq: sparse q}
			\omegasr{1}{} - \omegasr{0}{} = \begin{pmatrix} * & \bm{0} \\ \bm{0}^T & 0 \end{pmatrix}.
		\end{equation}
		Plugging \eqref{eq: sigma inv} into above and by simplification we can obtain that
		\begin{align}
			k^{(0)}_{j, S} &= k^{(1)}_{j, S}, \label{C1}\\
			\omegasr{0}{S, S}\Sigma^{(0)}_{S, j} &= \omegasr{1}{S, S}\Sigma^{(1)}_{S, j}, \label{C2}
		\end{align}
		By Lemma \ref{lem: qda consis 2}, $T(S_{\textrm{Full}}) = T(S)$. Besides, this also implies that
		\begin{equation}
			\omegasr{1}{} - \omegasr{0}{} = \begin{pmatrix} \omegasr{1}{S, S} - \omegasr{0}{S, S} & \bm{0} \\ \bm{0}^T & 0 \end{pmatrix}.
		\end{equation}
		
		By induction, it can be seen that for all subspace $S \supseteq S^*_q, T(S) = T(S^*_q)$.
		
		Then again consider $S_{\textup{Full}} = S \cup \{j\}$ but $j \notin S^*$, therefore by sparsity, in addition to \eqref{eq: sparse q}, we also have
		\begin{equation}
			\omegasr{1}{}\bmu^{(1)} - \omegasr{0}{}\bmu^{(0)} = \begin{pmatrix} * \\ 0 \end{pmatrix},
		\end{equation}
		which together with \eqref{eq: sparse q} leads to
		\begin{align}
			\Sigma^{(0)}_{j, S}\omegasr{0}{S, S}(\bmu^{(1)}_{S} - \bmu^{(0)}_{S}) - (\bmu^{(1)}_j - \bmu^{(0)}_j) &= 0, \label{C3}\\
			\Sigma^{(1)}_{j, S}\omegasr{1}{S, S}(\bmu^{(1)}_{S} - \bmu^{(0)}_{S}) - (\bmu^{(1)}_j - \bmu^{(0)}_j) &= 0. \label{C4}
		\end{align}
		By Lemma \ref{lem: qda consis 2}, $D(S_{\textrm{Full}}) = D(S)$. And it holds that
		\begin{equation}
			\omegasr{1}{}\bmu^{(1)} - \omegasr{0}{}\bmu^{(0)} = \begin{pmatrix}
				\omegasr{1}{S, S}\bmu^{(1)}_{S} - \omegasr{0}{S, S}\bmu^{(0)}_{S} \\ 0
			\end{pmatrix}.
		\end{equation}
		Again by induction, it can be seen that for any subspace $S \supseteq S^*, D(S) = D(S^*)$.
	\end{enumerate}
	
\end{proof}

\begin{lemma}\label{lem: qda consis 3}
	It holds that
	\begin{equation}
		 \inf_{j \in S_q^*} \inf_{\substack{S: S^* \backslash S = \{j\}\\ |S| \leq D+p^*}}T(S) - T(S^*)\geq \frac{m^2}{4M}\gammaq \cdot \min\left\{m\gammaq, \frac{1}{3}, \frac{m^2}{4M}\gammaq \right\}.
	\end{equation}
\end{lemma}

\begin{proof}[Proof of Lemma \ref{lem: qda consis 3}]
	Suppose $S^* \backslash S = \{j\}, \tilde{S} = S \cup \{j\} \supseteq S^*$ and  $j\in S_q^*$. Due to Lemma \ref{lem: qda consis 2}, equation \eqref{eq: sigma inv} and Assumption \ref{asmp: qda}, we have either $\norma{(k_{j, S}^{(1)})^{-1} - (k_{j, S}^{(0)})^{-1}} \geq \gammaq$ or $\infnorm{(k_{j, S}^{(1)})^{-1}\omegasr{1}{S, S}\sigsr{1}{S, j} - (k_{j, S}^{(0)})^{-1}\omegasr{0}{S, S}\sigsr{0}{S, j}} \geq \gammaq$ holds. And it's easy to notice that for $r = 0, 1$,
	\begin{align}
		(k_{j, S}^{(r)})^{-1} &\leq \twonorma{\omegasr{r}{\tilde{S}, \tilde{S}}} = \lambda_{\min}^{-1}(\sigsr{r}{\tilde{S}, \tilde{S}})\leq \lambda_{\min}^{-1}(\sigsr{r}{}) \leq m^{-1},\\
		k_{j, S}^{(r)} &= \sigsr{r}{j, j} - \sigsr{r}{j, S}\omegasr{r}{S, S}\sigsr{r}{S, j} \leq \sigsr{r}{j, j} \leq M,
	\end{align}
	which implies that
	\[
		m \leq (k_{j, S}^{(r)}) \leq M, r = 0, 1.
	\]
	\begin{enumerate}[label=(\roman*)]
		\item If $\norma{(k_{j, S}^{(1)})^{-1} - (k_{j, S}^{(0)})^{-1}} \geq \gammaq$, then we have
		\begin{equation}
			\norma{\frac{k^{(0)}_{j, S}}{k^{(1)}_{j, S}} - 1} = \norma{k^{(0)}_{j, S}}\cdot \norma{(k_{j, S}^{(1)})^{-1} - (k_{j, S}^{(0)})^{-1}} \geq m\gammaq,
		\end{equation}
		leading to
		\begin{equation}
			\frac{k^{(0)}_{j, S}}{k^{(1)}_{j, S}} - 1 - \log\left(\frac{k^{(0)}_{j, S}}{k^{(1)}_{j, S}}\right) \geq m\gammaq - \log(1 + m\gammaq).
		\end{equation}
		When $m\gammaq \leq 1$, we know that 
		\[
			m\gammaq - \log(1 + m\gammaq) \geq \frac{1}{2}(m\gammaq)^2 - \frac{1}{3}(m\gammaq)^3 \geq \frac{1}{6}(m\gammaq)^2.
		\]
		When $m\gammaq > 1$, it holds that
		\[
			m\gammaq - \log(1 + m\gammaq) \geq\frac{1}{6}m\gammaq.
		\]
		Thus, we have
		\[
			\frac{k^{(0)}_{j, S}}{k^{(1)}_{j, S}} - 1 - \log\left(\frac{k^{(0)}_{j, S}}{k^{(1)}_{j, S}}\right) \geq \frac{1}{6}\min\{m\gammaq, (m\gammaq)^2\}.
		\]
		And the same result holds for $\frac{k^{(1)}_{j, S}}{k^{(0)}_{j, S}}$ as well. Therefore by \eqref{T}, we have
		\begin{align}
			T(S) - T(\tilde{S}) &= T(S) - T(S^*)\\
			&\geq \pi_0\left[\frac{k^{(0)}_{j, S}}{k^{(1)}_{j, S}} - 1 - \log\left(\frac{k^{(0)}_{j, S}}{k^{(1)}_{j, S}}\right)\right] + \pi_1\left[\frac{k^{(1)}_{j, S}}{k^{(0)}_{j, S}} - 1 - \log\left(\frac{k^{(1)}_{j, S}}{k^{(0)}_{j, S}}\right)\right] \\
			&\geq \frac{1}{6}\min\{m\gammaq, (m\gammaq)^2\}.
		\end{align}
		Since this holds for arbitrary $j \in S^*_q$ and $S$ satisfying $S^* \backslash S = \{j\}$, we obtain that
		\begin{equation}\label{eq: inf 1}
			\inf_{j \in S_q^*} \inf_{\substack{S: S^* \backslash S = \{j\}\\ |S| \leq D+p^*}}T(S) - T(S^*) \geq \frac{1}{6}\min\{m\gammaq, (m\gammaq)^2\}.
		\end{equation}
		\item If $\infnorma{(k_{j, S}^{(1)})^{-1}\omegasr{1}{S, S}\sigsr{1}{S, j} - (k_{j, S}^{(0)})^{-1}\omegasr{0}{S, S}\sigsr{0}{S, j}} \geq \gammaq$, 
		when $\norma{(k_{j, S}^{(1)})^{-1} - (k_{j, S}^{(0)})^{-1}} \geq \frac{m}{2M}\gammaq$, similar to (\rom{1}), we have
		\begin{equation}\label{eq: inf 2}
			T(S) - T(\tilde{S}) \geq \frac{1}{6}\min\left\{\frac{m^2}{2M}\gammaq, \frac{m^4}{4M^2}\gammaq^2\right\}.
		\end{equation}
		Otherwise,
		\begin{align}
			\gammaq &\leq \infnorma{(k_{j, S}^{(1)})^{-1}\omegasr{1}{S, S}\sigsr{1}{S, j} - (k_{j, S}^{(0)})^{-1}\omegasr{0}{S, S}\sigsr{0}{S, j}} \\
			&\leq \norma{(k_{j, S}^{(1)})^{-1} - (k_{j, S}^{(0)})^{-1}}\cdot \infnorma{\omegasr{1}{S, S}\sigsr{1}{S, j}} +  (k_{j, S}^{(0)})^{-1}\cdot \infnorma{\omegasr{1}{S, S}\sigsr{1}{S, j} - \omegasr{0}{S, S}\sigsr{0}{S, j}}.
		\end{align}
		By \eqref{eq: sigma inv}, $(k_{j, S}^{(1)})^{-1}\omegasr{1}{S, S}\sigsr{1}{S, j}$ is part of $\omegasr{1}{\tilde{S}, \tilde{S}}$, then
		\begin{equation}
			M^{-1}\infnorma{\omegasr{1}{S, S}\sigsr{1}{S, j}} \leq \infnorma{(k_{j, S}^{(1)})^{-1}\omegasr{1}{S, S}\sigsr{1}{S, j}} \leq \maxnorma{\omegasr{1}{\tilde{S}, \tilde{S}}} \leq \twonorma{\omegasr{1}{\tilde{S}, \tilde{S}}} \leq m^{-1},
		\end{equation}
		yielding
		\begin{equation}
			\infnorma{\omegasr{1}{S, S}\sigsr{1}{S, j}}\leq \frac{M}{m}.
		\end{equation}
		Then we have
		\begin{equation}
			\twonorma{\omegasr{1}{S, S}\sigsr{1}{S, j} - \omegasr{0}{S, S}\sigsr{0}{S, j}} \geq \infnorma{\omegasr{1}{S, S}\sigsr{1}{S, j} - \omegasr{0}{S, S}\sigsr{0}{S, j}} \geq \frac{1}{2}m\gammaq,
		\end{equation}
		leading to
		\begin{align}
			T(S) - T(\tilde{S}) &\geq \twonorma{\omegasr{1}{S, S}\sigsr{1}{S, j} - \omegasr{0}{S, S}\sigsr{0}{S, j}}^2 \cdot \lambda_{\min}\left(\pi_0 (k_{j, S}^{(1)})^{-1}\sigsr{0}{S, S} + \pi_1 (k_{j, S}^{(0)})^{-1}\sigsr{1}{S, S}\right)\\
			&\geq \left(\frac{1}{2}m\gammaq\right)^2\cdot \left(\pi_0 M^{-1}\lambda_{\min}(\sigsr{0}{S, S}) + \pi_1 M^{-1}\lambda_{\min}(\sigsr{1}{S, S})\right) \\
			&\geq \frac{m^3\gammaq^2}{4M}.\label{eq: inf 3}
		\end{align}
		Combining \eqref{eq: inf 1}, \eqref{eq: inf 2} and \eqref{eq: inf 3}, we complete the proof.
	\end{enumerate}
\end{proof}

\begin{lemma}\label{lem: qda consis 4}
	It holds that
	\begin{equation}
		\inf_{j \in S^*_l\backslash S^*_q} \inf_{\substack{S: S^* \backslash S = \{j\}\\ |S| \leq D+p^*}}[D(S^*) - D(S)]\geq m^3\gammal^2.
	\end{equation}
\end{lemma}

\begin{proof}
	For any $j \in S^*_l\backslash S^*_q$ and $S$ satisfying $S^* \backslash S = \{j\}$, let $\tilde{S} = S \cup \{j\}$, it's easy to see that there holds
\begin{equation}\label{eq: del}
	\omegasr{1}{\tilde{S}, \tilde{S}}\begin{pmatrix}
		\bmu_{S}^{(1)}\\ \bmu_{j}^{(1)}
	\end{pmatrix} - \omegasr{0}{\tilde{S}, \tilde{S}}\begin{pmatrix}
		\bmu_{S}^{(0)}\\ \bmu_{j}^{(0)}
	\end{pmatrix} = \begin{pmatrix}
		* \\ \bm{\delta}_{j}'
	\end{pmatrix},
\end{equation}
Denote $c = \Sigma_{j, S}^{(1)}\omegasr{1}{S, S}(\bmu_{S}^{(1)} - \bmu_{S}^{(0)}) - (\bmu_{j}^{(1)} - \bmu_{j}^{(0)})$. By \eqref{C1}, \eqref{C2}, we can obtain
\begin{equation}
	\Sigma^{(0)}_{j, S}\omegasr{0}{S, S}(\bmu^{(1)}_{S} - \bmu^{(0)}_{S}) - (\bmu^{(1)}_j - \bmu^{(0)}_j) = \Sigma^{(1)}_{j, S}\omegasr{1}{S, S}(\bmu^{(1)}_{S} - \bmu^{(0)}_{S}) - (\bmu^{(1)}_j - \bmu^{(0)}_j).
\end{equation}
Combining with \eqref{eq: del}, we have
\begin{equation}
	m^{-1}|c| \geq (k^{(1)}_{j, S})^{-1}|c| = |\bm{\delta}_{j}'| \geq \gammal,
\end{equation}
because $\bm{\delta}_{j}'$ is the corresponding component of $\bm{\delta}_{S^*}$ with respect to feature $j$. We obtain that $|c| \geq m\gammal$, which leads to 
\begin{equation}
	 D(S^*) - D(S) = D(\tilde{S}) - D(S) = k^{(1)}_{j, S}|c|^2 \geq m^3\gammal^2.
\end{equation}
\end{proof}

\begin{lemma}\label{lem: qda consis 5}
	It holds that
	\begin{equation}
		\inf_{\substack{S: S \not\supseteq S^*\\ |S| \leq D}}\ric(S) - \ric(S^*) \geq \min\left\{\frac{m^2}{4M}\gammaq \cdot \min\left\{m\gammaq, \frac{1}{3}, \frac{m^2}{4M}\gammaq \right\}, m^3\gammal^2\right\}.
	\end{equation}
\end{lemma}

\begin{proof}[Proof of Lemma \ref{lem: qda consis 5}]
	Because of Lemmas \ref{lem: qda consis 3} and \ref{lem: qda consis 4}, it suffices to prove
	\begin{align}
		&\inf_{\substack{S: S \not\supseteq S^*\\ |S| \leq D}}\ric(S) - \ric(S^*) \\
		&\geq \inf_{j \in S^*} \inf_{\substack{S: S^* \backslash S = \{j\}\\ |S| \leq D+p^*}} \ric(S) - \ric(S^*)\\
		&\geq \min\left\{\inf_{j \in S^*_q} \inf_{\substack{S: S^* \backslash S = \{j\}\\ |S| \leq D+p^*}} [T(S) - T(S^*)], \inf_{j \in S^*_l\backslash S^*_q} \inf_{\substack{S: j\notin S\\ |S| \leq D+p^*}} [D(S^*) - D(S)]\right\}.
	\end{align}
	For any subset $S$ which does not cover $S^*$ and $|S| \leq D$, consider feature $j \in S^*\backslash S$ and $\tilde{S} = S \cup S^* \backslash \{j\} \supseteq S$. It's easy to notice that $|\tilde{S}| \leq D + p^*$. By the monotonicity proved in Lemma \ref{lem: qda consis 2}, we know that $\ric(\tilde{S}) \leq \ric(S)$. In addition, we know that
	\[
		\ric(\tilde{S}) \geq \inf_{j \in S^*} \inf_{\substack{S: S^* \backslash S = \{j\}\\ |S| \leq D+p^*}} \ric(S),
	\]which yields the first inequality. For the second inequality, it directly comes from Lemma \ref{lem: qda consis 2}.(\rom{3}).
\end{proof}

\begin{lemma}\label{lem: qda consis 6}
	If Assumption \ref{asmp: qda} holds, for $\epsilon$ smaller than some constant and $n, p$ larger than some constants, we have
	\begin{equation}
		\tp\left(\sup_{S: |S| \leq D}|\widehat{\ric}(S) - \ric(S)| > \epsilon \right) \lesssim p^2\exp\left\{-Cn\left(\frac{\epsilon}{D^2}\right)^2\right\}.
	\end{equation}
\end{lemma}

\begin{proof}[Proof of Lemma \ref{lem: qda consis 6}]
	Recall that $\ric(S) = 2[T(S) - D(S)], \widehat{\ric}(S) = 2[\hat{T}(S) - \hat{D}(S)]$. Denote $T_1(S) = \tr\left[\left(\omegasr{1}{S, S} - \omegasr{0}{S, S}\right)\left(\pi_1\Sigma_{S, S}^{(1)} - \pi_0\Sigma_{S, S}^{(0)}\right)\right], T_2(S) = (\pi_1 - \pi_0)(\log|\Sigma_{S, S}^{(1)}| - \log|\Sigma_{S, S}^{(0)}|)$, then $T(S) = T_1(S) + T_2(S)$.
	
	And since for any $S$ with $|S| \leq D$, we have
	\begin{align}
		&\norma{\hat{T_1}(S) - T_1(S)} \\
		&= \norma{\tr\left[\left(\homegasr{1}{S, S} - \homegasr{0}{S, S}\right)\left(\hpi_1\hsig_{S, S}^{(1)} - \hpi_0\hsig_{S, S}^{(0)}\right) - \left(\omegasr{1}{S, S} - \omegasr{0}{S, S}\right)\left(\pi_1\Sigma_{S, S}^{(1)} - \pi_0\Sigma_{S, S}^{(0)}\right)\right]}\\
		&\leq D\twonorma{\left(\homegasr{1}{S, S} - \homegasr{0}{S, S}\right)\left(\hpi_1\hsig_{S, S}^{(1)} - \hpi_0\hsig_{S, S}^{(0)}\right) - \left(\omegasr{1}{S, S} - \omegasr{0}{S, S}\right)\left(\pi_1\Sigma_{S, S}^{(1)} - \pi_0\Sigma_{S, S}^{(0)}\right)}\\
		&\leq D\twonorma{\omegasr{1}{S, S} - \omegasr{0}{S, S}}\cdot \left[|\hpi_1 - \pi_1|\cdot \left(\twonorma{\sigsr{1}{S, S}} + \twonorma{\sigsr{0}{S, S}}\right) + \twonorma{\hsigsr{1}{S, S} - \sigsr{1}{S, S}} \right.\\
		&\quad \left.+ \twonorma{\hsigsr{0}{S, S} - \sigsr{0}{S, S}}\right] + D\twonorma{\pi_1\sigsr{1}{S, S} - \pi_0\sigsr{0}{S, S}}\cdot \left(\twonorma{\homegasr{1}{S, S} - \omegasr{1}{S, S}} + \twonorma{\homegasr{0}{S, S} - \omegasr{0}{S, S}}\right)\\
		&\leq 2Dm^{-1}\cdot \left[|\hpi_1 - \pi_1|\cdot 2M + \twonorma{\hsigsr{1}{S, S} - \sigsr{1}{S, S}} + \twonorma{\hsigsr{0}{S, S} - \sigsr{0}{S, S}}\right]\\
		&\quad  + DM\cdot \left(\twonorma{\homegasr{1}{S, S} - \omegasr{1}{S, S}} + \twonorma{\homegasr{0}{S, S} - \omegasr{0}{S, S}}\right),
	\end{align}
	where the last inequality comes from $\twonorma{\sigsr{r}{S, S}}\leq M$ and $\twonorma{\omegasr{1}{S, S} - \omegasr{0}{S, S}} \leq 2m^{-1}$. And this yields
	\begin{align}
		&\tp\left(\supp\norma{\hat{T_1}(S) - T_1(S)} > \epsilon \right) \\
		&\leq \tp\left(4Dm^{-1}M|\hpi_1 - \pi_1| > \frac{\epsilon}{4}\right) +  \tp\left(2Dm^{-1}\cdot \supp\twonorma{\hsigsr{1}{S, S} - \sigsr{1}{S, S}} > \frac{\epsilon}{8}\right) \\
		&\quad +  \tp\left(2Dm^{-1}\cdot \supp\twonorma{\hsigsr{0}{S, S} - \sigsr{0}{S, S}} > \frac{\epsilon}{8}\right) + \tp\left(DM\supp \twonorma{\homegasr{1}{S, S} - \omegasr{1}{S, S}} > \frac{\epsilon}{4}\right) \\
		&\quad + \tp\left(DM\supp \twonorma{\homegasr{0}{S, S} - \omegasr{0}{S, S}} > \frac{\epsilon}{4}\right) \\
		&\lesssim p^2\exp\left\{-Cn\left(\frac{\epsilon}{D^2}\right)^2\right\}.
	\end{align}
	On the other hand, we have
	\begin{align}
		\norma{\log\left(\frac{|\hsigsr{1}{S, S}|}{|\sigsr{1}{S, S}|}\right)} &\leq \sum_{i=1}^{|S|}\norma{\log\left(\frac{\lambda_i(\hsigsr{1}{S, S})}{\lambda_i(\sigsr{1}{S, S})}\right)} \\
		&\leq \sum_{i=1}^{|S|}\norma{\frac{\lambda_i(\hsigsr{1}{S, S})}{\lambda_i(\sigsr{1}{S, S})} - 1} + o\left(\norma{\frac{\lambda_i(\hsigsr{1}{S, S})}{\lambda_i(\sigsr{1}{S, S})} - 1}\right).
	\end{align}
	By Weyl's inequality,
	\begin{equation}
		\norma{\frac{\lambda_i(\hsigsr{1}{S, S}) - \lambda_i(\sigsr{1}{S, S})}{\lambda_i(\sigsr{1}{S, S})}} \leq \frac{1}{m}\twonorma{\hsigsr{1}{S, S} - \sigsr{1}{S, S}}.
	\end{equation}
	Therefore there exists $\epsilon' > 0$ such that when $\twonorma{\hsigsr{1}{S, S} - \sigsr{1}{S, S}} \leq \epsilon'$, we have
	\begin{equation}
		\norma{\log\left(\frac{|\hsigsr{1}{S, S}|}{|\sigsr{1}{S, S}|}\right)} \leq \frac{D}{m}\twonorma{\hsigsr{1}{S, S} - \sigsr{1}{S, S}} + o\left(D\twonorma{\hsigsr{1}{S, S} - \sigsr{1}{S, S}}\right) \leq \frac{2D}{m}\twonorma{\hsigsr{1}{S, S} - \sigsr{1}{S, S}}.
	\end{equation}
	Therefore, there holds
	\begin{align}
		&\tp\left(\supp\norma{\hat{T_2}(S) - T_2(S)} > \epsilon \right) \\
		&\leq \tp\left(\supp 2\norma{\hpi_1 - \pi_1}\cdot \norma{\log |\sigsr{1}{S, S}| + \log |\sigsr{0}{S, S}|} + \log\left(\frac{|\hsigsr{1}{S, S}|}{|\sigsr{1}{S, S}|}\right) + \log\left(\frac{|\hsigsr{0}{S, S}|}{|\sigsr{0}{S, S}|}\right) > \epsilon \right)\\
		&\leq \tp\left(2\norma{\hpi_1 - \pi_1}\cdot \supp \sum_{i=1}^{|S|}\left[\norma{\log\left(\lambda_i\left(\sigsr{1}{S, S}\right)\right)} + \norma{\log\left(\lambda_i\left(\sigsr{1}{S, S}\right)\right)}\right] > \frac{\epsilon}{3}\right) \\
		&\quad + \tp\left(\supp \log\left(\frac{|\hsigsr{1}{S, S}|}{|\sigsr{1}{S, S}|}\right) > \frac{\epsilon}{3}\right) + \tp\left(\supp \log\left(\frac{|\hsigsr{0}{S, S}|}{|\sigsr{0}{S, S}|}\right) > \frac{\epsilon}{3}\right) \\
		&\leq \tp\left(2D\max\{|\log M|, |\log m|\}\cdot \norma{\hpi_1 - \pi_1} > \frac{\epsilon}{3}\right)+ \tp\left(\frac{2D}{m}\cdot \supp \twonorma{\hsigsr{1}{S, S} - \sigsr{1}{S, S}} > \frac{\epsilon}{3}\right)\\
		&\quad  + \tp\left(\frac{2D}{m}\cdot \supp \twonorma{\hsigsr{0}{S, S} - \sigsr{0}{S, S}} > \frac{\epsilon}{3}\right) + \tp\left(\supp\twonorma{\hsigsr{1}{S, S} - \sigsr{1}{S, S}} > \epsilon'\right)\\
		&\quad + \tp\left(\supp\twonorma{\hsigsr{0}{S, S} - \sigsr{0}{S, S}} > \epsilon'\right)\\
		&\lesssim p^2\exp\left\{-Cn\left(\frac{\epsilon}{D^2}\right)^2\right\}.
	\end{align}
	Thus we have
	\begin{equation}\label{eq: t conv}
		\tp\left(\supp\norma{\hat{T}(S) - T(S)} > \epsilon \right) \lesssim p^2\exp\left\{-Cn\left(\frac{\epsilon}{D^2}\right)^2\right\}.
	\end{equation}
	Besides, following the same strategy in the proof of Lemma \ref{lem: lda consis 5}, it can be shown that
	\begin{equation}\label{eq: d conv}
		\tp\left(\supp\norma{\hat{D}(S) - D(S)} > \epsilon \right) \lesssim p^2\exp\left\{-Cn\left(\frac{\epsilon}{D^2}\right)^2\right\}.
	\end{equation}
	By \eqref{eq: t conv} and \eqref{eq: d conv}, we complete the proof.
\end{proof}

By all the above lemmas, and following similar idea in the proof of Theorem \ref{thm: consistency}, we can prove the consistency stated in the theorem.

\subsection{Proof of Theorem \ref{thm: error bound using RIC}}
Firstly since taking subspace can also be seen as an axis-aligned projection, applying Theorem 2 in \cite{cannings2017random} we have
\begin{align}\label{eq: train ric 1}
	\be[R(C_n^{RaSE}) - R(C_{Bayes})] \leq \frac{\be[R(C_n^{S_{1*}}) - R(C_{Bayes})]}{\min(\alpha, 1 - \alpha)}.
\end{align}
Then it can be easily noticed that
\begin{align}\label{eq: train ric 2}
	\te\{\be[R(C_n^{S_{1*}}) - R(C_{Bayes})]\} &\leq \te\{\be[(R(C_n^{S_{1*}}) - R(C_{Bayes}))\mathds{1}(S_{1*} \supseteq S^*)]\} \nonumber\\ 
	&\quad + (1-R(C_{Bayes}))\p(S_{1*} \not\supseteq S^*)\nonumber\\
	&\leq \te\sup_{\substack{S: S\supseteq S^* \\ |S| \leq D}}\left[R(C_n^{S}) - R(C_{Bayes})\right] + \p(S_{1*} \not\supseteq S^*).
\end{align}
Combining \eqref{eq: train ric 1} and \eqref{eq: train ric 2}, we obtain the conclusion.

\subsection{Proof of Theorem \ref{thm: error bound using training error}}
This conclusion is almost the same as Theorem 2 in \cite{cannings2017random}. However, since we are studying the discrete space of subspaces and discriminative sets are ideal subspaces here, we can drop Assumptions 2 and 3 in their paper and get a similar upper bound. Firstly we have
\begin{align}\label{eq: train1}
	\be[R(C_n^{RaSE}) - R(C_{Bayes})] \leq \frac{\be[R(C_n^{S_{1*}}) - R(C_{Bayes})]}{\min(\alpha, 1 - \alpha)},
\end{align}
by \cite{cannings2017random}. Then write
\begin{equation}\label{eq: train2}
	\be[R(C_n^{S_{1*}})] = \be[R_n(C_n^{S_{1*}})] + \epsilon_n,
\end{equation}
where $\epsilon_n = \be[R(C_n^{S_{1*}})] - \be[R_n(C_n^{S_{1*}})]$.
Then we have
\begin{align}\label{eq: train3}
	&\be[R_n(C_n^{S_{1*}})] \nonumber\\
	&\leq \supps R_n(C_n^{S}) + \be\left[R(C_n^{S_{1*}}) \cdot \mathds{1}\left(R_n(C_n^{S_{1*}}) > \supps R_n(C_n^{S})\right)\right] \nonumber\\
	&\leq  \supps R_n(C_n^{S})  + \bp\left(R_n(C_n^{S_{1*}}) > \supps R_n(C_n^{S})\right) \nonumber\\
	&=  \supps R_n(C_n^{S}) + \left[\bp\left(R_n(C_n^{S_{11}}) > \supps R_n(C_n^{S})\right)\right]^{B_2} \nonumber\\
	&\leq \supps R_n(C_n^{S}) + (1 - p_{S^*})^{B_2} \nonumber\\
	&= \supps R(C_n^{S}) + \supps |\epsilon_n^{S}|+ (1 - p_{S^*})^{B_2},
\end{align}
where $\epsilon_n^{S} = R(C_n^{S}) - R_n(C_n^{S})$.
Combining \eqref{eq: train1}, \eqref{eq: train2} and \eqref{eq: train3}, we obtain
\begin{align}
	&\te\{\be[R(C_n^{RaSE}) - R(C_{Bayes})]\} \\
	&\leq \frac{\te\sup\limits_{\substack{S: S\supseteq S^* \\ |S| \leq D}}\left[R(C_n^{S}) - R(C_{Bayes})\right] + \te\supps\limits|\epsilon_n^{S}| + \te(\epsilon_n) +  (1 - p_{S^*})^{B_2}}{\min(\alpha, 1 - \alpha)}.
\end{align}

\subsection{Proof of Proposition \ref{prop: lda error rate}}
First we prove the following lemma.
\begin{lemma}[\cite{devroye2013probabilistic}]\label{lem: devroye}
	For a non-negative random variable $z$ satisfying \[
		\p(z > t) \leq C_1\exp\{-C_2nt^\alpha\},
	\]for any $t > 0$, where $C_1 > 1$, $C_2 > 0$ and $\alpha \geq 1$ are three fixed constants, then we have\[
		\e z \leq \left(\frac{\log C_1 + 1}{C_2n}\right)^{\frac{1}{\alpha}}.
	\]
\end{lemma}

\begin{proof}[Proof of Lemma \ref{lem: devroye}]
	It's easy to see
	\begin{equation}
		\e z^\alpha = \int_0^{\infty} \p(z^\alpha > t) dt \\
		= \epsilon + C_1\int_{\epsilon}^{\infty} \exp\{-C_2nt\} dt \\
		= \epsilon + \frac{C_1}{C_2 n}\cdot \exp\{-C_2n\epsilon\},
	\end{equation}
	leading to\[
		\e z^\alpha \leq \inf_{\epsilon > 0} \left[\epsilon + \frac{C_1}{C_2 n}\cdot \exp\{-C_2n\epsilon\}\right] = \frac{\log C_1 + 1}{C_2n}.
	\]Then by Jensen's inequality, it holds\[
		\e z \leq \left(\e z^\alpha\right)^{\frac{1}{\alpha}} \leq \left(\frac{\log C_1 + 1}{C_2n}\right)^{\frac{1}{\alpha}},
	\]which completes the proof.
\end{proof}

Then let's prove the proposition. Notice that\[
	\Delta_S^2 = |\bdeltas^T \Sigma_{S, S}\bdeltas| \geq \twonorm{\bdeltas}^2\cdot \lambda_{\min}(\Sigma_{S, S}) \geq mp^*\gamma^2,
\]for any $S \supseteq S^*$. In addition, due to mean value theorem, it follows that\[
\Phi\left(-\frac{\hat{\Delta}_S}{2} + \htaus\right)  - \Phi\left(-\frac{\Delta_S}{2} + \taus\right) \leq \frac{1}{2}|\hat{\Delta}_S - \Delta_S| + |\htaus - \taus|.
\]By a similar argument, due to \eqref{eq: error diff lda}, we can obtain that
\begin{equation}
	R(C_n^{S-LDA}) - R(C_{Bayes}) \leq \frac{1}{2}|\hat{\Delta}_S - \Delta_S| + |\htaus - \taus|.
\end{equation}
By Lemma \ref{lem: lda consis 5}, we know that
\begin{equation}
	\tp\left(\supp \norm{\hat{\Delta}_S^2 - \Delta_S^2} > \epsilon\right) \lesssim p^2\exp\left\{-Cn\left(\frac{\epsilon}{D^2}\right)^2\right\}.
\end{equation}
This yields that
\begin{align}
	&\tp\left(\supps \left[ R(C_n^{S-LDA}) - R(C_{Bayes})\right] > \epsilon \right) \\
	&\leq \tp\left(\supps \norm{\hat{\Delta}_S^2 - \Delta_S^2} > \frac{3}{4}mp^*\gamma^2\right) \\
	&\quad + \bp\left(\frac{1}{2}\cdot \supps\frac{\norm{\hat{\Delta}_S^2 - \Delta_S^2}}{\hat{\Delta}_S + \Delta_S} > \frac{1}{2}\epsilon, \supps \norm{\hat{\Delta}_S^2 - \Delta_S^2} \leq \frac{3}{4}mp^*\gamma^2 \right) \\
	&\quad + \bp\left(\supps\norm{\htaus - \taus} > \frac{1}{2}\epsilon, \supps \norm{\hat{\Delta}_S^2 - \Delta_S^2} \leq \frac{3}{4}mp^*\gamma^2 \right) \\
	&\leq \tp\left(\supp \norm{\hat{\Delta}_S^2 - \Delta_S^2} > \frac{3}{4}mp^*\gamma^2\right) + \tp\left(\supp \norm{\hat{\Delta}_S^2 - \Delta_S^2} > \frac{3}{2}\epsilon\sqrt{mp^*\gamma^2}\right) \\
	&\quad + \tp\left(\log \left(\frac{\pi_1}{\pi_0}\right)\cdot \supp \frac{\norm{\hat{\Delta}_S - \Delta_S}}{\hat{\Delta}_S \Delta_S} > \frac{1}{4}\epsilon, \supps\norm{\hat{\Delta}_S^2 - \Delta_S^2} \leq \frac{3}{4}mp^*\gamma^2\right) \\
	&\quad + \tp\left(\norma{\log \left(\frac{\hpi_1}{\hpi_0}\right) - \log \left(\frac{\pi_1}{\pi_0}\right)}\cdot \supp\frac{1}{\hat{\Delta}_S} > \frac{1}{4}\epsilon, \supps\norm{\hat{\Delta}_S^2 - \Delta_S^2} \leq \frac{3}{4}mp^*\gamma^2\right) \\
	&\leq \tp\left(\supp \norm{\hat{\Delta}_S^2 - \Delta_S^2} > \frac{3}{4}mp^*\gamma^2\right) + \tp\left(\supp \norm{\hat{\Delta}_S^2 - \Delta_S^2} > \frac{3}{2}\epsilon\sqrt{mp^*\gamma^2}\right) \\
	&\quad +\tp\left(\supp \norm{\hat{\Delta}_S^2 - \Delta_S^2} > C(p^*)^{\frac{3}{2}}\gamma^3\epsilon\right) + \tp\left(\norm{\hpi_0 - \pi_0} > C\epsilon\sqrt{mp^*\gamma^2}\right) \\
	&\quad + \tp\left(\norm{\hpi_1 - \pi_1} > C\epsilon\sqrt{mp^*\gamma^2}\right) \\
	&\lesssim p^2\exp\left\{-Cn\left(\frac{\epsilon(p^*)^{\frac{3}{2}}\gamma^3}{D^2}\right)^2\right\} + p^2\exp\left\{-Cn\left(\frac{\epsilon(p^*)^{\frac{1}{2}}\gamma}{D^2}\right)^2\right\},
\end{align}
which completes the proof by applying Lemma \ref{lem: devroye}.

\subsection{Proof of Proposition \ref{prop: qda error rate}}
Denote $\bdeltas = \omegasr{1}{S, S}\bmur{1}_{S} - \omegasr{0}{S, S}\bmur{0}_{S}$, $\Omega_{S,S} = \omegasr{1}{S, S} - \omegasr{0}{S, S}$, and $\ds(\bxs) = \log\left(\frac{\pi_1}{\pi_0}\right) - \frac{1}{2}\bxs^T \Omega_{S, S}\bxs + \bdeltas^T\bxs - \frac{1}{2}(\bmusr{1})^T\omegasr{1}{S, S}\bmusr{1} + \frac{1}{2}(\bmusr{0})^T\omegasr{0}{S, S}\bmusr{0}$. Their estimators are correspondingly denoted as $\hdeltas = \homegasr{1}{S, S}\hmur{1}_{S} - \homegasr{0}{S, S}\hmur{0}_{S}$, $\hat{\Omega}_{S,S} = \homegasr{1}{S, S} - \homegasr{0}{S, S}$, and $\hds(\bxs) = \log\left(\frac{\hpi_1}{\hpi_0}\right) - \frac{1}{2}\bxs^T \hat{\Omega}_{S, S}\bxs + \hdeltas^T\bxs - \frac{1}{2}(\hmusr{1})^T\homegasr{1}{S, S}\hmusr{1} + \frac{1}{2}(\hmusr{0})^T\homegasr{0}{S, S}\hmusr{0}$. From the proof of Theorem \ref{thm: qda consistency}, we know that $\ds(\bxs) = d_{S^*}(\bx_{S^*})$ for any $S \supseteq S^*$. Denote the training data as $D_{tr}$, then $R(C_n^S) = \pi_0\p^{(0)}(\hds(\bxs) > 0| D_{tr}) + \pi_1\p^{(1)}(\hds(\bxs) \leq 0| D_{tr}), R(C_{Bayes}) = \pi_0\p^{(0)}(d_{S^*}(\bx_{S^*}) > 0) + \pi_1\p^{(1)}(d_{S^*}(\bx_{S^*}) \leq 0)$. Then we have
\begin{align} 
	&\supps [\p^{(0)}(\hds(\bxs) > 0| D_{tr}) - \p^{(0)}(d_{S^*}(\bx_{S^*}) > 0)] \\
	&\leq \supps\p^{(0)}(\ds(\bxs) >\ds(\bxs) -\hds(\bxs)| D_{tr}) - \p^{(0)}(d_{S^*}(\bx_{S^*}) > 0) \\
	&\leq  \p^{(0)}(d_{S^*}(\bx_{S^*}) > -\epsilon) + \supps \p^{(0)}(\norm{\ds(\bxs) -\hds(\bxs)} > \epsilon| D_{tr}) \\
	&\quad- \p^{(0)}(d_{S^*}(\bx_{S^*}) > 0) \\
	&\leq \int_{-\epsilon}^0 h^{(0)}(z)dz + \supps \p^{(0)}(\norm{\ds(\bxs) -\hds(\bxs)} > \epsilon| D_{tr}) \\
	&\leq \epsilon u_c + \supps \p^{(0)}(\norm{\ds(\bxs) -\hds(\bxs)} > \epsilon| D_{tr}),\label{eq: prob d s}
\end{align}
for any $\epsilon \in (0, c)$. Denote $\bm{a}_S = (\sigsr{0}{S,S})^{\frac{1}{2}}[\omegasr{1}{S, S}(\bmusr{1} - \hmusr{1}) + (\omegasr{1}{S, S} - \homegasr{1}{S, S})(\hmusr{1} - \hmusr{0}) + \homegasr{0}{S, S}(\hmusr{0} - \bmusr{0})]$. Notice that
\begin{align}
	\ds(\bxs) - \hds(\bxs) &= \log\left(\frac{\pi_1}{\pi_0}\right) - \log\left(\frac{\hpi_1}{\hpi_0}\right) + \frac{1}{2}(\bxs - \bmusr{0})^T(\hat{\Omega}_{S, S}- \Omega_{S, S})(\bxs - \bmusr{0}) \\
	&\quad + \bm{a}_S^T(\sigsr{0}{S,S})^{-\frac{1}{2}}(\bxs - \bmusr{0}) + \frac{1}{2}(\bmusr{1} - \bmusr{0})^T\omegasr{1}{S, S}(\hmusr{1} - \bmusr{1}) \\
	&\quad- \frac{1}{2}(\hmusr{0} - \bmusr{0})^T\omegasr{0}{S, S}(\hmusr{0} - \bmusr{0}).
\end{align}
Further, denote $\bm{z}_S = (\sigsr{0}{S,S})^{-\frac{1}{2}}(\bxs - \bmusr{0}) \sim N(\bm{0}_{|S|}, I_{|S|})$. Then it follows that
\begin{align}
	\norm{\ds(\bxs) - \hds(\bxs)} &\leq C\norm{\hpi_1 - \pi_1} + \frac{1}{2}\twonorm{\hat{\Omega}_{S, S} - \Omega_{S, S}}\cdot \twonorm{\sigsr{0}{S, S}}\cdot \twonorm{\bm{z}_S}^2 + \norm{\bm{a}_S^T\bm{z}_S} \\
	&\quad + \frac{1}{2}\twonorm{\bmusr{1} - \bmusr{0}} \cdot \twonorm{\omegasr{1}{S, S}} \cdot \twonorm{\hmusr{1} - \bmusr{1}} \\
	&\quad + \frac{1}{2}\twonorm{\hmusr{0} - \bmusr{0}}^2\cdot \twonorm{\omegasr{0}{S, S}} \\
	&\leq C\norm{\hpi_1 - \pi_1} + \frac{M}{2}\twonorm{\hat{\Omega}_{S, S} - \Omega_{S, S}}\cdot \twonorm{\bm{z}_S}^2 + \norm{\bm{a}_S^T\bm{z}_S} \\
	&\quad + M'Dm^{-1}\cdot \infnorm{\hmusr{1} - \bmusr{1}}+ Dm^{-1}\cdot \infnorm{\hmusr{0} - \bmusr{0}}^2,
\end{align}
where 
\begin{align}
	\twonorm{\bm{a}_S} &\lesssim  \twonorm{\hmusr{1} - \bmusr{1}} + \twonorm{\omegasr{1}{S, S} - \homegasr{1}{S, S}}\cdot \twonorm{\hmusr{1} - \hmusr{0}} + \twonorm{\hmusr{0} - \bmusr{0}} \\
	&\lesssim D^{\frac{1}{2}}(\infnorm{\hmusr{1} - \bmusr{1}} + \infnorm{\hmusr{0} - \bmusr{0}}) + \twonorm{\omegasr{1}{S, S} - \homegasr{1}{S, S}}\cdot \twonorm{\hmusr{1} - \hmusr{0}}.
\end{align}
For any $t, t' > 0$, denote event $\mathcal{B} = \{C\norm{\hpi_1 - \pi_1} \leq \epsilon/4, \supp\twonorm{\hat{\Omega}_{S, S} - \Omega_{S, S}} \leq t,  M'Dm^{-1}\infnorm{\hmusr{1} - \bmusr{1}}\leq \epsilon/8, Dm^{-1}\infnorm{\hmusr{0} - \bmusr{0}} \leq \epsilon/8, \twonorm{\bm{a}_S} \leq t'\}$. When $\twonorm{\bm{a}_S} \leq t'$, $\bm{a}_S^T \bm{z}_S$ is a $t'$-subGaussian.
This yields that
\begin{align}
	\p\left(\supps \norm{\ds(\bxs) - \hds(\bxs)} > \epsilon \bigg|D_{tr}\in \mathcal{B}\right) 
	&\leq \p(t\twonorm{\bm{z}_S}^2 > \epsilon/4) + \p(\norm{\bm{a}_S^T\bm{z}_S} > \epsilon/4|\twonorm{\bm{a}_S} \leq t')\\
	&\lesssim \exp\left\{-\frac{1}{D}\left(\frac{C\epsilon}{t} - D\right)^2\right\} + \exp\left\{-C\left(\frac{\epsilon}{t'}\right)^2\right\}, \label{eq: ds sup}
\end{align}
where $\epsilon$ satisfies $\frac{C\epsilon}{t} > D$ and the first term comes from the tail bound of $\chi^2_1$-distribution
\begin{equation}
	\p\left(\twonorm{\bm{z}_S}^2 > \frac{C\epsilon}{t}\right) = \p\left(\twonorm{\bm{z}_S}^2 > D + \left(\frac{C\epsilon}{t} - D\right)\right) \leq \exp\left\{-\frac{1}{D}\left(\frac{C\epsilon}{t} - D\right)^2\right\}.
\end{equation}
Taking the expectation for the training data in \eqref{eq: prob d s}, we have
\begin{align}
	&\te \supps[R(C_n^S) - R(C_{Bayes})] \\
	&\leq \epsilon u_c + \te\supps \p(\norm{\ds(\bxs) -\hds(\bxs)} > \epsilon|D_{tr} \in \mathcal{B}) + \tp(\mathcal{B}^c)\\
	&\lesssim \epsilon u_c + \exp\left\{-\frac{1}{D}\left(\frac{C\epsilon}{t} - D\right)^2\right\} + \exp\left\{-C\left(\frac{\epsilon}{t'}\right)^2\right\} + \exp\{-Cn\epsilon^2\} \\
	&\quad + p^2\exp\left\{-Cn\left(\frac{t}{D}\right)^2\right\} + p^2\exp\left\{-Cn\left(\frac{t'}{D^{3/2}}\right)^2\right\}+ p\exp\left\{-Cn\left(\frac{\epsilon}{D}\right)^2\right\}.\label{eq: error qda diff bound}
\end{align}

Let $\epsilon = \frac{D^2}{C}\left(\frac{\log p}{n}\right)^{\frac{1-2\varpi}{2}}, t = D\left(\frac{\log p}{n}\right)^{\frac{1-\varpi}{2}}, t' = D^{\frac{3}{2}}\left(\frac{\log p}{n}\right)^{\frac{1-\varpi}{2}}$ for an arbitrary $\varpi \in (0, 1/2)$ and plug them into \eqref{eq: ds sup}, we obtain that
\begin{align}
&\exp\left\{-\frac{1}{D}\left(\frac{C\epsilon}{t} - D\right)^2\right\} + \exp\left\{-C\left(\frac{\epsilon}{t'}\right)^2\right\} + \exp\{-Cn\epsilon^2\} + p^2\exp\left\{-Cn\left(\frac{t}{D}\right)^2\right\} \\
&\quad + p^2\exp\left\{-Cn\left(\frac{t'}{D^{3/2}}\right)^2\right\}+ p\exp\left\{-Cn\left(\frac{\epsilon}{D}\right)^2\right\} \\
	&\lesssim \exp\left\{-CD^3\left(\frac{n}{\log p}\right)^{\varpi}\right\} + p^2\exp\left\{-Cn^{\varpi}(\log p)^{1-\varpi}\right\}+ p\exp\left\{-CDn^{2\varpi}(\log p)^{1-\varpi}\right\} \\
	&\quad + \exp\left\{-C\left(\frac{n}{\log p}\right)^{\varpi}\right\} \\
	&\ll D^2\left(\frac{\log p}{n}\right)^{\frac{1-2\varpi}{2}},
\end{align}
because $\log p \lesssim n^{\varpi_0}$ for some $\varpi_0 \in (0,1)$. Therefore, due to \eqref{eq: error qda diff bound}, we have
\begin{equation}
	 \te\supps[R(C_n^S) - R(C_{Bayes})] \lesssim D^2\left(\frac{\log p}{n}\right)^{\frac{1-2\varpi}{2}}.
\end{equation}
\subsection{Proof of Theorem \ref{thm: iterative rase}}
We first prove the following helpful lemma.
\begin{lemma}\label{lem: iter thm}
	For $H \sim \textup{Hypergeometric}(p, \barp, d)$, if $\min(d, \barp)\cdot \frac{\barp d}{p} = o(1)$ and $B_2 \ll \left(\frac{p}{\barp d}\right)^{t+1}$, where $t$ is a positive integer no larger than $d$, then $(\Pr(H \leq t))^{B_2} \rightarrow 1$ as $p \rightarrow \infty$.
\end{lemma}

\begin{proof}[Proof of Lemma \ref{lem: iter thm}]
	The cumulative distribution function of $H$ satisfies
\begin{equation}
	\Pr(H \leq t) = 1 - \frac{\binom{d}{t+1}\binom{p-d}{\barp-t-1}}{\binom{p}{\barp}}\prescript{}{3}F_2\left[ \begin{array}{ll}1, t+1-\barp, t+1-n\\t+2, p+t+2-\barp-d\end{array}; 1\right],
\end{equation}
where $\prescript{}{3}F_2$ is the generalized hypergeometric function \citep{abadir1999introduction}. And we have
\begin{align}
	&\prescript{}{3}F_2\left[ \begin{array}{ll}1, t+1-\barp, t+1-n\\t+2, p+t+2-\barp-d\end{array}; 1\right]\\
	&= \sum_{n=0}^{\infty}\frac{1_{n}(t+1-\barp)_n (t+1-d)_n}{(t+2)_n (p+t+2-\barp-d)_n}\cdot \frac{1}{n!}\\
	&= \sum_{n=0}^{\infty}\frac{(t+1-\barp)_n (t+1-d)_n}{(t+2)_n (p+t+2-\barp-d)_n}\\
	&= 1 + \sum_{n=1}^{\min(d-t,\barp-t)}\frac{(\barp-t-1)\cdots (\barp-t-n)(d-t-1)\cdots (d-t-n)}{(t+2)_n(p+t+2-\barp-d)\cdots (p+t+1-\barp-d+n)}\\
	&\leq 1 + \sum_{n=1}^{\min(d-t,\barp-t)}\left(\frac{\barp d}{p}\right)^n \\
	&\leq 1 + \min(d, \barp)\cdot O\left(\frac{\barp d}{p}\right)\\
	&\leq 1 + o(1),
\end{align}
where $a_n \coloneqq a(a+1)\cdots (a+n-1), a_0 \coloneqq 1$ for any real number $a$ and positive integer $n$. On the other hand, we can also see that
\begin{equation}
	\prescript{}{3}F_2\left[ \begin{array}{ll}1, t+1-\barp, t+1-n\\t+2, p+t+2-\barp-d\end{array}; 1\right] \geq 1.
\end{equation}
For convenience, denote $\tilde{F} = \prescript{}{3}F_2\left[ \begin{array}{ll}1, t+1-\barp, t+1-n\\t+2, p+t+2-\barp-d\end{array}; 1\right]$, then by Taylor expansion, it holds that
\begin{equation}
	B_2\log\left[1 - \frac{\binom{d}{t+1}\binom{p-d}{\barp-t-1}}{\binom{p}{\barp}}\tilde{F}\right] =  - B_2 \cdot \frac{\binom{d}{t+1}\binom{p-d}{\barp-t-1}}{\binom{p}{\barp}}\tilde{F} + O\left(B_2 \left(\frac{\binom{d}{t+1}\binom{p-d}{\barp-t-1}}{\binom{p}{\barp}}\right)^2\right).
\end{equation}
In addition, we have
\begin{align}
	\frac{\binom{d}{t+1}\binom{p-d}{\barp-t-1}}{\binom{p}{\barp}} \leq \frac{d^{t+1}\barp(\barp-1)\cdots (\barp-t)}{(t+1)!(p-\barp+t+1)\cdots (p-\barp+1)} \leq \left(\frac{\barp d}{p}\right)^{t+1}\cdot \frac{1}{(t+1)!}.
\end{align}
Therefore
\begin{equation}
	B_2\cdot \frac{\binom{d}{t+1}\binom{p-d}{\barp-t-1}}{\binom{p}{\barp}} = o(1), B_2\cdot \left(\frac{\binom{d}{t+1}\binom{p-d}{\barp-t-1}}{\binom{p}{\barp}}\right)^2 = o(1),
\end{equation}
leading to $B_2\log\left[1 - \frac{\binom{d}{t+1}\binom{p-d}{\barp-t-1}}{\binom{p}{\barp}}\tilde{F}\right] = o(1)$, which implies the conclusion.
\end{proof}

Next let's prove the original theorem. Without loss of generality, assume there is a positive constant C such that at the first step of Algorithm \ref{algo_iteration}, $B_2$ satisfies
	\begin{equation}
		\frac{CDp^{\barp}}{D-\bar{p}^*+1}\leq B_2 \ll \left(\frac{p}{\barp D}\right)^{\barp + 1},
	\end{equation}
	and at the following steps it follows
	\begin{equation}
		 C\cdot\frac{(D+C_0)^{D+1}p^{\barp}(\log p)^{p^*}}{C_0^D(D-p^*+1)} \leq B_2  \ll \left(\frac{p}{\barp D}\right)^{\barp + 1}.
	\end{equation}
	
	Notice that condition (\rom{1}) in Assumption \ref{asmp: iterative rase} implies that
	\begin{equation}
			\lim_{n\rightarrow \infty}\tp\left(\supp \norma{\cri_n(S) - \cri(S)} > M_n\nu(n, p, D)\right) = 0.
	\end{equation}

We will prove the original theorem in three steps. Denote $\eta_j^{(t)} = \p(j \in S^{(t)}_{1*}), \hat{\eta}_j^{(t)} = \frac{1}{B_1}\sum_{i=1}^{B_1}\mathds{1}(j \in S^{(t)}_{i*}), \tilde{S}^{(t)} = \{j: \hat{\eta}_j^{(t)} > 1/\log p\}, \tilde{S}^{(t)}_* = \tilde{S}^{(t)} \cap S^*, \tilde{p}^*_t = |\tilde{S}^{(t)}_{*}|$.
\begin{enumerate}[label=(\roman*)]
	\item Step 1: When $t = 0$, due to the stepwise detectable condition, there exists a subset $S^{(0)}_{*} \subseteq S^*$ with cardinality $\barp$ that satisfies the conditions. On the other hand,
	\begin{align}
		\bp\left(\bigcup_{k=1}^{B_2} \{S_{1k} \supseteq S^{(0)}_{*}\}\right) &= 1 - \left[1 - \bp\left(S_{1k} \supseteq S^{(0)}_{*}\right)\right]^{B_2} \\
		&= 1- \left[1 - \frac{1}{D}\sum_{\barp \leq d \leq D}\frac{\binom{p-\barp}{d-\barp}}{\binom{p}{d}}\right]^{B_2} \\
		&\geq 1- \left[1 - \frac{D - p^* +1}{D}\cdot\frac{1}{p^{\barp}}\right]^{B_2},
	\end{align}
	where\[
		B_2 \log \left(1 - \frac{D - p^* +1}{D}\cdot\frac{1}{p^{\barp}}\right) \leq -B_2\cdot \frac{D - p^* +1}{D}\cdot\frac{1}{p^{\barp}} \leq -C,
	\]yielding that\[
		\bp\left(\bigcup_{k=1}^{B_2} \{S_{1k}^{(0)} \supseteq S^{(0)}_{*}\}\right) \geq 1 - e^{-C}.
	\] Then we have
	\begin{align}
		&\p\left(S_{1*}^{(0)} \supseteq S^{(0)}_{*}\right) \\
		&\geq \p\left(\bigcup_{k=1}^{B_2} \{S_{1k}^{(0)} \supseteq S^{(0)}_{*}\}, \inf_{\substack{S: |S \cap S^*| \leq \barp \\ |S| \leq D}}\cri_n(S) - \sup_{S: S \cap S^* = S^{(0)}_{*}}\cri_n(S) > M_n \nu(n, p, D),\right. \\
		&\quad \left. \bigcap_{k=1}^{B_2}\left\{|S_{1k}^{(0)} \cap S^*| \leq \barp\right\}\right) \\
		&\geq \p\left(\bigcup_{k=1}^{B_2} \{S_{1k}^{(0)} \supseteq S^{(0)}_{*}\}, \sup_{S: |S| \leq D}|\cri_n(S) -\cri(S)| \leq \frac{1}{2}M_n\nu(n, p, D), \bigcap_{k=1}^{B_2}\left\{|S_{1k}^{(0)} \cap S^*| \leq \barp\right\}\right)\\
		&\geq \bp\left(\bigcup_{k=1}^{B_2} \{S_{1k}^{(0)} \supseteq S^{(0)}_{*}\}\right) - \tp\left(\sup_{S: |S| \leq D}|\cri_n(S) -\cri(S)| > \frac{1}{2}M_n\nu(n, p, D)\right)\\
		&\quad - \bp\left(\bigcup_{k=1}^{B_2}\left\{|S_{1k}^{(0)} \cap S^*| \geq \barp+1\right\}\right)\\
		&\geq 1 - \frac{3}{2}e^{-C},\label{eq: s1s}
	\end{align}
	as $n$ is sufficiently large. The last inequality holds because there exists a variable $H \sim \textup{Hypergeometric}(p, \barp, D)$, such that
	\begin{align}
		\bp\left(\bigcup_{k=1}^{B_2}\left\{|S_{1k}^{(0)} \cap S^*| \geq \barp+1\right\}\right) &= 1 - \left[1 -\bp\left(|S_{1k}^{(0)} \cap S^*| \geq \barp+1\right) \right]^{B_2}\\
		&= 1 - \left[1 - \frac{1}{D}\sum_{1 \leq d \leq D}\bp\left(|S_{1k}^{(0)} \cap S^*| \geq \barp+1||S_{1k}^{(0)}| = d\right) \right]^{B_2}\\
		&\leq 1 - \left[1 - \frac{1}{D}\sum_{1 \leq d \leq D}\bp\left(|S_{1k}^{(0)} \cap S^*| \geq \barp+1||S_{1k}^{(0)}| = D\right) \right]^{B_2}\\
		&= \bp\left(\bigcup_{k=1}^{B_2}\left\{|S_{1k}^{(0)} \cap S^*| \geq \barp+1\right\}\bigg | \bigcap_{k=1}^{B_2}\left\{|S_{1k}^{(0)}| = D\right\}\right) \\
		&= 1 - \p(H \leq \barp) \\
		&\rightarrow 0
	\end{align}
	due to Lemma \ref{lem: iter thm}.
	
	Then for any $j \in S^{(0)}_{*}$, it holds that
	\begin{equation}
		\p\left(j \in S_{1*}^{(0)}\right) \geq \p\left(S_{1*}^{(0)}\supseteq S^{(0)}_{*}\right) \geq 1 - \frac{3}{2}e^{-C}.
	\end{equation}
	And by Hoeffding's inequality \citep{petrov2012sums}:
	\begin{equation}
		\p\left(\hat{\eta}_j^{(1)} \geq 1 - 2e^{-C}\right) \geq \p\left(\hat{\eta}_j^{(1)} - \eta_j^{(1)} > -\frac{1}{2}e^{-C}\right) \geq 1 - \exp\left\{-2B_1\cdot \frac{1}{4}e^{-2C}\right\}.
	\end{equation}
	Following by union bounds, we can obtain
	\begin{equation}\label{eq: wj}
		\p\left(\bigcap_{j \in S^{(0)}_{*}}\left\{\hat{\eta}_j^{(1)} \geq 1 - 2e^{-C}\right\}\right) \geq 1 - p^*\exp\left\{-\frac{1}{2}B_1e^{-2C}\right\}.
	\end{equation}
	
	\item Step 2: When $t = 1$, let's first condition on some specific $\tilde{S}^{(0)}$ defined before as $\{j: \hat{\eta}_j^{(0)} > C_0/\log p\}$ satisfying $|\tilde{S}^{(0)}_*| = \barp$. Later we will take the expectation to get the inequality without conditioning. To simplify and distinguish the notations, we omit the condition mentioned above and denote the new corresponding conditional probabilities as $\p_c, \bp_c, \tp_c$.
	
	For any specific $\tilde{S}^{(0)}_*$, due to the stepwise detectable condition again, there exists a subset $S^{(1)}_{*} \subseteq S^*$ with cardinality $\barp$ satisfies the conditions. Similar to step 1, we have
	\begin{align}\label{eq: allall}
		&\bp_c \left(\bigcup_{k=1}^{B_2} \left\{S_{1k}^{(1)} \supseteq \tilde{S}^{(0)}_{*} \cup S^{(1)}_{*} \right\}\right) \\
		&= 1 - \left[1 - \frac{1}{D}\sum_{\tilde{p}^*_1 + \barp \leq d \leq D}\bp_c\left(S_{1k}^{(1)} \supseteq \tilde{S}^{(0)}_{*} \cup S^{(1)}_*\big||S_{1k}^{(1)}| = d \right)\right]^{B_2},
	\end{align}
	where 
	\begin{align}
		&\bp_c\left(S_{1k}^{(1)} \supseteq \tilde{S}^{(0)}_{*} \cup S^{(1)}_*\big||S_{1k}^{(1)}| = d \right) \\
		&\geq \bp_c\left(S_{1k}^{(1)} \supseteq S^{(1)}_*\big||S_{1k}^{(1)}| = d, S_{1k}^{(1)} \supseteq \tilde{S}^{(0)}_{*} , S_{1k}^{(1)}\cap(\tilde{S}^{(0)}\backslash \tilde{S}^{(0)}_*) = \emptyset\right)\\
		&\quad \times \bp_c\left(S_{1k}^{(1)} \supseteq \tilde{S}^{(0)}_{*}, S_{1k}^{(1)}\cap(\tilde{S}^{(0)}\backslash \tilde{S}^{(0)}_*) = \emptyset \big||S_{1k}^{(1)}| = d\right).\label{eq: overall}
	\end{align}
	For convenience, let's consider a series $(j_1, \ldots, j_d)$ sampled from $\{1, \ldots, p\}$ without replacement with sampling weight $(\hat{\eta}_1^{(1)}, \ldots, \hat{\eta}_p^{(1)})$ in this order. We use $\p_c(j_1, \ldots, j_d||S_{1k}^{(1)}| = d)$ to represent the corresponding probability and use $\p_c(j_i|j_1, \ldots, j_{i-1},|S_{1k}^{(1)}| = d)$ to represent the conditional probability for sampling $j_i$ given $j_1, \ldots, j_{i-1}$.
	
	Based on the notations defined above, there holds
	\begin{align}
		&\bp_c\left(S_{1k}^{(1)} \supseteq \tilde{S}^{(0)}_{*}, S_{1k}^{(1)}\cap(\tilde{S}^{(0)}\backslash \tilde{S}^{(0)}_*) = \emptyset \big||S_{1k}^{(1)}| = d\right) \\
		&= \sum_{\substack{(j_1, \ldots, j_d) \supseteq \tilde{S}^{(0)}_{*}\\ (j_1, \ldots, j_d) \cap(\tilde{S}^{(0)}\backslash \tilde{S}^{(0)}_*) = \emptyset}}\bp_c\left(j_1, \ldots, j_d\big||S_{1k}^{(1)}| = d\right) \\
		&= \sum_{\substack{(j_1, \ldots, j_d) \supseteq \tilde{S}^{(0)}_{*}\\ (j_1, \ldots, j_d) \cap(\tilde{S}^{(0)}\backslash \tilde{S}^{(0)}_*) = \emptyset}}\bp_c\left(j_1\big||S_{1k}^{(1)}| = d\right)\bp_c\left(j_2\big|j_1, |S_{1k}^{(1)}| = d\right)\\
		&\quad \cdots \bp_c(j_d|j_1, \ldots, j_{d-1}, |S_{1k}^{(1)}| = d).\label{eq: ineq pr}
	\end{align}
	Since for $j_i \in \tilde{S}^{(0)}_*$, it holds
	\begin{align}
		\bp_c(j_i|j_1, \ldots, j_{i-1}, |S_{1k}^{(1)}| = d) &\geq \bp_c(j_i||S_{1k}^{(1)}| = d)\\
		 &\geq \frac{\hat{\eta}_{j_i}^{(0)}}{\sum_{j \in \tilde{S}^{(0)}_*}\hat{\eta}_j^{(0)} + \sum_{j \notin \tilde{S}^{(0)}_*}\frac{C_0}{p}} \\
		 &\geq \frac{C_0}{(D+C_0)\log p}.
	\end{align}
	And for $j_i \in \{1, \ldots, p\} \backslash \tilde{S}^{(0)}$, it holds
	\begin{align}
		\bp_c(j_i|j_1, \ldots, j_{i-1}, |S_{1k}^{(1)}| = d) &\geq \bp_c(j_i||S_{1k}^{(1)}| = d) \\
		&\geq \frac{\hat{\eta}_{j_i}^{(0)}}{\sum_{j \in \tilde{S}^{(0)}_*}\hat{\eta}_{j}^{(0)} + \sum_{j \notin \tilde{S}^{(0)}_*}\frac{C_0}{p}} \\
		&\geq \frac{C_0}{(D+C_0)p}.
	\end{align}
	Therefore by plugging these inequalities into \eqref{eq: ineq pr}, it yields that when $d \geq \tilde{p}_*^{(0)}$,
	\begin{align}
		&\bp_c\left(S_{1k}^{(1)} \supseteq \tilde{S}^{(0)}_{*}, S_{1k}^{(1)}\cap(\tilde{S}^{(0)}\backslash \tilde{S}^{(0)}_*) = \emptyset \big||S_{1k}^{(1)}| = d\right) \\
		&\geq \binom{p-|\tilde{S}^{(0)}|}{d - \tilde{p}_*^{(0)}}\binom{d}{\tilde{p}_*^{(0)}}\tilde{p}_*^{(0)}!(d - \tilde{p}_*^{(0)})! \cdot \left(\frac{C_0}{(D+C_0)\log p}\right)^{\tilde{p}_*^{(0)}}\cdot \left(\frac{C_0}{(D+C_0)p}\right)^{d-\tilde{p}_*^{(0)}}\\
		&= \left(1-\frac{|\tilde{S}^{(0)}|}{p}\right)\cdots \left(1-\frac{|\tilde{S}^{(0)}|-\tilde{p}_*^{(0)}+d-1}{p}\right)\cdot d(d-1)\cdots (d-\tilde{p}_*^{(0)}+1) \\
		&\quad \cdot \frac{C_0^d}{(D+C_0)^d(\log p)^{\tilde{p}_*^{(0)}}}\\
		&\geq \left(1-\frac{D\log p}{C_0p}\right)^{d-\tilde{p}_*^{(0)}}\cdot \frac{\tilde{p}_*^{(0)}!}{(\log p)^{\tilde{p}_*^{(0)}}} \cdot \left(\frac{C_0}{D+C_0}\right)^d\\
		&\geq \frac{1}{(\log p)^{p^*}} \cdot \left(\frac{C_0}{D+C_0}\right)^D,\label{eq: last}
	\end{align}
	when $n$ is sufficiently large. In the last second inequality, we used the fact that $|\tilde{S}^{(0)}| \leq \frac{D\log p}{C_0}$. On the other hand, given $|S_{1k}^{(1)}| = d, S_{1k}^{(1)} \supseteq \tilde{S}^{(0)}_{*} , S_{1k}^{(1)}\cap(\tilde{S}^{(0)}\backslash \tilde{S}^{(0)}_*) = \emptyset$, the indicators of whether the remaining features are sampled out or not follow the restricted multinomial distribution with parameters $(p-|\tilde{S}^{(0)}|, d - \tilde{p}_*^{(0)}, \bm{1})$, which means that the remaining variables have the same sampling weights. Then for $d \geq \tilde{p}_*^{(0)}-\barp$,
	\begin{align}
		&\bp_c\left(S_{1k}^{(1)} \supseteq S^{(1)}_*\big||S_{1k}^{(1)}| = d, S_{1k}^{(1)} \supseteq \tilde{S}^{(0)}_{*} , S_{1k}^{(1)}\cap(\tilde{S}^{(0)}\backslash \tilde{S}^{(0)}_*) = \emptyset\right) \\
		&= \frac{\binom{p-|\tilde{S}^{(0)}|-\barp}{d-\tilde{p}_*^{(0)}-\barp}}{\binom{p-|\tilde{S}^{(0)}|}{d-\tilde{p}_*^{(0)}}} \\
		&\geq \left(\frac{\max(d-\tilde{p}_*^{(0)}-\barp, 1)}{p-|\tilde{S}^{(0)}|}\right)^{\barp} \\
		&\geq \frac{1}{p^{\barp}},
	\end{align}
	which combined with \eqref{eq: allall}, \eqref{eq: overall} and \eqref{eq: last} leads to
	\begin{align}
		&\bp_c \left(\bigcup_{k=1}^{B_2} \left\{S_{1k}^{(1)} \supseteq \tilde{S}^{(0)}_{*} \cup S^{(1)}_{*} \right\}\right) \\
		&\geq 1 - \left[1 - \frac{1}{D}\sum_{\tilde{p}^*_1 + \barp \leq d \leq D}\frac{1}{(\log p)^{p^*}} \cdot \left(\frac{C_0}{D+C_0}\right)^D\cdot \frac{1}{p^{\barp}}\right]^{B_2} \\
		&\geq 1-\frac{1}{2}e^{-C},
	\end{align}
since $B_2 \geq C\cdot\frac{(D+C_0)^{D+1}p^{\barp}(\log p)^{p^*}}{C_0^D(D-p^*+1)}$. 
	
	Furthermore,
	\begin{equation}
		\bp_c\left(\bigcup_{k=1}^{B_2}\left\{|S_{1k}^{(1)} \cap (S^*\backslash \tilde{S}^{(0)}_*)| \geq \barp+1\right\}\right) = 1-\left[1 - \bp_c\left(|S_{11}^{(1)} \cap (S^*\backslash \tilde{S}^{(0)}_*)| \geq \barp+1\right)\right]^{B_2},
		\end{equation}
		where
		\begin{align}
		&\bp_c\left(|S_{11}^{(1)} \cap (S^*\backslash \tilde{S}^{(0)}_*)| \geq \barp+1\right) \\
		&= \bp_c\left(|S_{11}^{(1)} \cap (S^*\backslash \tilde{S}^{(0)}_*)| \geq \barp+1\big|S_{11}^{(1)} \cap (\tilde{S}^{(0)}\backslash \tilde{S}^{(0)}_*) = \emptyset\right)\bp_c\left(S_{11}^{(1)} \cap (\tilde{S}^{(0)}\backslash \tilde{S}^{(0)}_*) = \emptyset\right) \\
		&\quad + \bp_c\left(|S_{11}^{(1)} \cap (S^*\backslash \tilde{S}^{(0)}_*)| \geq \barp+1\big|S_{11}^{(1)} \cap (\tilde{S}^{(0)}\backslash \tilde{S}^{(0)}_*) \neq \emptyset\right)\bp_c\left(S_{11}^{(1)} \cap (\tilde{S}^{(0)}\backslash \tilde{S}^{(0)}_*) \neq \emptyset\right) \\
		&\leq \bp_c\left(|S_{11}^{(1)} \cap (S^*\backslash \tilde{S}^{(0)}_*)| \geq \barp+1\big|S_{11}^{(1)} \cap (\tilde{S}^{(0)}\backslash \tilde{S}^{(0)}_*) = \emptyset\right)\\
		&= \frac{1}{D}\sum_{1 \leq d \leq D}\bp_c\left(|S_{11}^{(1)} \cap (S^*\backslash \tilde{S}^{(0)}_*)| \geq \barp+1\big|S_{11}^{(1)} \cap (\tilde{S}^{(0)}\backslash \tilde{S}^{(0)}_*) = \emptyset, |S_{11}^{(1)}| = d\right)\\
		&\leq \frac{1}{D}\sum_{1 \leq d \leq D}\bp\left(|S_{11}^{(1)} \cap (S^*\backslash \tilde{S}^{(0)}_*)| \geq \barp+1\big|S_{11}^{(1)} \cap (\tilde{S}^{(0)}\backslash \tilde{S}^{(0)}_*) = \emptyset, |S_{11}^{(1)}| = D\right) \\
		&=\bp_c\left(|S_{11}^{(1)} \cap (S^*\backslash \tilde{S}^{(0)}_*)| \geq \barp+1\big|S_{11}^{(1)} \cap (\tilde{S}^{(0)}\backslash \tilde{S}^{(0)}_*) = \emptyset, |S_{11}^{(1)}| = D\right).
	\end{align}
	Similar to step 1, because of Lemma \ref{lem: iter thm}, it follows that
	\begin{align}
		&\bp_c\left(\bigcup_{k=1}^{B_2}\left\{|S_{1k}^{(1)} \cap (S^*\backslash \tilde{S}^{(0)}_*)| \geq \barp+1\right\}\right) \\
		&\geq 1-\left[1 - \bp_c\left(|S_{11}^{(1)} \cap (S^*\backslash \tilde{S}^{(0)}_*)| \geq \barp+1\big|S_{11}^{(1)} \cap (\tilde{S}^{(0)}\backslash \tilde{S}^{(0)}_*) = \emptyset, |S_{11}^{(1)}| = D\right)\right]^{B_2} \\
		&= \bp_c\left(\bigcup_{k=1}^{B_2}\left\{|S_{1k}^{(1)} \cap (S^*\backslash \tilde{S}^{(0)}_*)| \geq \barp+1\right\}\bigg| \bigcap_{k=1}^{B_2}\left\{S_{1k}^{(1)} \cap (\tilde{S}^{(0)}\backslash \tilde{S}^{(0)}_*) = \emptyset, |S_{1k}^{(1)}| = D\right\}\right)\\
		&= o(1),
	\end{align}
	uniformly for any $\tilde{S}^{(0)}$.
	Then similar to \eqref{eq: s1s}, we have for $j \in \tilde{S}^{(0)}_{*} \cup S^{(1)}_{*}$:
	\begin{equation}
		\eta_j^{(1)} = \p_c\left(j \in S_{1*}^{(1)}\right) \geq \p_c\left(S_{1*}^{(1)}\supseteq \tilde{S}^{(0)}_{*} \cup S^{(1)}_{*}\right) \geq 1 - \frac{3}{2}e^{-C}.
	\end{equation}
	And by Hoeffding's inequality \citep{petrov2012sums}:
	\begin{equation}
		\p\left(\hat{\eta}_j^{(1)} \geq 1 - 2e^{-C}\big| \tilde{S}^{(0)}\right) \geq \p\left(\hat{\eta}_j^{(1)} - \eta_j^{(1)} > -\frac{1}{2}e^{-C}\right) \geq 1 - \exp\left\{-2B_1\cdot \frac{1}{4}e^{-2C}\right\}.
	\end{equation}
	By union bounds,  it holds that
	\begin{equation}
		\p\left(\bigcap_{j \in \tilde{S}^{(0)}_{*} \cup S^{(1)}_{*}}\left\{\hat{\eta}_j^{(1)} \geq 1 - 2e^{-C}\right\}\bigg| \tilde{S}^{(0)}\right) \geq 1 - p^*\exp\left\{-\frac{1}{2}B_1e^{-2C}\right\}.
	\end{equation}
	Since the above conclusions hold for any $\tilde{S}^{(0)}$ and $|\tilde{S}^{(0)}_*| \geq |S^{(0)}_*| = \barp$, we can conclude that
	\begin{align}
		&\p\left(\sum_{j=1}^p \mathds{1}(\hat{\eta}_j^{(1)} \geq 1 - 2e^{-C}) \geq 2\barp\right) \\
		&\geq \e_{\tilde{S}^{(0)}}\left[\p\left(\bigcap_{j \in \tilde{S}^{(0)}_{*} \cup S^{(1)}_{*}}\left\{\hat{\eta}_j^{(1)} \geq 1 - 2e^{-C}\right\}\bigg| \tilde{S}^{(0)}\right)\Bigg||S^{(0)}_*| = \barp\right]\\
		&\geq \left(1 - p^*\exp\left\{-\frac{1}{2}B_1e^{-2C}\right\}\right)\left(1 - p^*\exp\left\{-\frac{1}{2}B_1e^{-2C}\right\}\right)\\
		&\geq 1- 2p^*\exp\left\{-\frac{1}{2}B_1e^{-2C}\right\}.
	\end{align}
	After the the second iteration step, with probability $1- 2p^*\exp\left\{-\frac{1}{2}B_1e^{-2C}\right\}$, there will be at least $2\barp$ features of $S^*$ covered in $\tilde{S}^{(1)}$ and having $\hat{\eta}_j >  1 - 2e^{-C}$. Similarly, after the the $t$-th iteration step, there will be at least $(t+1)\barp$ features of $S^*$ covered in $\tilde{S}^{(t)}$ and having $\hat{\eta}_j >  1 - 2e^{-C}$.
	
	\item Step 3: By step 2, after at most $t' = \lceil \frac{p^*}{\barp}\rceil - 1$ iterations, $\tilde{S}^{(t')}$ will cover $S^*$. Without loss of generality, let's assume the smallest $t'$ satisfying $\tilde{S}^{(t')} \supseteq S^*$ equal to $\lceil \frac{p^*}{\barp}\rceil$. Then when the iteration number $t = \lceil \frac{p^*}{\barp}\rceil$, we have
	\begin{equation}
		\p\left(\bigcap_{j \in S^*}\left\{\hat{\eta}_j^{(t-1)} \geq 1 - 2e^{-C}\right\}\right) \geq 1 - tp^*\exp\left\{-\frac{1}{2}B_1e^{-2C}\right\}.
	\end{equation} 
	Using the same notations defined at the beginning of step 2 (notice that here we only need to condition on $\bigcap\limits_{j \in S^*}\{\hat{\eta}_j^{(t-1)} \geq 1 - 2e^{-C}\}$), we have
	\begin{align}
		&\bp_c\left(S_{1k}^{(t)} \supseteq S^*\right) \\
		&\geq \frac{1}{D}\sum_{p^* \leq d \leq D} \sum_{(j_1,\ldots, j_d) \supseteq S^*}\bp_c(j_1,\ldots, j_d| |S_{1k}^{(t)}| = d) \\
		&= \frac{1}{D}\sum_{p^* \leq d \leq D} \sum_{(j_1,\ldots, j_d) \supseteq S^*}\bp_c\left(j_1\big||S_{1k}^{(t)}| = d\right)\cdots \bp_c\left(j_d|j_1, \ldots, j_{d-1}, |S_{1k}^{(t)}| = d\right).
	\end{align}
	Similar to step 2, for $j_i \in S^*$, it holds
	\begin{equation}
		\bp_c(j_i|j_1, \ldots, j_{i-1}, |S_{1k}^{(t)}| = d) \geq \bp_c(j_i||S_{1k}^{(t)}| = d) \geq \frac{\hat{\eta}_{j_i}^{(t-1)}}{\sum_{j \in S^*}\hat{\eta}_j^{(t-1)} + \sum_{j \notin S^*}\frac{C_0}{p}} \geq \frac{1-2e^{-C}}{D+C_0}.
	\end{equation}
	And for $j_i \in \{1, \ldots, p\} \backslash S^*$, it holds
	\begin{equation}
		\bp_c(j_i|j_1, \ldots, j_{i-1}, |S_{1k}^{(t)}| = d) \geq \bp_c(j_i||S_{1k}^{(t)}| = d) \geq \frac{\hat{\eta}_{j_i}^{(1)}}{\sum_{j \in S^*}\hat{\eta}_j^{(t)} + \sum_{j \notin S^*}\frac{C_0}{p}} \geq \frac{C_0}{(D + C_0)p}.
	\end{equation}
	Thus, we have
	\begin{align}
		\bp_c\left(S_{1k}^{(t)} \supseteq S^*\right) &\geq \frac{1}{D}\sum_{p^* \leq d \leq D} \sum_{(j_1,\ldots, j_d) \supseteq S^*}\left(\frac{1-2e^{-C}}{D+C_0}\right)^{p^*}\cdot \left(\frac{C_0}{(D + C_0)p}\right)^{d-p^*}\\
		&\geq  \frac{1}{D}\sum_{p^* \leq d \leq D}\binom{p - p^*}{d - p^*}d! \left(\frac{1-2e^{-C}}{D+C_0}\right)^{p^*}\cdot \left(\frac{C_0}{(D + C_0)p}\right)^{d-p^*} \\
		&\geq \frac{D - p^* +1}{D}\cdot \left(\frac{1-2e^{-C}}{D+C_0}\right)^{p^*}\cdot \left(\frac{C_0}{D + C_0}\right)^{D-p^*}\cdot \left(1-\frac{D-1}{p}\right)^{D-p^*} \\
		&= \frac{(D-p^* + 1)(1-2e^{-C})^{p^*}C_0^{D-p^*}}{D(D + C_0)^{D}}\cdot \left(1-\frac{D-1}{p}\right)^{D-p^*},
	\end{align}
	leading to
	\begin{align}
		&\bp_c\left(\bigcup_{k=1}^{B_2}\left\{S_{1k}^{(t)} \supseteq S^*\right\}\right) \\
		&= 1-\left[1-\bp_c\left(S_{1k}^{(t)} \supseteq S^*\right)\right]^{B_2}\\
		&\geq 1- \exp\left\{-B_2\cdot \frac{(D-p^* + 1)(1-2e^{-C})^{p^*}C_0^{D-p^*}}{D(D + C_0)^{D}}\cdot \left(1-\frac{D-1}{p}\right)^{D-p^*}\right\} \\
		&\geq 1 - \exp\left\{-C\cdot \frac{(D+C_0)p^{\barp}}{DC_0^{p^*}}\cdot \left(1-\frac{D-1}{p}\right)^{D-p^*}\right\}.
	\end{align}
	Thus we have
	\begin{align}
		&\bp\left(\bigcup_{k=1}^{B_2}\left\{S_{1k}^{(t)} \supseteq S^*\right\}\right) \\
		&\geq \bp_c\left(S_{1k}^{(t)} \supseteq S^*\right)\p\left(\bigcap_{j \in S^*}\left\{\hat{\eta}_j^{(t-1)} \geq 1 - 2e^{-C}\right\}\right)\\
		&\geq 1 - \exp\left\{-C\cdot \frac{(D+C_0)^{p^*+1}p^{\barp}}{DC_0^{p^*}}\cdot \left(1-\frac{D-1}{p}\right)^{D-p^*}\right\} -p^*\left\lceil\frac{p^*}{\barp}\right\rceil\exp\left\{-\frac{1}{2}B_1e^{-2C}\right\}.
	\end{align}
	Then similar to \eqref{eq: s1s}, there holds
	\begin{align}
		\p(S^{(t)}_{1*} \supseteq S^*) &\geq \p\left(\bigcup_{k=1}^{B_2} \{S_{1k} \supseteq S^{(t)}_{*}\}, \inf_{\substack{S: S\not\supseteq S^* \\ |S| \leq D}}\cri_n(S) - \sup_{S: S \supseteq S^*}\cri_n(S) > M_n\nu(n, p, D)\right) \\
		&\geq \p\left(\bigcup_{k=1}^{B_2} \{S_{1k} \supseteq S^{(t)}_{*}\}, \supp|\cri_n(S) -\cri(S)| \leq \frac{1}{2}M_n\nu(n, p, D)\right) \\
		&\geq 1- \bp\left(\bigcup_{k=1}^{B_2}\left\{S_{1k}^{(t)} \supseteq S^*\right\}\right) - \tp\left(\supp|\cri_n(S) -\cri(S)| > \frac{1}{2}M_n\nu(n, p, D)\right).
	\end{align}
	Combined with all the conclusions above, as $n, B_1, B_2 \rightarrow \infty$, we get
	\begin{align}
		\p(S^{(t)}_{1*} \not\supseteq S^*) &\leq \exp\left\{-C\cdot \frac{(D+C_0)^{p^*+1}p^{\barp}}{DC_0^{p^*}}\cdot \left(1-\frac{D-1}{p}\right)^{D-p^*}\right\} \\
		&\quad+ p^*\left\lceil\frac{p^*}{\barp}\right\rceil\exp\left\{-\frac{1}{2}B_1e^{-2C}\right\}\\
		&\quad+ \tp\left(\supp|\cri_n(S) -\cri(S)| > \frac{1}{2}M_n\nu(n, p, D)\right) \\
		&\rightarrow 0.
	\end{align}
	And for $t > \lceil \frac{p^*}{\barp}\rceil$, the same conclusion can be obtained by following the same procedure, which completes our proof.
\end{enumerate}
\vskip 0.2in
\bibliography{reference}

\end{document}